\documentclass[11pt]{article}
\usepackage{amsmath,amsthm,amssymb,graphicx}
\usepackage[margin = 1in]{geometry}
\usepackage{grffile}
\usepackage{amssymb}
\usepackage{amsmath}
\usepackage{bm}
\usepackage{amsthm}
\usepackage{epsfig}
\usepackage{graphicx}
\usepackage{graphics}
\usepackage{float}
\usepackage{subfigure}
\usepackage{multirow}
\usepackage{color}
\usepackage{xcolor}
\usepackage[normalem]{ulem} 
\usepackage{makeidx}
\usepackage{xspace}
\usepackage{wrapfig}
\usepackage{import}
\usepackage{pdfpages}
\usepackage[export]{adjustbox}

\usepackage[toc,page]{appendix}
\usepackage{minitoc}

\counterwithout{figure}{section}

\usepackage{hyperref}
\hypersetup{
    colorlinks=true,
    linkcolor=blue,
    filecolor=magenta,      
    urlcolor=cyan,
    pdftitle={Overleaf Example},
    pdfpagemode=FullScreen,
}

\usepackage{authblk}
\usepackage[numbers]{natbib}

\usepackage[utf8]{inputenc} 
\usepackage[T1]{fontenc}    
\usepackage{hyperref}       
\usepackage{url}            
\usepackage{booktabs}       
\usepackage{amsfonts}       
\usepackage{nicefrac}       
\usepackage{microtype}      
\usepackage{xcolor}         

\usepackage[framemethod=tikz,xcolor=true]{mdframed}

\newtheorem{theorem}{Theorem}[section]
\newtheorem{rmq}{Remark}[section]
\newtheorem{lemma}[theorem]{Lemma}
\newtheorem{corollary}[theorem]{Corollary}
\newtheorem{prop}[theorem]{Proposition}
\newtheorem{definition}[theorem]{Definition}

\newtheorem{assumption}{Assumption}[section]

\theoremstyle{definition}

\newcommand{\E}{\mathbb{E}}

\title{\textbf{Triangular Flows for Generative Modeling:\\ Statistical Consistency, Smoothness Classes,\\ and Fast Rates}}

\date{December 31, 2021}

\author[1]{Nicholas J. Irons}
\affil[1]{Department of Statistics,
  University of Washington}

\author[2]{Meyer Scetbon}
\affil[2]{Department of Statistics, CREST ENSAE}

\author[3]{Soumik Pal}
\affil[3]{Department of Mathematics,
  University of Washington}

\author[1]{Zaid Harchaoui}

\begin{document}
\maketitle

\doparttoc 
\faketableofcontents 

\begin{abstract}
Triangular flows, also known as Kn\"othe-Rosenblatt measure couplings, comprise an important building block of normalizing flow models for generative modeling and density estimation, including popular autoregressive flows such as Real NVP. We present statistical guarantees and sample complexity bounds for triangular flow statistical models. In particular, we establish the statistical consistency and the finite sample convergence rates of the Kullback-Leibler estimator of the Kn\"othe-Rosenblatt measure coupling using tools from empirical process theory. Our results highlight the anisotropic geometry of function classes at play in triangular flows, shed light on optimal coordinate ordering, and lead to statistical guarantees for Jacobian flows. We conduct numerical experiments on synthetic data to illustrate the practical implications of our theoretical findings.
\end{abstract}

\section{Introduction}
\label{sec:intro}
Triangular flows are popular generative models that allow one to define complex multivariate distributions via push-forwards from simpler multivariate distributions \cite{kobyzev}. Triangular flow models target the \emph{Kn\"{o}the-Rosenblatt map}~\cite{spantini}, which originally appeared in two independent papers by Kn\"{o}the and Rosenblatt \cite{knothe,rosenblatt}. The Kn\"{o}the-Rosenblatt map is a multivariate function $S^*$ from $\mathbb{R}^d$ onto itself pushing a Lebesgue probability density $f$ onto another one $g$:
\begin{equation*}
S^*\#f=g
\end{equation*}
The Kn\"{o}the-Rosenblatt map has the striking property of being \emph{triangular} in that its Jacobian is an upper triangular matrix. 
This map and its properties have been studied in probability theory, nonparametric statistics, and optimal transport, under the name 
of Kn\"othe-Rosenblatt (KR) coupling or rearrangement. 

KR can also be seen as an \emph{optimal transport under anisotropic distortion} \cite{santambrogio},
hence shared connections with optimal transport theory~\cite{otam,villani}. Unlike the well-known Brenier map, KR
has a simple and explicit definition in terms of  conditional densities of distributions at play, sparing the need for optimization to be defined. 

KR can be used to synthesize a sampler of a probability distribution given data drawn from that probability distribution. Moreover, 
any probability distribution can be well approximated via a KR map. This key property has been an important motivation of 
a number of flow models~\cite{nice,nvp,moselhy,glow,kobyzev,sampling,spantini,naf,maf,mnn,made} in machine learning, statistical science, computational science, and AI domains. Flow models or normalizing flows have achieved great success in a number of settings, allowing 
one to generate images that look realistic as well as texts that look as if they were written by humans. The theoretical analysis of 
normalizing flows is yet a huge undertaking as many challenges arise at the same time: the recursive construction of a push forward, 
the learning objective to estimate the push forward, the neural network functional parameterization, and the statistical modeling of complex data. 

We focus in this paper on the Kn\"{o}the-Rosenblatt coupling, and more generally triangular flows, as it can be seen as the statistical backbone of normalizing flows. 
The authors of~\citep{spantini} showed how to estimate KR from data by minimizing a Kullback-Leibler objective. The estimated KR can then be used to sample at will 
from the probability distribution at hand. 
The learning objective of~\cite{spantini} has the benefit of being statistically classical, 
hence amenable to detailed analysis compared to adversarial learning objectives which are still subject to active research. 
The sample complexity or rate of convergence of this KR estimator is however, to this day, unknown. 
On the other hand, KR and its multiple relatives are frequently motivated from a universal approximation perspective \cite{naf,kobyzev}, which can be misleading as it focuses on the approximation aspect only. Indeed, while universal approximation holds for many proposed models, slow rates can occur (see Theorem \ref{thm-slow-rates}).

\paragraph{Contributions.} We present a theoretical analysis of the Kn\"{o}the-Rosenblatt coupling, from its statistical framing to convergence rates.
We put forth a simple example of slow rates  showing the limitations of a viewpoint based on universal approximation only. 
This leads us to identify the function classes that the KR maps belong to and bring to light their anisotropic geometry. We then establish
finite sample rates of convergence using tools from empirical process theory. 
Our analysis delineates different regimes of statistical convergence, 
depending on the dimension and the sample. Our theoretical results hold under general conditions.
Assuming that the source density is log-concave, we establish fast rates of Sobolev-type convergence in the smooth regime. 
We outline direct implications of our results on Jacobian flows. We provide numerical illustrations on synthetic data to demonstrate the practical implications of our theoretical results.

\section{Triangular Flows}
\label{sec:krmap}

\paragraph{Kn\"{o}the-Rosenblatt rearrangement.}
KR originated from independent works of Rosenblatt and Kn\"{o}the and has spawned fruitful applications in diverse areas. M. Rosenblatt studied the KR map for statistical purposes, specifically multivariate goodness-of-fit testing~\cite{rosenblatt}. 
H. Kn\"{o}the, on the other hand, elegantly utilized the KR map to extend the Brunn-Minkowski inequality, which can be used to prove the celebrated isoperimetric inequality \cite{knothe}.
More recently, triangular flows have been proposed as simple and expressive building blocks of generative models, which aim to model a distribution given samples from it. They have been studied and implemented for the problems of sampling and density estimation, among others \cite{kobyzev,spantini,sampling,moselhy}. A triangular flow can be used to approximate the KR map between a source density and a target density from their respective samples.

Consider two Lebesgue probability densities $f$ and $g$ on $\mathbb{R}^d$. The Kn\"othe-Rosenblatt map $S^*:\mathbb{R}^d\to\mathbb{R}^d$ between $f$ and $g$ is the unique monotone non-decreasing upper triangular measurable map pushing $f$ forward to $g$, written $S^*\#f=g$ \cite{spantini}. The KR map is upper triangular in the sense that its Jacobian is an upper triangular matrix, since $S^*$ takes the form
\[
S^*(x) = \begin{bmatrix}
S^*_1(x_1,\ldots,x_d) \\
S^*_2(x_2,\ldots,x_d) \\
\vdots \\
S^*_{d-1}(x_{d-1},x_d) \\
S^*_d(x_d)
\end{bmatrix}.
\]
Furthermore, $S^*$ is monotone non-decreasing in the sense that the univariate map $x_k\mapsto S^*_k(x_k,\ldots,x_d)$ is monotone non-decreasing for any $(x_{k+1},\ldots,x_d)\in\mathbb{R}^{d-k}$ for each $k\in\{1,\ldots,d\}=:[d]$. We will sometimes abuse notation and write $S^*_k(x)$ for $x=(x_1,\ldots,x_d)$, even though $S^*_k$ depends only on $(x_k,\ldots,x_d)$. 

The components of the KR map can be defined recursively via the monotone transport between the univariate conditional densities of $f$ and $g$. Let $F_k(x_k|x_{(k+1):d})$ denote the cdf of the conditional density 
\[
f_k(x_k|x_{(k+1):d})
:= \frac{\int f(x_1,\ldots,x_d)dx_1\cdots dx_{k-1}}{\int f(x_1,\ldots,x_d)dx_1\cdots dx_k}.
\]
Similarly, let $G_k(y_k|y_{(k+1):d})$ denote the conditional cdf of $g$. In particular, when $k=d$ we obtain $F_d(x_d)$, 
the cdf of the $d$-th marginal density $f_d(x_d)= \int f(x_1,\ldots,x_d)dx_1\cdots dx_{d-1}$, and similarly for $g$.
Assuming $g(y)>0$ everywhere, the maps $y_k\mapsto G_k(y_k|y_{(k+1):d})$ are strictly increasing and therefore invertible. In this case, the $d$th component of $S^*$ is defined as the monotone transport between $f_d(x_d)$ and $g_d(y_d)$ \cite{otam},
\[
S^*_d(x_d) = G_d^{-1}(F_d(x_d)).
\]
From here the $k$th component of $S^*$ is given by
\[
S^*_k(x_k,\ldots,x_d) = G_k^{-1}\left(F_k(x_k|x_{(k+1):d})\bigg|S^*_{(k+1):d}(x_{(k+1):d})\right)
\]
for $k=1,\ldots, d-1$, where $x_{(k+1):d}=(x_{k+1},\ldots,x_d)$ and 
\[
S^*_{(k+1):d}(x_{(k+1):d}) = (S^*_{k+1}(x_{(k+1):d}),\ldots,S^*_d(x_d)).
\]
That $S^*$ is upper triangular and monotonically non-decreasing is clear from the construction. Under tame assumptions on $f$ and $g$ discussed below, $S^*$ is invertible. We denote $T^*=(S^*)^{-1}$, which is the KR map from $g$ to $f$.
In this paper, we shall be interested in the asymptotic and non-asymptotic convergence of statistical estimators towards $S^*$ and  $T^*$, respectively.

In one dimension ($d=1$), the Kn\"{o}the-Rosenblatt rearrangement simply matches the quantiles of $f$ and $g$ via monotone transport, $S^*(x)=(G^{-1}\circ F)(x)$. In this case, the KR rearrangement is an optimal transport map for any convex cost that is a function of the displacement $x-S(x)$ (see, e.g., Theorem 2.9 in \cite{otam}). 
In higher dimensions, the KR map is not generally an optimal transport map for the commonly studied cost functions. However, it does arise as a limit of optimal transport maps under an anisotropic quadratic cost.
We discuss this connection further in Section \ref{sec:anisotropic}, in which we prove a theorem concerning densities of anisotropic smoothness that complements the result of Carlier, Galichon, and Santambrogio \cite{santambrogio}. 

\begin{figure*}[t]
\centering
\begin{tabular}{llll}
\hspace{-4mm}
\includegraphics[height=40mm,width=0.25\textwidth]{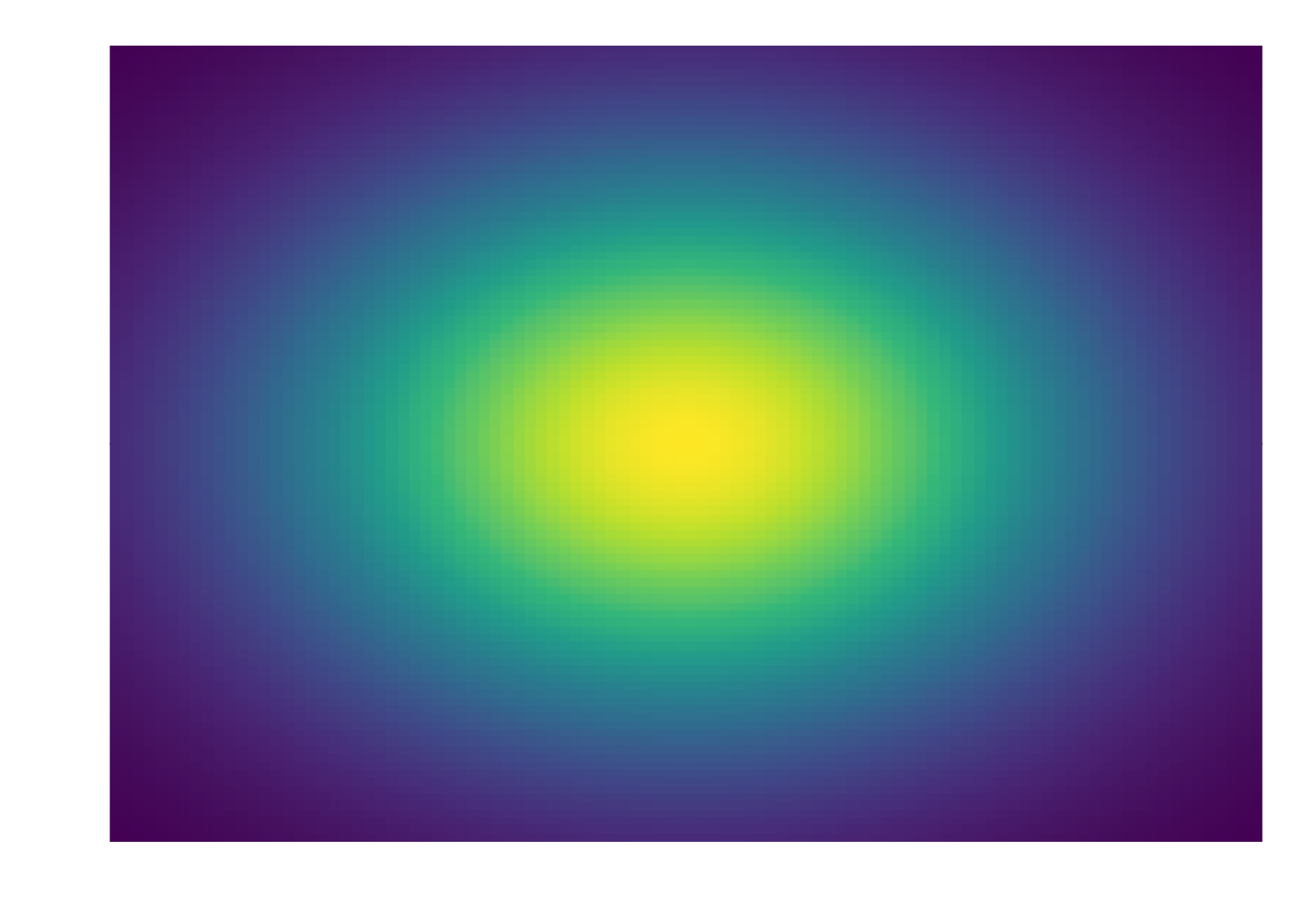} 
&
\hspace{-8mm}
\includegraphics[height=40mm,width=0.25\textwidth]{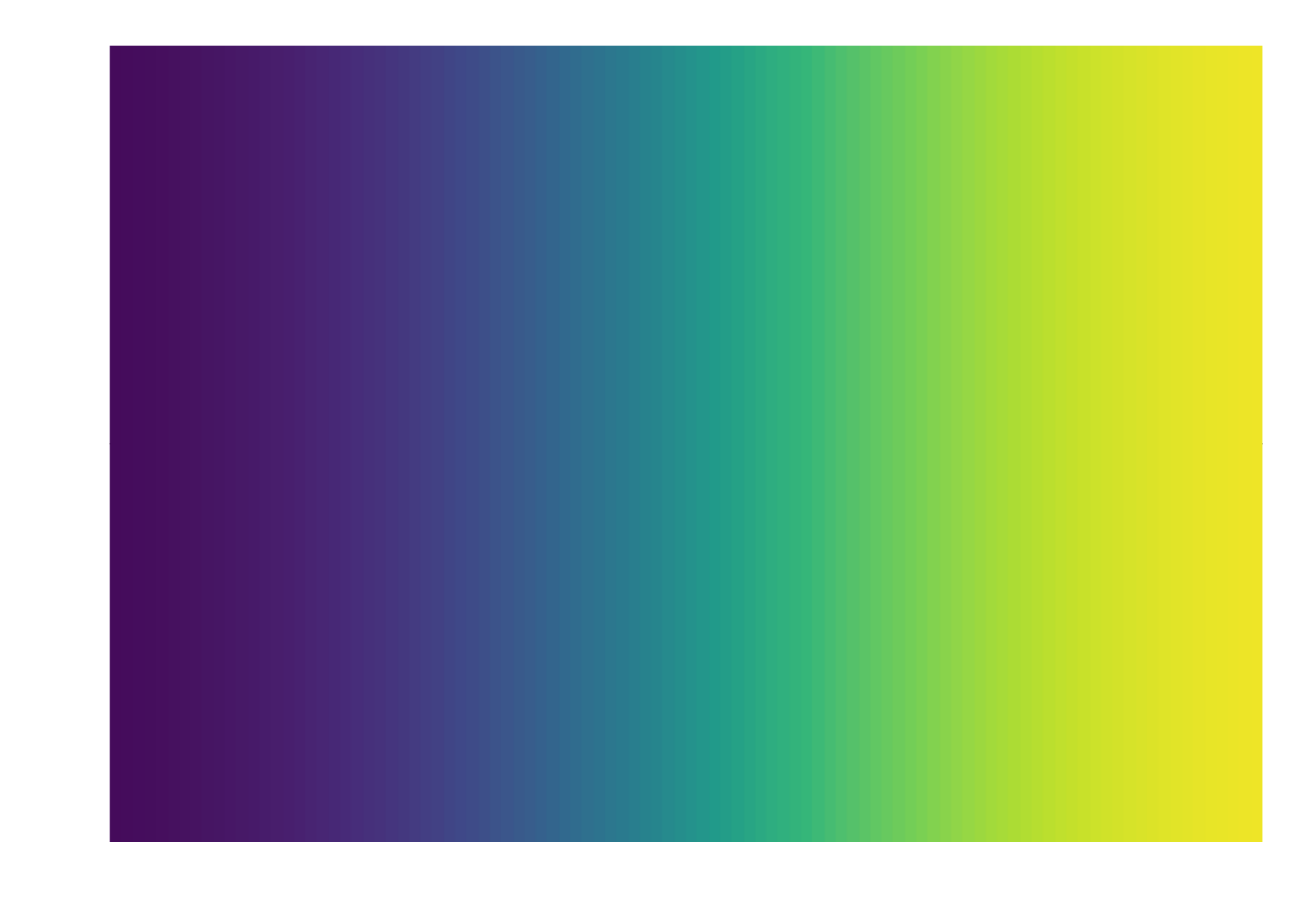}
&
\hspace{-8mm}
\includegraphics[height=40mm,width=0.25\textwidth]{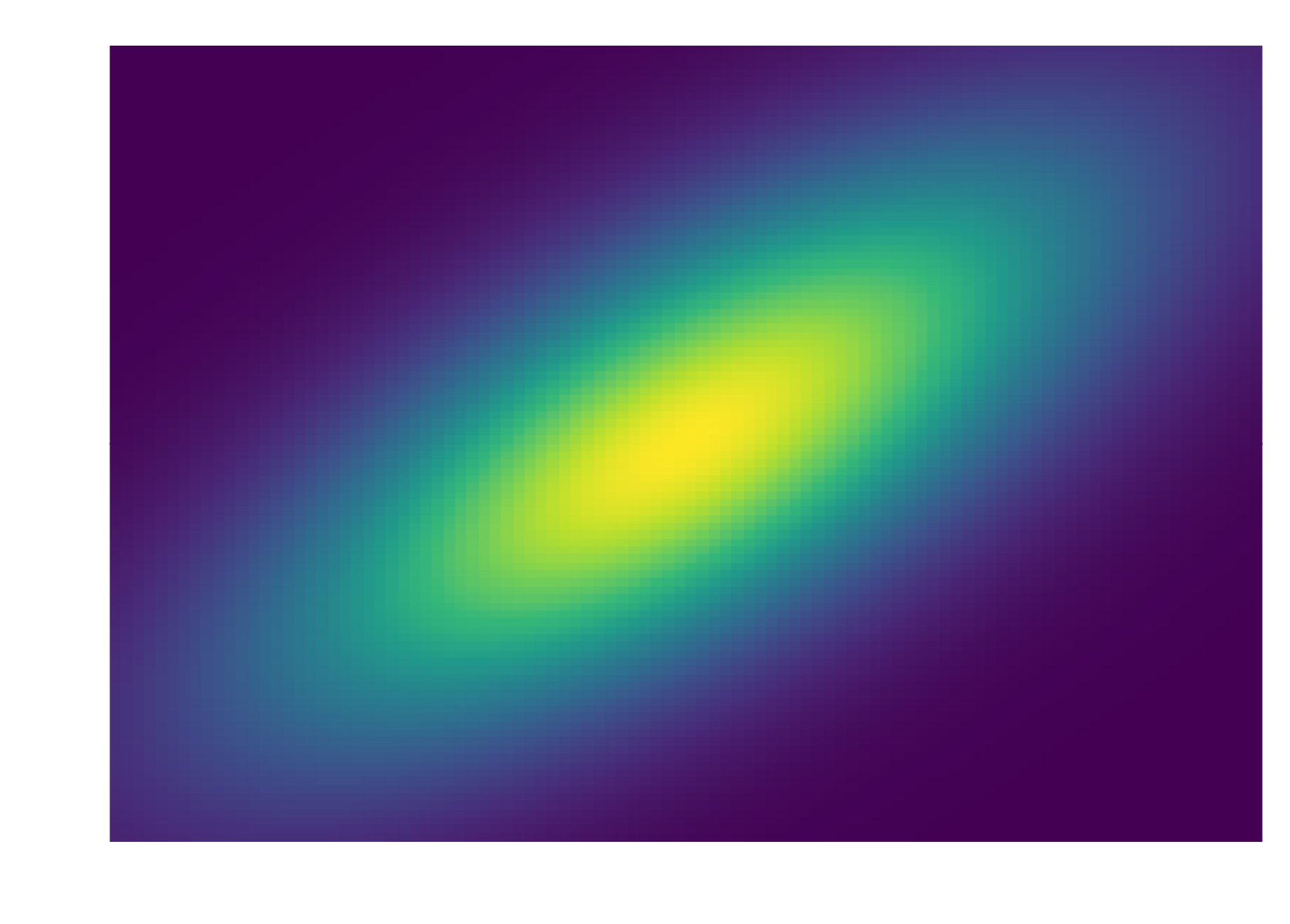}
&
\hspace{-8mm}
\includegraphics[height=40mm,width=0.25\textwidth]{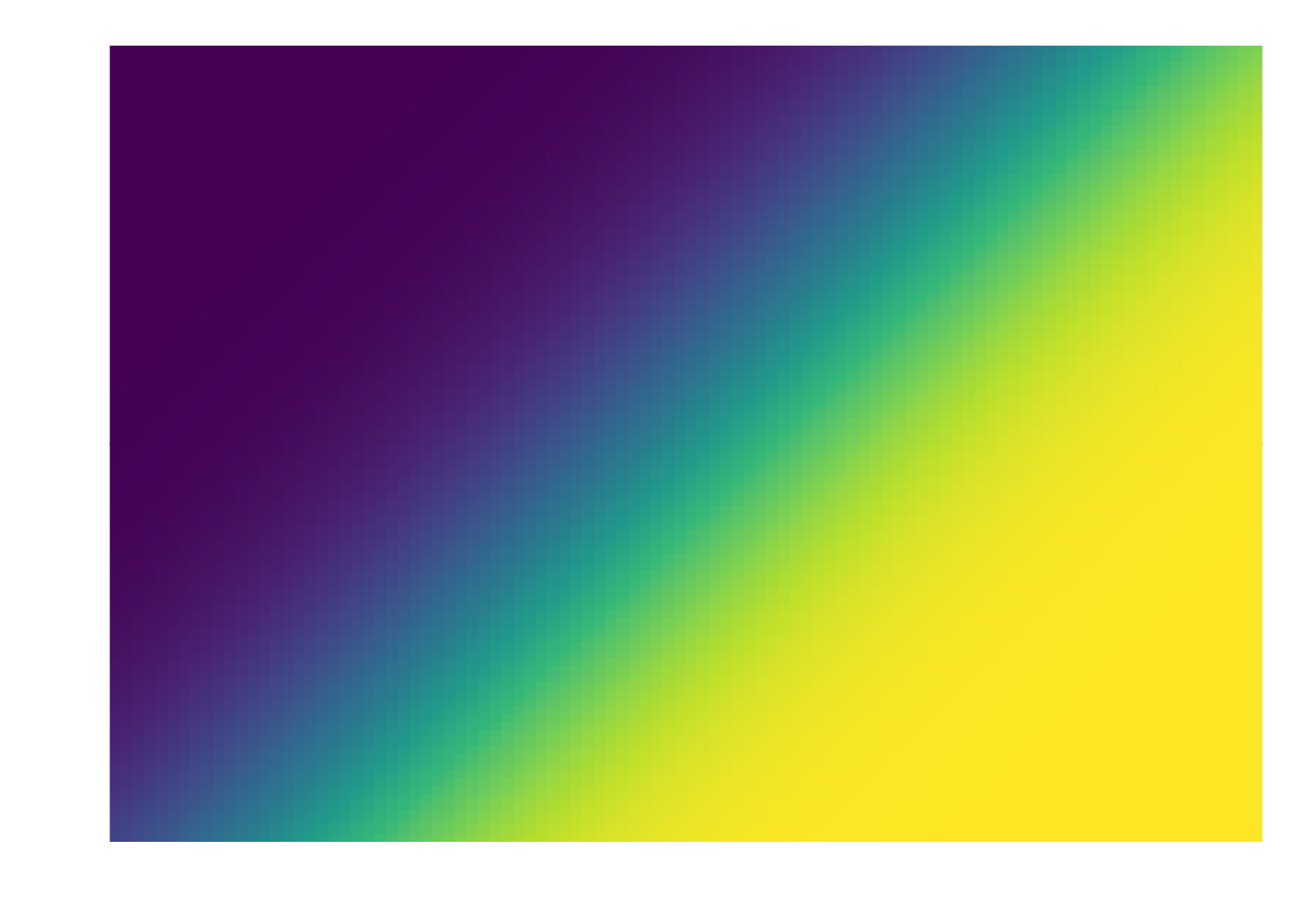} 
\end{tabular}
\caption{\textbf{Left:} Heat map of standard bivariate Gaussian density $f(x_1,x_2)$ with correlation $\rho=0$ (leftmost) and the first component of the KR map $F_1(x_1|x_2)$ from $f$ to $g$, the uniform density on $[0,1]^2$. \textbf{Right:} Bivariate Gaussian density with $\rho=0.7$ and first component of the KR map $F_1(x_1|x_2)$ (rightmost).}
\label{fig:bvnorm}
\end{figure*}

\paragraph{From the uniform distribution to any distribution.}
In his seminal paper on the Kn\"{o}the-Rosenblatt rearrangement, Rosenblatt considered the special case in which $g$ is the uniform density on the hypercube $[0,1]^d$ \cite{rosenblatt}. This implies that the conditional cdfs are identity maps, $G_k(y_k|y_{k+1},\ldots,y_d)=y_k$, and so the KR map from $f$ to $g$ consists simply of the conditional cdfs
\begin{align*}
S^*_d(x_d) &= P(X_d\le x_d) = F_d(x_d), \\
S^*_{d-1}(x_{d-1},x_d) &= P(X_{d-1}\le x_{d-1}|X_d=x_d) = F_{d-1}(x_{d-1}|x_d), \\
&\qquad\vdots \\
S^*_{1}(x_1,\ldots,x_d) &= P(X_1\le x_1|X_{2}= x_{2},\ldots,X_d=x_d) = F_{1}(x_{1}|x_2,\ldots,x_d), \\
\end{align*}
When $d=2$ and $f$ is a bivariate Gaussian 
$N\left(\begin{bmatrix} \mu_1 \\ \mu_2 \end{bmatrix}, \begin{bmatrix} \sigma_1^2 & \rho\sigma_1\sigma_2 \\ \rho\sigma_1\sigma_2 & \sigma_2^2  \end{bmatrix}\right)$, we have
\begin{align*}
F_2(x_2) &= \Phi\left(\frac{x_2-\mu_2}{\sigma_2}\right),\\
F_1(x_1|x_2) &= \Phi\left(\frac{x_1-\mu_1-\frac{\rho\sigma_2}{\sigma_1}(x_2-\mu_2)}{\sigma_1\sqrt{1-\rho^2}}\right),
\end{align*}
where $\Phi$ is the univariate standard normal cdf. Figure \ref{fig:bvnorm} exhibits heat maps of the target density $f(x_1,x_2)$ and the first component of the KR map $F_1(x_1|x_2)$ for the choice $\mu_1=\mu_2=0, \sigma_1=\sigma_2=1,$ with $\rho=0$ (left panels) and $\rho=0.7$ (right panels). 

\paragraph{Kullback-Leibler objective.}
The authors of~\cite{spantini} proposed to estimate the KR map using a minimum distance approach based on the Kullback-Leibler divergence.
Since $S^*$ is the unique monotone upper triangular map satisfying $S\# f=g$, it is the unique such map satisfying $\text{KL}(S\#f|g)=0$, where KL denotes the Kullback-Leibler divergence,
\[
\text{KL}(p|q) = \int p(x)\log\frac{p(x)}{q(x)}dx
\] 
for Lebesgue densities $p,q$ on $\mathbb{R}^d$. As such, $S^*$ uniquely solves the variational problem
$
\min_{S\in\mathcal{T}}\text{KL}(S\# f|g),
$
where $\mathcal{T}$ denotes the convex cone of monotone non-decreasing upper triangular maps $S:\mathbb{R}^d\to\mathbb{R}^d$. The change of variables formula for densities states that
\begin{equation}
(S^{-1}\# g)(x) = g(S(x))|\det(JS(x))|,
\label{eq:COV}
\end{equation}
where $JS(x)$ denotes the Jacobian matrix of $S$ evaluated at $x$. Applying this formula to $\text{KL}(S\#f|g)=\text{KL}(f|S^{-1}\# g)$, we rewrite $\min_{S\in\mathcal{T}}\text{KL}(S\# f|g)$ as 
\begin{equation}
\min_{S\in\mathcal{T}} \E\left\{\log [f(X)/g(S(X))] - \sum_{k=1}^d \log D_k S_k(X)\right\},   
\label{min:pop}
\end{equation}
where $X$ is a random variable on $\mathbb{R}^d$ with density $f$ and $D_k$ is shorthand for differentiation with respect to the $k$th component $\frac{\partial}{\partial x_k}$. 
By monotonicity, $D_k S_k(x)$ is defined Lebesgue almost everywhere for every $S\in\mathcal{T}$. The relation (\ref{min:pop}) is proved in the Supplement.
\subsection{Statistical estimator of the KR map}

We shall study an estimator $S^n$ of $S^*$ derived from the sample average approximation to (\ref{min:pop}), which yields the minimization problem~\cite{spantini}
\begin{equation}
\min_{S\in\mathcal{S}}
\frac{1}{n}\sum_{i=1}^n \left\{ \log [f(X^i)/g(S(X^i))] - \sum_{k=1}^d \log D_k S_k(X^i) \right\},
\label{min:sample}   
\end{equation}
where $X^1,\ldots,X^n$ is an i.i.d random sample from $f$ and $\mathcal{S}$ is a hypothesis function class. In generative modeling, we have a finite sample from $f$, perhaps an image dataset, that we use to train a map that can generate more samples from $f$. In this case, $f$ is the target density and $g$, the source density, is a degree of freedom in the problem. In practice, $g$ should be chosen so that it is easy to sample from, e.g., a multivariate normal density. 
The target density $f$ could be also unknown in practice, and if necessary we can omit the terms involving $f$ from the objective function in (\ref{min:sample}), since they do not depend on the argument $S$. 

With an estimator $S^n$ in hand, which approximately solves the sample average problem (\ref{min:sample}), we can generate approximate samples from $f$ by pulling back samples from $g$ under $S^n$, or equivalently by pushing forward samples from $g$ under $T^n =(S^n)^{-1}$. 
As $S^n$ is defined via KL projection, it can also be viewed as a non-parametric maximum likelihood estimator (MLE). 

In practice, $S^n$ can be estimated by parameterizing the space of triangular maps via some basis expansion, for example, using orthogonal polynomials~\cite{sampling}. In this case, the problem becomes parametric and $S^n$ is the MLE. Universal approximation is insufficient to reason about nonparametric estimators,  since slow rates can happen, as we shall show next.

Figure \ref{fig:ex} illustrates the problem at hand in the case where the target density $f$ is an unbalanced mixture of three bivariate normal distributions centered on the vertices of an equilateral triangle with spherical covariance. 
Panel (a) of Figure \ref{fig:ex} displays level curves of $f$. The source density $g$ is a standard bivariate normal density. We solved for $S^n$ by parametrizing its components via a Hermite polynomial basis, as described in \cite{sampling}. Since $g$ is log-concave, the optimization problem (\ref{min:sample}) is convex and can be minimized efficiently using standard convex solvers. By the change of variables formula, $f(x)=g(S^*(x))|\det(J S^*(x)|$, where $|\det(J S^*(x))|$ denotes the Jacobian determinant of the KR map $S^*$ from $f$ to $g$. Hence we take $f_n(x)=g(S^n(x))|\det(J S^n(x))|$ as an estimate of the target density. Panels (b)-(d) of Figure \ref{fig:ex} display the level curves of $f_n$ as the sample size $n$ increases from $1500$ to $5000$. 
The improving accuracy of the density estimates $f_n$ as the sample size grows is consistent with our convergence results established in Section \ref{sec:consistency}.

\begin{figure*}[tp]
\centering
\begin{tabular}{llll}
\includegraphics[height=35mm,width=0.22\textwidth]{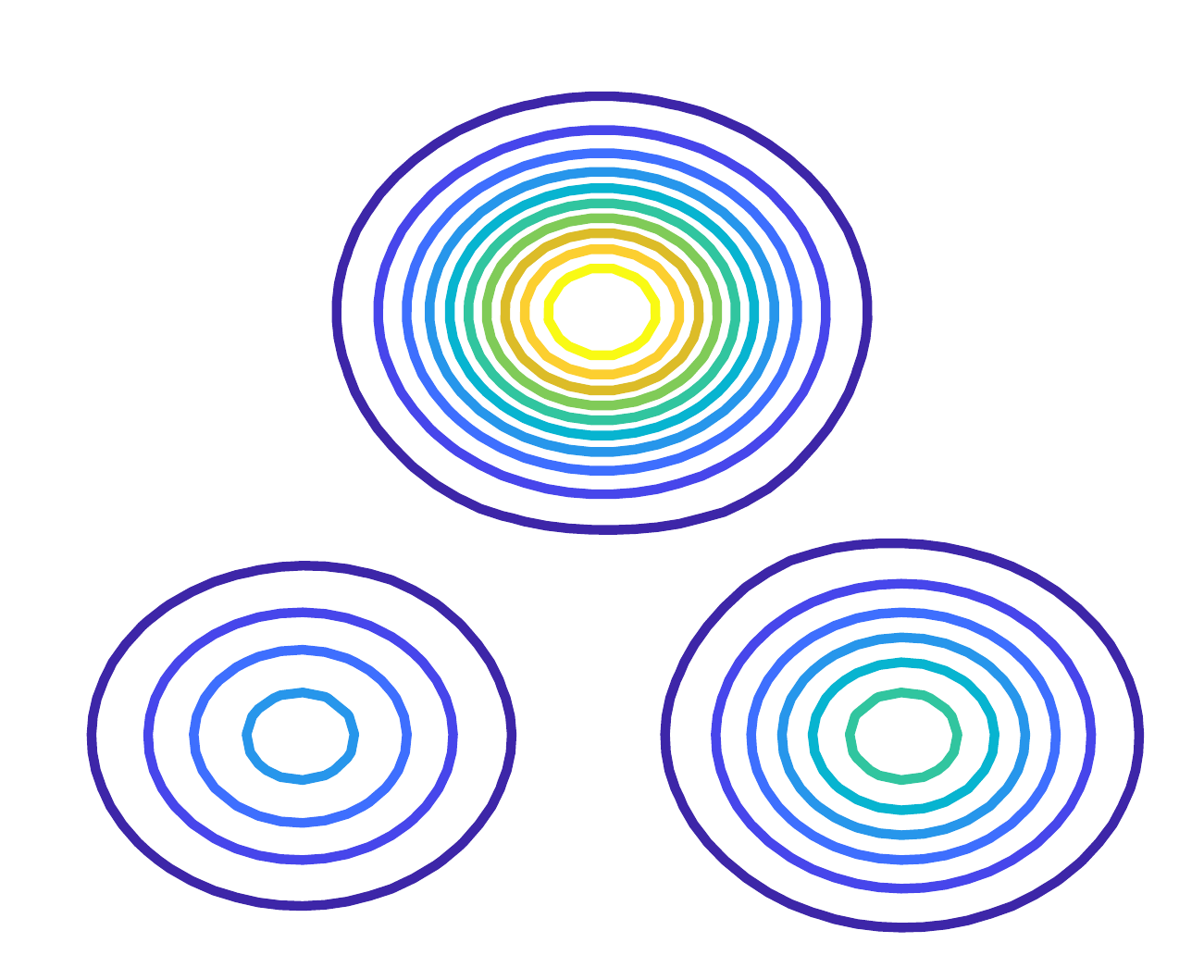}
&
\includegraphics[height=35mm,width=0.22\textwidth]{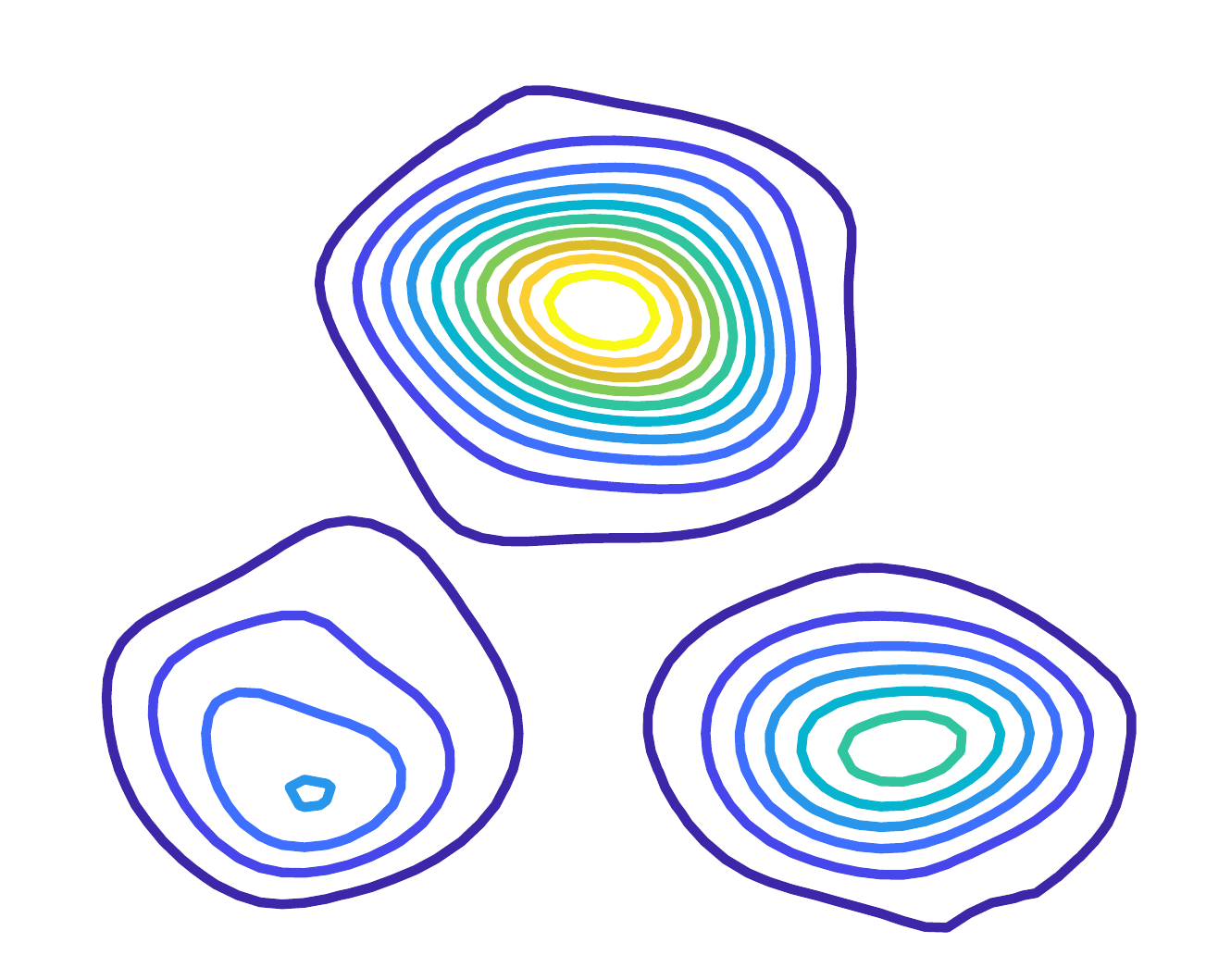}
&
\includegraphics[height=35mm,width=0.22\textwidth]{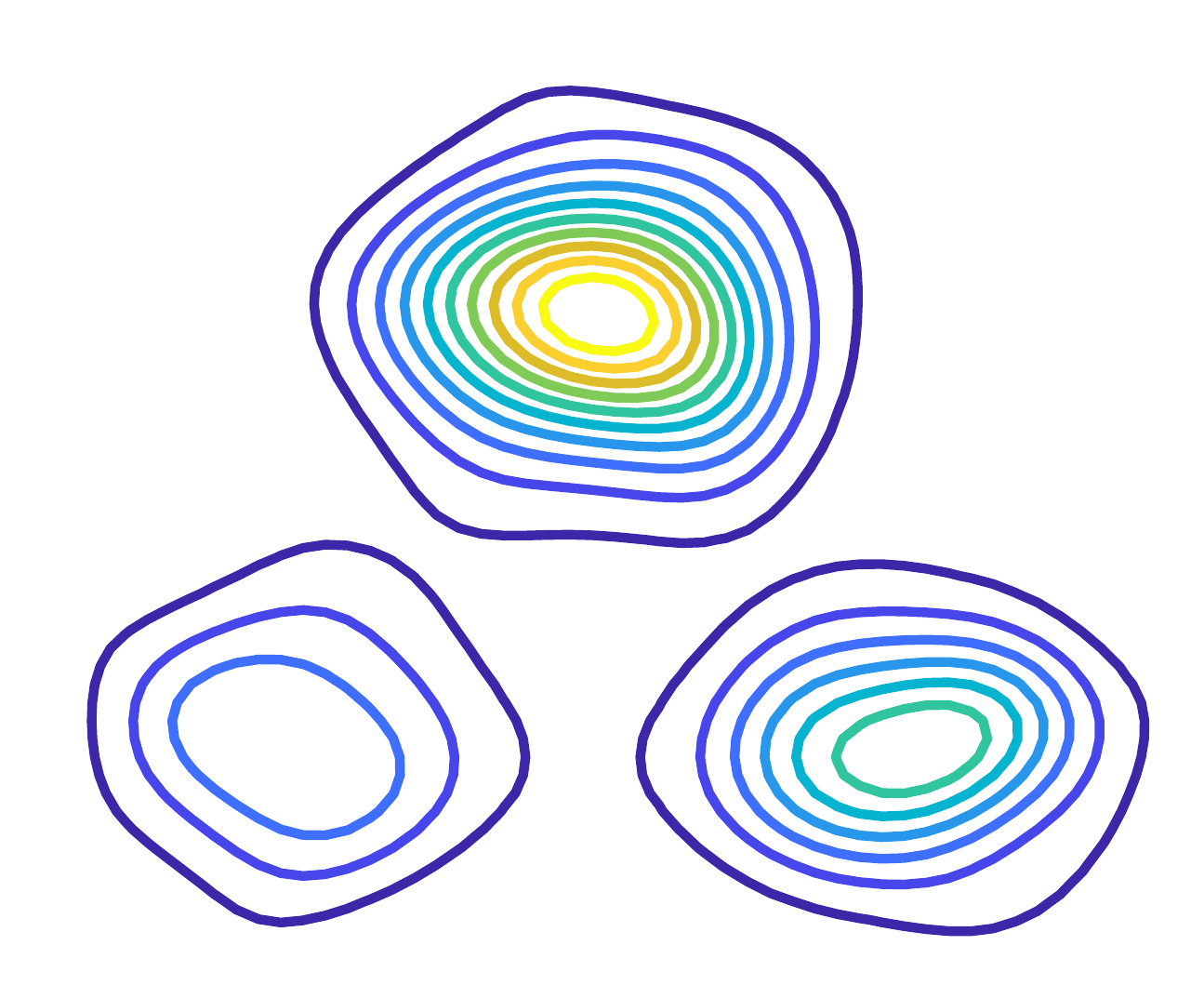}
&
\includegraphics[height=35mm,width=0.22\textwidth]{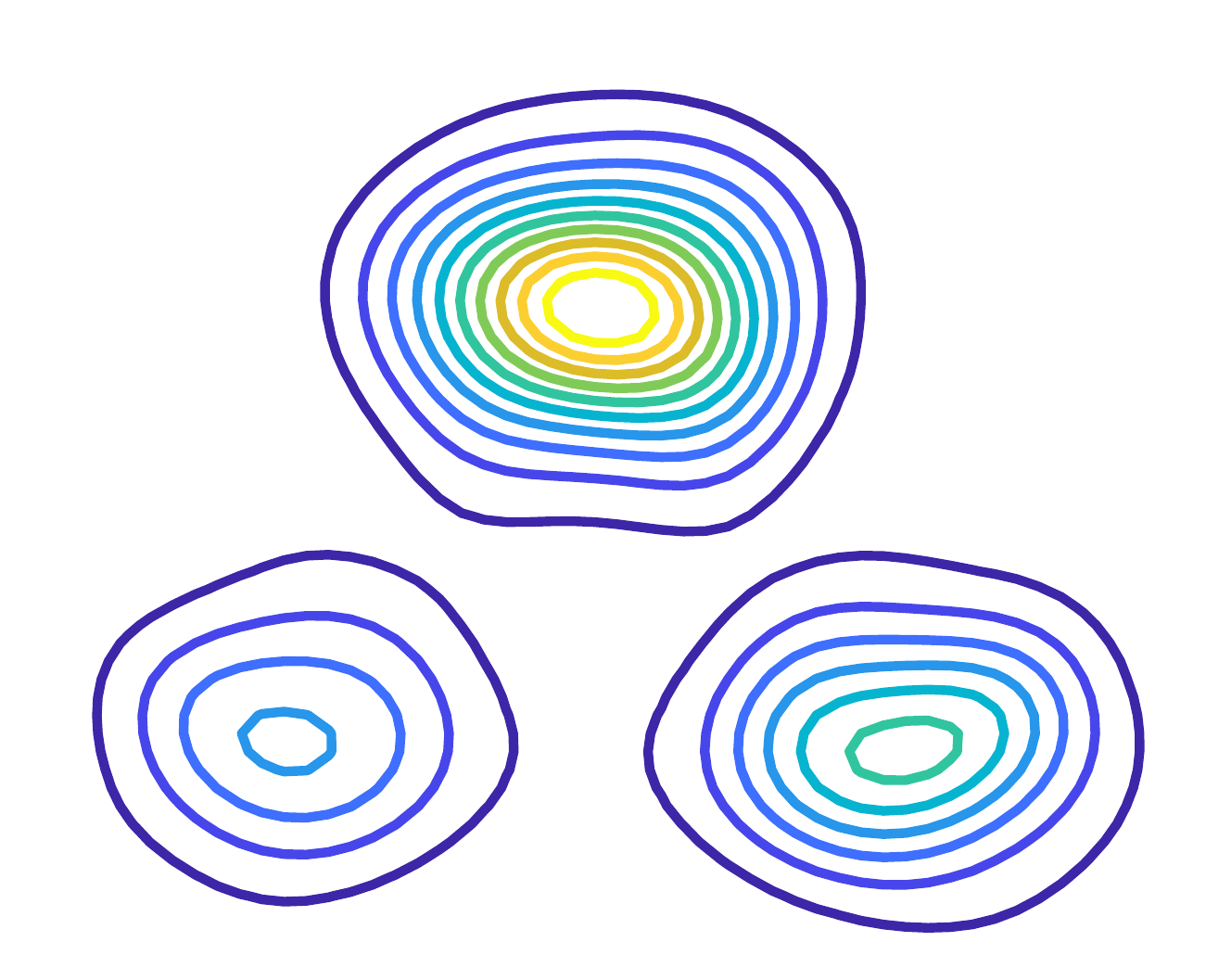} 
\\ 
\qquad (a) Truth & \qquad (b) $n=1500$ & \qquad (c) $n=3000$ & \qquad (d) $n=5000$
\end{tabular}
\caption{(a) Level curves of the normal mixture density $f(x)=g(S^*(x))|\det(JS^*(x))|$; (b)-(d) Level curves of the density estimates $f_n(x)=g(S^n(x))|\det(J S^n(x))|$ trained on $n=1500,3000,5000$ samples, respectively.}
\label{fig:ex}
\end{figure*}

\paragraph{Slow Rates.}

Without combining both a tail condition (e.g., common compact support) and a smooth regularity condition (e.g., uniformly bounded derivatives) on the function class $\mathcal{F}$ of the target density $f$, we show that convergence of any estimator $T^n$ of the direct map $T$ from $g$ to $f$ can occur at an arbitrarily slow rate.

\begin{theorem}
\label{thm-slow-rates}
Let $\mathcal{F}$ denote the class of infinitely continuously differentiable Lebesgue densities supported on the $d$-dimensional hypercube $[0,1]^d$ and uniformly bounded by 2, i.e.,
 $\sup_{f\in\mathcal{F}} \Vert f \Vert_\infty\le 2$. Let $g$ be any Lebesgue density on $\mathbb{R}^d$. 
 
 For any $n\in\mathbb{N}$, the minimax risk in terms of KL divergence is bounded below as
\[
\inf_{T^n}\sup_{f\in\mathcal{F}} \E_f[\text{KL}(f|f_n))] \ge 1/2,
\] 
where $T^n:\mathbb{R}^d\to[0,1]^d$ is any estimate of the KR map from $g$ to $f$ based on an iid sample of size $n$ from $f$, and $f_n = T^n\# g$ is the density estimate of $f$.
\label{thm:slow}
\end{theorem}

Theorem~\ref{thm-slow-rates} underscores the importance of going beyond universal approximation results to study the sample complexity and statistical performance of KR estimation.
The proof of this ``no free lunch'' theorem follows an idea of Birg\'{e}~\cite{birge}; see also~\cite{birge2,lugosi,devroye3,devroye2,devroye,gyorfi2002distribution}. We construct a family of densities in $\mathcal{F}$ built from rapidly oscillating perturbations of the uniform distribution on $[0,1]^d$.
Such densities are intuitively difficult to estimate. As is evident from the construction, however, a suitable uniform bound on the derivatives of the functions in $\mathcal{F}$ would preclude the existence of such pathological examples. As such, in what proceeds we aim to derive convergence rate bounds under the assumption that the target and source densities $f$ and $g$, respectively, are compactly supported and sufficiently regular, in the sense that they lie in a Sobolev space of functions with continuous and uniformly bounded partial derivatives up to order $s$ for some $s\ge 1$. For simplicity, our theoretical treatment assumes that $f$ and $g$ are fixed, but our convergence rate bounds as stated also bound the worst-case KL risk over any $f$ and $g$ lying in the $L^\infty$ Sobolev ball $\mathcal{F} = \left\{h:\sum_{|\mathbf{\alpha}|\le s}\|D^\mathbf{\alpha} h\|_\infty \le B\right\}$ for any fixed $B>0$.

Theorem \ref{thm-slow-rates} provides a lower bound on the minimax KL risk of KR estimation over the hypothesis class of target densities $\mathcal{F}$. 
The following stronger result, based on work of Devroye \cite{devroye3,devroye2,devroye}, demonstrates that convergence can still occur arbitrarily slowly for the task of estimating a single target density.

\begin{theorem}
Let $g$ be any Lebesgue density on $\mathbb{R}^d$ and $\{a_n\}_{n=1}^\infty$ any sequence converging to zero with $1/512\ge a_1\ge a_2\ge\cdots\ge 0.$ For every sequence of 
KR map estimates $T^n:\mathbb{R}^d\to\mathbb{R}^d$ based on a random sample of size $n$, there exists a target distribution $f$ on $\mathbb{R}^d$ such that
\[
\mathbb{E}_f[\text{KL}(f|f_n)] 
\ge a_n,
\]
where $f_n = T^n\# g$ is the density estimate of $f$.
\label{thm:slow2}
\end{theorem}
\section{Statistical Consistency}
\label{sec:consistency}
Setting the stage for the theoretical results, let us introduce the main assumptions.
\begin{assumption}
The Lebesgue densities $f$ and $g$ have convex compact supports $\mathcal{X},\mathcal{Y}\subset\mathbb{R}^d$, respectively.
\label{assumption:compact}
\end{assumption}

\begin{assumption}
The densities $f,g$ are bounded away from 0, i.e., $\inf_{x\in\mathcal{X},y\in\mathcal{Y}}\{f(x),g(y)\}>0$
\label{assumption:convex}
\end{assumption}

\begin{assumption}
Let $s\ge 1$ be a positive integer. The densities $f$ and $g$ are $s$-smooth on their supports, in the sense that
\[
D^\alpha f(x) := \frac{\partial^{|\alpha|}}{\partial x_1^{\alpha_1}\cdots\partial x_d^{\alpha_d}}f(x)
\]
is continuous for every multi-index $\mathbf{\alpha}=(\alpha_1,\ldots,\alpha_d)\in\mathbb{Z}^d_+$ satisfying
$|\mathbf{\alpha}| := \sum_{k=1}^d \alpha_k \le s$
and similarly for $g$. 
\label{assumption:smooth}
\end{assumption}

It is well known that the KR map $S^*$ from $f$ to $g$ is as smooth as the densities $f$ and $g$, but not smoother \cite{otam}. As such, under Assumptions \ref{assumption:compact}-\ref{assumption:smooth} we can restrict our attention from $\mathcal{T}$, the set of monotone non-decreasing upper triangular maps, to the smaller function class of monotone upper triangular maps that are $s$-smooth, of which the KR map $S^*$ is an element. That is, we can limit our search for an estimator $S^n$ solving (\ref{min:sample}) to a space of functions with more structure. This restriction is crucial to establishing a rate of convergence of the estimator $S^n$, as we can quantitatively bound the complexity of spaces of smooth maps. We discuss these developments in further detail below. Proofs of all results are included in the Supplement.

\subsection{Upper bounds on metric entropy}
\label{sec:entropy}
We first derive useful estimates of the metric entropy of function classes previously introduced. Assumptions \ref{assumption:compact}-\ref{assumption:smooth} allow us focus on smooth subsets of the class of monotone upper triangular maps $\mathcal{T}$.

\begin{definition} 
Let $M>0$. For $\mathfrak{s}\in\mathbb{Z}^d_+$ and $k\in[d]$, let $\tilde{\mathfrak{s}}_k=(s_k+1,s_{k+1},\ldots,s_d)$. Define $\mathcal{T}(\mathfrak{s},d,M)\subset\mathcal{T}$ as the convex subset of strictly increasing upper triangular maps $S:\mathcal{X}\to\mathcal{Y}$ satisfying:
\begin{enumerate}
    \item $\inf_{k\in[d],x\in\mathcal{X}}D_kS_k(x)\ge 1/M$, 
    \label{LB}
    \item $\|D^\alpha S_k\|_\infty\le M$ for all $k\in[d]$ and $\alpha_{k:d}\preceq\tilde{\mathfrak{s}}_k$.
    \label{UB}
\end{enumerate}
For $s\in\mathbb{N}$, we also define the homogeneous smoothness class $\mathcal{T}(s,d,M)=\mathcal{T}((s,\ldots,s),d,M)$.
\label{def:tsdm}
\end{definition}

Condition \ref{LB} guarantees that the Jacobian term in (\ref{min:sample}) is bounded, and condition \ref{UB} guarantees smoothness of the objective. In Section \ref{sec:anisotropic} below we consider densities with anisotropic smoothness, in which case the number of continuous derivatives $s_k$ varies with the coordinate $x_k$. For simplicity and clarity of exposition, we first focus on the case when $f$ and $g$ are smooth in a homogeneous sense, as in Assumption \ref{assumption:smooth}, and we work in the space $\mathcal{T}(s,d,M)$. As remarked above, the KR map $S^*$ from $f$ to $g$ lies in $\mathcal{T}(s,d,M^*)$ 
under Assumptions \ref{assumption:compact}-\ref{assumption:smooth}
when $M^*$ is sufficiently large. 
The same is true of the direct map $T^*$ from $g$ to $f$. In fact, all of the results stated here for the sampling map $S^*$ also hold for the direct map $T^*$, possibly with minor changes (although the proofs are generally more involved). For brevity, we mainly discuss $S^*$ and direct the interested reader to the Supplement.
Henceforth, we consider estimators $S^n$ lying in $\mathcal{T}(s,d,M^*)$ that minimize the objective in (\ref{min:sample}). We leave the issue of model selection, i.e., determining a sufficiently large $M^*$ such that $\mathcal{T}(s,d,M^*)$ contains the true KR map $S^*$, for future work. In the Supplement, we calculate explicit quantitative bounds on the complexity of this space as measured by the metric entropy in the $d$-dimensional $L^\infty$ norm 
$\|S\|_{\infty,d} := \max_{k\in[d]}\|S_k\|_\infty$.
The compactness of $(\overline{\mathcal{T}(s,d,M^*)},\|\cdot\|_{\infty,d})$ derived as a corollary of this result is required to establish the convergence of a sequence of estimators $S^n$ to $S^*$, and the entropy bound on the corresponding class of Kullback-Leibler loss functions over $\mathcal{T}(s,d,M^*)$ in Proposition \ref{prop:entbound} below allows us to go further by deriving bounds on the rate of convergence in KL divergence. This result builds off known entropy estimates for function spaces of Besov type \cite{birge2,birge,nickl}.

\begin{definition}
For a map
$S\in\mathcal{T}$,
define the loss function $\psi_S:\mathcal{X}\to\mathbb{R}$ by
\begin{align*}
\psi_S(x) 
&= \log f(x) - \log g(S(x)) - \log|\det(JS(x))|\\
&=\log f(x)-\log g(S(x)) - \sum_{k=1}^d \log D_kS_k(x).
\end{align*}
We also define the class $\Psi(s,d,M^*)$ of loss functions over $\mathcal{T}(s,d,M^*)$ as
\[
\Psi(s,d,M^*) := \{\psi_S:S\in\mathcal{T}(s,d,M^*)\}.
\]
\label{def:psi}
\end{definition}

By (\ref{min:pop}) we have $\E[\psi_S(X)] = \text{KL}(S\#f|g)$, where $X\sim f$. Similarly, the sample average of $\psi_S$ is the objective in (\ref{min:sample}). Hence, to derive finite sample bounds on the expected KL loss, we must study the sample complexity of the class $\Psi(s,d,M^*)$.
Define $N(\epsilon,\Psi(s,d,M^*),\|\cdot\|_{\infty})$ as the $\epsilon$-covering number of $\Psi(s,d,M^*)$ with respect to the uniform norm $\|\cdot\|_{\infty}$, and the metric entropy
\[
H(\epsilon,\Psi(s,d,M^*),\|\cdot\|_{\infty}) = \log N(\epsilon,\Psi(s,d,M^*),\|\cdot\|_{\infty}).
\]

\begin{prop}
Under Assumptions \ref{assumption:compact}-\ref{assumption:smooth}, 
the metric entropy of $\Psi(s,d,M^*)$ in $L^\infty(\mathcal{X})$ is bounded as
\[
H(\epsilon, \Psi(s,d,M^*),\|\cdot\|_\infty) 
\lesssim \epsilon^{-d/s}.
\]
Consequently, $\Psi(s,d,M^*)$ is totally bounded and therefore precompact in $L^\infty(\mathcal{X})$.
\label{prop:entbound}
\end{prop}

Here, for functions $a(\epsilon),b(\epsilon)$ (or sequences $a_n,b_n$) we write $a(\epsilon)\lesssim b(\epsilon)$ (resp. $a_n\lesssim b_n$) if $a(\epsilon)\le Cb(\epsilon)$ (resp. $a_n\le Cb_n$) for all $\epsilon$ (resp. $n$) for some constant $C>0$.
For brevity, in this result and those that follow, we suppress scalar prefactors that do not depend on the sample size $n$. As our calculations in the Supplement demonstrate, the constant prefactors in this and subsequent bounds are polynomial in the $\|\cdot\|_\infty$ radius $M^*$ and exponential in the dimension $d$. This dependence is similar to others
in the literature concerning sample complexity of transport map estimators \cite{rigollet}.

\subsection{Statistical consistency}
\label{sec:klconv}

For the sake of concision, we introduce empirical process notation~\cite{dudley67,dudley68,vdv-jaw,hds}. Let $\mathcal{H}$ be a collection of functions from $\mathcal{X}\subseteq\mathbb{R}^d\to\mathbb{R}$ measurable and square integrable with respect to $P$, a Borel probability measure on $\mathbb{R}^d$. Let $P_n$ denote the empirical distribution of an iid random sample $X_1,\ldots,X_n$ drawn from $P$. For a function $h\in\mathcal{H}$ we write $Ph := \E[h(X)]$, $P_nh := \frac{1}{n}\sum_{i=1}^n h(X_i)$, and $\|P_n-P\|_{\mathcal{H}} := \sup_{h\in\mathcal{H}}|(P_n-P)h|$.
As discussed in \cite{vdv-jaw}, measurability of the supremum process is a delicate subject. For simplicity, we gloss over this technicality in our theoretical treatment, but note that all results are valid as stated with outer expectation $\E^*$ replacing $\E$ where relevant. 

Let $P$ denote the probability measure with density $f$. With these new definitions, the sample average minimization objective in (\ref{min:sample}) can be expressed $P_n\psi_S$,
while the population counterpart in (\ref{min:pop}) reads as $P\psi_S =\text{KL}(S\# f|g)$.
Suppose the estimator $S^n$ of the sampling map $S^*$ is a random element in $\mathcal{T}(s,d,M^*)$ obtained as a near-minimizer of $P_n\psi_S$. Let
\[
R_n=P_n\psi_{S^n} - \inf_{S\in\mathcal{T}(s,d,M)}P_n\psi_S \ge 0
\]
denote the approximation error of our optimization algorithm.
Our goal is to bound the loss $P\psi_{S^n}$. 
Fix $\epsilon > 0$ and let $\tilde{S}$ be any deterministic element of $\mathcal{T}(s,d,M^*)$ that nearly minimizes $P\psi_S$, i.e., suppose
\[
P\psi_{\tilde{S}} \le \inf_{S\in\mathcal{T}(s,d,M^*)}P\psi_S + \epsilon.
\]
It follows that
\begin{align*}
P\psi_{S^n}-\inf_{S\in\mathcal{T}(s,d,M^*)} P\psi_{S} &\le P\psi_{S^n}-P\psi_{\tilde{S}}+\epsilon\\
&= [P\psi_{S^n}-P\psi_{\tilde{S}}]+[P_n\psi_{\tilde{S}}-P_n\psi_{S^n}]+R_n+\epsilon \\
&\le 2\|P_n-P\|_{\Psi(s,d,M^*)}+R_n+\epsilon.
\end{align*}
As $\epsilon > 0$ was arbitrary, we conclude that
\begin{equation}
P\psi_{S^n}-\inf_{S\in\mathcal{T}(s,d,M^*)} P\psi_{S} \le 2\|P_n-P\|_{\Psi(s,d,M^*)}+ R_n.
\label{eq:lossbound}
\end{equation}
Controlling the deviations of the empirical process $\|P_n-P\|_{\Psi(s,d,M^*)}$ as in Lemma \ref{lemma:EPrate} allows us to bound the loss of the estimator $S^n$ and establish consistency and a rate of convergence in KL divergence.

\begin{lemma}
Under Assumptions \ref{assumption:compact}-\ref{assumption:smooth}, we have
\[
\E\|P_n-P\|_{\Psi(s,d,M^*)} \lesssim \begin{cases}
n^{-1/2}, & d < 2s, \\
n^{-1/2}\log n, &d=2s, \\
n^{-s/d}, &d>2s.
\end{cases}
\]
\label{lemma:EPrate}
\end{lemma}

\begin{rmq}
Consider the case where $2s>d$, for example when  both $f,~g$ are the densities of the standard normal distribution, then we have by the central limit theorem that for any $S\in \mathcal{T}(s,d,M^*)$, $\sqrt{n} \left[P_n\psi_S-P\psi_S\right]$ converge in law towards a centered Gaussian distribution. Therefore we have that for any $S\in \mathcal{T}(s,d,M^*)$
\begin{align*}
  \E\|P_n-P\|_{\Psi(s,d,M^*)}\geq \sup\limits_{S\in\mathcal{T}(s,d,M^*)}\E\left[|P_n\psi_S-P\psi_S|\right]\gtrsim n^{-1/2}
\end{align*}
which shows that our results are tight at least in the smooth regime.
\end{rmq}

The proof of Lemma \ref{lemma:EPrate} relies on metric entropy integral bounds established by Dudley in \cite{dudley67} and van der Vaart and Wellner in \cite{vdv-jaw}.
Although we have phrased the sample complexity bounds in Lemma \ref{lemma:EPrate} in terms of the expectation of the empirical process $\|P_n-P\|_{\Psi(s,d,M)}$, high probability bounds can be obtained similarly~\cite{hds}. 
Hence, the following KL consistency theorem is obtained as a direct result of Lemma \ref{lemma:EPrate} and the risk decomposition (\ref{eq:lossbound}). 
\begin{theorem}
Suppose Assumptions \ref{assumption:compact}-\ref{assumption:smooth} hold. 
Let $S^n$ be a near-optimizer of the functional $S\mapsto P_n\psi_S$ on $\mathcal{T}(s,d,M^*)$ with remainder $R_n$ given by
\[
R_n = P_n\psi_{S^n} - \inf_{S\in\mathcal{T}(s,d,M^*)} P_n\psi_S = o_P(1).
\]
Then $P\psi_{S^n} \stackrel{p}{\to}P\psi_{S^*}=0$, i.e., $S^n$ is a consistent estimator of $S^*$ with respect to KL divergence. 

Moreover, if $R_n$ is
bounded in expectation as 
\[
\E[R_n] \lesssim \E\|P_n-P\|_{\Psi(s,d,M^*)},
\]
then the expected KL divergence of $S^n$ is bounded as
\[
\E[P\psi_{S^n}] \lesssim \begin{cases}
n^{-1/2}, & d < 2s, \\
n^{-1/2}\log n, &d=2s, \\
n^{-s/d}, &d>2s.
\end{cases}
\]
\label{thm:KLconv}
\end{theorem}

\begin{rmq}
Note that, as we work on a compact set, we also obtain the rates of convergence of $S^n$ with respect to the Wasserstein metric:
\[
\E[W(S^n\# f,g)] \lesssim \begin{cases}
n^{-1/4}, & d < 2s, \\
n^{-1/4}\log n, &d=2s, \\
n^{-s/2d}, &d>2s.
\end{cases}
\]
where for any probability measures $\mu$ and $\nu$ on $\mathbb{R}^d$ with finite first moments, we denote:
\begin{align*}
    W(\mu,\nu):=\sup\left\{\left| \int hd\mu - \int h d\nu\right|~:|h(x)-h(y)|\leq \Vert x-y\Vert\right\}.
\end{align*}
\begin{proof}
Let $\mu$ and $\nu$ be two probability measures on $\mathcal{Z}:=\mathcal{X}\cup\mathcal{Y}\subset\mathbb{R}^d$. Then as both $\mathcal{X}$ and $\mathcal{Y}$ are compact, we have the following well-known inequalities \cite{villani}:
\begin{align*}
    W(\mu,\nu)\leq \text{diam}(\mathcal{Z})\text{TV}(\mu,\nu)\leq \frac{\text{diam}(\mathcal{Z})}{\sqrt{2}}\sqrt{\text{KL}(\mu|\nu)}\; ,
\end{align*}
then applying Jensen's inequality to our results yields control in the Wasserstein sense.
\end{proof}
\end{rmq}

\paragraph{Uniform convergence.}
Although Theorem \ref{thm:KLconv} only establishes a weak form of consistency in terms of the KL divergence, we leverage this result to prove strong consistency, in the sense of uniform convergence of $S^n$ to $S^*$ in probability, in Theorem \ref{thm:uniformconv}. The proof requires understanding the regularity of the KL divergence with respect to the topology induced by the $\|\cdot\|_{\infty,d}$ norm. 
In the Supplement, we establish that KL is lower semicontinuous in $\|\cdot\|_{\infty,d}$. Our result relies on the weak lower semicontinuity of KL proved by Donsker and Varadhan using their well known dual representation in Lemma 2.1 of \cite{donsker-varadhan}.

\begin{theorem}
Suppose Assumptions \ref{assumption:compact}-\ref{assumption:smooth} hold.
Let $S^n$ be any near-optimizer of the functional $S\mapsto P_n\psi_S$ on $\mathcal{T}(s,d,M^*)$, i.e., suppose
\[
P_n\psi_{S^n} = \inf_{S\in\mathcal{T}(s,d,M^*)} P_n\psi_S+o_P(1).
\]
Then $\|S^n-S^*\|_{\infty,d}\stackrel{p}{\to}0$, i.e, $S^n$ is a consistent estimator of $S^*$ with respect to the uniform norm $\|\cdot\|_{\infty,d}$ .
\label{thm:uniformconv}
\end{theorem}

\begin{proof}
Recall that $\overline{\mathcal{T}(s,d,M^*)}$ is compact with respect to $\|\cdot\|_{\infty,d}$. 
Together, lower semicontinuity of KL in $\|\cdot\|_{\infty,d}$ 
and compactness guarantee that the KR map $S^*$, which is the unique minimizer of $S\mapsto P\psi_S$ over $\overline{\mathcal{T}(s,d,M^*)}$, is well-separated in $\mathcal{T}(s,d,M^*)$. 
In other words, for any $\epsilon>0$,
\[
P\psi_{S^*} < \inf_{S\in\mathcal{T}(s,d,M^*):\|S-S^*\|_{\infty,d}\ge\epsilon} P\psi_S.
\]
Indeed, suppose to the contrary that we can find a deterministic sequence $\tilde{S}^n\in \mathcal{T}(s,d,M^*)$ satisfying $\|\tilde{S}^n-S^*\|_{\infty,d}\ge\epsilon$ such that $P\psi_{\tilde{S}^n}\to P\psi_{S^*}$. Since $\mathcal{T}(s,d,M^*)$ is precompact with respect to $\|\cdot\|_{\infty,d}$, we can extract a subsequence $\tilde{S}^{n_k}$ converging to some $\tilde{S}^*\in\overline{\mathcal{T}(s,d,M^*)}$, which necessarily satisfies $\|\tilde{S}^*-S^*\|_{\infty,d}\ge\epsilon$. By lower semicontinuity of $S\mapsto P\psi_S$ with respect to $\|\cdot\|_{\infty,d}$, it follows that
\[
P\psi_{\tilde{S}^*} \le \liminf_{k\to\infty}P\psi_{\tilde{S}^{n_k}} = P\psi_{S^*}.
\]
We have a contradiction, since the KR map $S^*$ is the unique minimizer of $S\mapsto P\psi_S$. This proves the claim that $S^*$ is a well-separated minimizer.

Now fix $\epsilon>0$ and define
\[
\delta = \inf_{S\in\mathcal{T}(s,d,M^*):\|S-S^*\|_{\infty,d}\ge\epsilon}P\psi_S-P\psi_{S^*}>0.
\]
It follows that 
\[
\{\|S^n-S^*\|_{\infty,d}\ge\epsilon\}\subseteq\{P\psi_{S^n}-P\psi_{S^*}\ge\delta\}.
\]
We have shown $P\psi_{S^n}\stackrel{p}{\to}P\psi_{S^*}$ in Theorem \ref{thm:KLconv}. As a consequence,
\[
P(\|S^n-S^*\|_{\infty,d}\ge\epsilon) \le P(P\psi_{S^n}-P\psi_{S^*}\ge\delta)\to 0.
\]
As $\epsilon>0$ was arbitrary, we have $\|S^n-S^*\|_{\infty,d}\stackrel{p}{\to}0$.
\end{proof}

\paragraph{Inverse consistency.}
\label{sec:invconv}
We have proved consistency of the estimator $S^n$ of the sampling map $S^*$, pushing forward the target $f$ to the source $g$.
We can also get the consistency and an identical rate of convergence of $T^n=(S^n)^{-1}$ estimating $T^*=(S^*)^{-1}$, although the proof of the analog to Theorem \ref{thm:uniformconv} establishing uniform consistency of $T^n$ is much more involved. We defer to the Supplement for details.

\section{Log-Concavity, Dimension Ordering, Jacobian Flows}
\label{sec:fast}
\paragraph{Sobolev-type rates under log-concavity.} Suppose the source density $g$ is log-concave. Then $\min_{S\in\mathcal{T}(s,d,M)} P_n\psi_S$
is a convex problem; moreover if  $g$ is strongly log-concave, strong convexity follows. The user can choose a convenient $g$, such as a multivariate Gaussian with support truncated to a compact convex set. In this case, we can establish a bound on the rate of convergence of $S^n$ to $S^*$ in the $L^2$ Sobolev-type norm
\[
\|S\|^2_{H_f^{1,2}(\mathcal{X})} := \sum_{k=1}^d\left\{\|S_k\|^2_{L^2_f(\mathcal{X})}+\|D_kS_k\|^2_{L^2_f(\mathcal{X})}\right\}.
\]
Here $\|\cdot\|_{L^2_f(\mathcal{X})}$ denotes the usual $L^2$ norm integrating against the target density $f$.
\begin{theorem}
Suppose Assumptions \ref{assumption:compact}-\ref{assumption:smooth} hold. Assume further that $g$ is
$m$-strongly log-concave. 
Let $S^n$ be a near-optimizer of the functional $S\mapsto P_n\psi_S$ on $\mathcal{T}(s,d,M^*)$ with remainder $R_n$ satisfying
\[
\E[R_n] 
= \E\left\{P_n\psi_{S^n} - \inf_{S\in\mathcal{T}(s,d,M^*)} P_n\psi_S\right\} 
\lesssim \E\|P_n-P\|_{\Psi(s,d,M^*)}.
\]
Then $S^n$ converges to the true sampling map $S^*$ with respect to the norm $H^{1,2}_f(\mathcal{X})$ norm with rate
\[
\E\|S^n-S^*\|_{H^{1,2}_f(\mathcal{X})}^2 \lesssim \begin{cases}
n^{-1/2}, & d < 2s, \\
n^{-1/2}\log n, &d=2s, \\
n^{-s/d}, &d>2s,
\end{cases}
\]
\label{thm:KRrate}
\end{theorem}

With more work, we can establish Sobolev convergence rates of the same order in $n$ for $T^n=(S^n)^{-1}$ to $T^*=(S^*)^{-1}$, yet now in the appropriate norm $\|\cdot\|_{H^{1,2}_g(\mathcal{Y})}$; see details in the Supplement.
\begin{rmq}
Note that here we obtain the rates with respect to the $H^{1,2}_f(\mathcal{X})$ norm from which we deduce immediately similar rates for the $H^{1,p}_f(\mathcal{X})$ norm with $p\geq 2$. In addition, for fixed $p$, under Assumptions \ref{assumption:compact}-\ref{assumption:smooth}, the $H^{1,p}_f(\mathcal{X})$ norm is equivalent to $H^{1,p}(\mathcal{X})$ norm with the Lebesgue measure as the reference measure. Finally, here we show the rates of the strong consistency of our estimator with respect $H^{1,\cdot}_f(\mathcal{X})$ norm as we do not require any additional assumptions on the higher order derivatives of the maps living in $\mathcal{T}(s,d,M^*)$. Additional assumptions on these derivatives may lead to the convergence rates of higher order derivatives of our estimate with respect to the $H^{k,p}_f(\mathcal{X}),k>1$ norm which is beyond the scope of the present paper.
\end{rmq}

\paragraph{Dimension ordering.}
\label{sec:anisotropic}
Suppose now that the smoothness of the target density $f$ is anisotropic. That is, assume $f(x_1,\ldots,x_d)$ is $s_k$-smooth in $x_k$ for each $k\in[d]$. As there are $d!$ possible ways to order the coordinates, the question arises: how should we arrange $(x_1,\ldots,x_d)$ such that the estimator $S^n$ converges to the true KR map $S^*$ at the fastest possible rate? 
Papamakarios et al. (2017) provide a discussion of this issue in the context of autoregressive flows in Section 2.1 of \cite{maf}; they construct a simple 2D example in which the model fails to learn the target density if the wrong order of the variables is chosen.

This relates to choices made in neural architectures for normalizing flows on images and texts. Our results suggest here that one would rather start with the coordinates (i.e. data parts) that are the least smooth and make their way through the triangular construction to the most smooth ones.

We formalize the anisotropic smoothness of the KR map as follows.

\begin{assumption}
Let $\mathfrak{s}=(s_1,\ldots,s_d)\in \mathbb{Z}^d_+$ be a multi-index with $s_k\ge 1$ for all $k\in[d]$. The density $f$ is $\mathfrak{s}$-smooth on its support, in the sense that 
\[
D^\alpha f(x) := \frac{\partial^{|\alpha|}}{\partial x_1^{\alpha_1}\cdots\partial x_d^{\alpha_d}}f(x)
\]
exists and is continuous for every multi-index $\alpha=(\alpha_1,\ldots,\alpha_d)\in\mathbb{Z}^d_+$ satisfying $\alpha\preceq \mathfrak{s}$, i.e., $\alpha_k\le s_k$ for every $k\in[d]$. Furthermore, we assume that $g(y)$ is 
$(\|\mathfrak{s}\|_\infty,\ldots,\|\mathfrak{s}\|_\infty)$-smooth
with respect to $y=(y_1,\ldots,y_d)$ on $\mathcal{Y}$.
\label{assumption:multismooth}
\end{assumption}

As the source density $g$ is a degree of freedom in the problem, we are free to impose this assumption on $g$. Note that $(\|\mathfrak{s}\|_\infty,\ldots,\|\mathfrak{s}\|_\infty)$-smoothness of $g$ as defined in Assumption \ref{assumption:multismooth} is equivalent to $\|\mathfrak{s}\|_\infty$-smoothness of $g$ as defined in Assumption \ref{assumption:smooth}. The results that follow are slight variations on those in Sections \ref{sec:entropy} and \ref{sec:klconv} adapted to the anisotropic smoothness of the densities posited in Assumption \ref{assumption:multismooth}. 

Under Assumptions \ref{assumption:compact}, \ref{assumption:convex}, and \ref{assumption:multismooth}, there exists some $M^*>0$ such that the KR map $S^*$ from $f$ to $g$ lies in $\mathcal{T}(\mathfrak{s},d,M^*)$; see the Supplement for a proof.
We also define the class of loss functions
\[
\Psi(\mathfrak{s},d,M^*) = \{\psi_S:S\in\mathcal{T}(\mathfrak{s},d,M^*)\},
\]
which appear in the objective $P_n\psi_S$.
Hence, we can proceed as above to bound the metric entropy and obtain uniform convergence and Sobolev-type rates for estimators in the function class $\mathcal{T}(\mathfrak{s},d,M^*)$. Appealing to metric entropy bounds for anisotropic smoothness classes \cite[Proposition 2.2]{birge}, we have the following analog of Lemma \ref{lemma:EPrate}.

\begin{lemma}
For $k\in[d]$, let $d_k = d-k+1$ and $\sigma_k = d_k\left(\sum_{j=k}^d s_j^{-1}\right)^{-1}$.
Under Assumptions \ref{assumption:compact}, \ref{assumption:convex}, and \ref{assumption:multismooth}, 
we have
\[
\E\|P_n-P\|_{\Psi(\mathfrak{s},d,M^*)} \lesssim \sum_{k=1}^d c_{n,k}.
\]
where we define
\[
c_{n,k}= 
\begin{cases}
n^{-1/2}, & d_k < 2\sigma_k, \\
n^{-1/2}\log n, &d_k=2\sigma_k, \\
n^{-\sigma_k/d_k}, &d_k>2\sigma_k.
\end{cases}
\]
\label{lemma:anisorate}
\end{lemma}

Lemma \ref{lemma:anisorate} is proved via a chain rule decomposition of relative entropy which relies upon the triangularity of the hypothesis maps $\mathcal{T}(\mathfrak{s},d,M^*)$. From here we can repeat the analysis in Section \ref{sec:klconv} to obtain consistency and bounds on the rate of convergence of the estimators $S^n$ and $T^n=(S^n)^{-1}$ of the sampling map $S^*$ and the direct map $T^*=(S^*)^{-1}$, respectively, in the anisotropic smoothness setting of Assumption \ref{assumption:multismooth}. All the results in these Sections are true with $\sum_{k}c_{n,k}$ replacing the rate under isotropic smoothness.

In order to minimize this bound to obtain an optimal rate of convergence, we should order the coordinates $(x_1,\ldots,x_d)$ such that $\sigma_k$ is as large as possible for each $k\in[d]$. Inspecting the definition of $\sigma_k$ in Lemma \ref{lemma:anisorate}, we see that this occurs when
$
s_1\le\cdots\le s_d.
$
\begin{theorem}
The bound on the rate of convergence $\sum_k c_{n,k}$
is minimized when $s_1\le\cdots\le s_d$, i.e., when the smoothness of the target density $f$ in the direction $x_k$ increases with $1\le k\le d$.
\label{thm:order}
\end{theorem}

Our result on the optimal ordering of coordinates complements the following theorem of Carlier, Galichon, and Santambrogio adapted to our setup.

\begin{theorem}[Theorem 2.1; Carlier, Galichon, and Santambrogio (2008) \cite{santambrogio}]
Let $f$ and $g$ be compactly supported Lebesgue densities on $\mathbb{R}^d$. Let $\epsilon > 0$ and let $\gamma^\epsilon$ be an optimal transport plan between $f$ and $g$ for the cost 
\[
c_\epsilon (x,y)=\sum_{k=1}^d \lambda_k(\epsilon)(x_k-y_k)^2,
\]
for some weights $\lambda_k(\epsilon)>0$. Suppose that for all $k\in\{1,\ldots,d-1\}$, $\lambda_k(\epsilon)/\lambda_{k+1}(\epsilon)\to 0$ as $\epsilon\to 0$. Let $S^*$ be the Kn\"othe-Rosenblatt map between $f$ and $g$ and $\gamma^* = (\text{id}\times S^*)\# f$ the associated transport plan. Then $\gamma^\epsilon\rightsquigarrow \gamma^*$ as $\epsilon \to 0$. Moreover, should the plans $\gamma^\epsilon$ be induced by transport maps $S^\epsilon$, then these maps would converge to $S^*$ in $L^2(f)$ as $\epsilon\to 0$. 
\label{thm:santambrogio}
\end{theorem}

With this theorem in mind, the KR map $S^*$ can be viewed as a limit of optimal transport maps $S^\epsilon$ for which transport in the $d$th direction is more costly than in the $(d-1)$th, and so on. The anisotropic cost function $c_{\epsilon}(x,y)$ inherently promotes increasing regularity of $S^\epsilon$ in $x_k$ for larger $k\in[d]$. Theorem \ref{thm:order} establishes the same heuristic for learning triangular flows based on Kn\"othe-Rosenblatt rearrangement to build generative models.

In particular, our result suggests that we should order the coordinates such that $f$ is smoothest in the $x_d$ direction. This makes sense intuitively because the estimates of the component maps $S^*_k,k\in[d]$ all depend on the $d$th coordinate of the data. As such, we should leverage smoothness in $x_d$ to stabilize all of our estimates as much as possible. If not, we risk propagating error throughout the estimation. In comparison, only the estimate of the first component $S^*_1$ depends on $x_1$. Since our estimator depends comparatively little on $x_1$, we should order the coordinates so that $f$ is least smooth in the $x_1$ direction.

\paragraph{Jacobian Flows.}

Suppose now we solve
\[
\inf_{S}\text{KL}(S\# f|g) = \inf_S P\psi_S
\]
where the candidate map $S$ is a composition of smooth monotone increasing upper triangular maps $U^j$ and orthogonal linear transformations $\Sigma^j$ for $j\in[m]$, i.e.,
\begin{equation}
S(x) = U^m \circ \Sigma^m \circ\cdots\circ U^1\circ \Sigma^1 (x).
\label{eq:flow}
\end{equation}
We call $S$ a \textit{Jacobian flow} of order $m$.
This model captures many popular normalizing autoregressive flows \cite{kobyzev}, such as Real NVP \cite{nvp}, in which the $\Sigma^j$ are ``masking'' permutation matrices.

For simplicity, in this section we assume that $f$ and $g$ are supported on the unit ball $B_d(0,1)\subset\mathbb{R}^d$ centered at the origin. Since $\Sigma^j$ are orthogonal, we have $\Sigma^j(B_d(0,1))=B_d(0,1)$. Hence, to accommodate the setup of the preceding sections, we can guarantee that $S$ maps from $\mathcal{X}=B_d(0,1)$ to $\mathcal{Y}=B_d(0,1)$ by requiring $\sup_{x\in B_d(0,1)}\|U^j(x)\|_2\le 1$ for $j\in[m]$.

\begin{definition}
Define the class $\mathcal{J}_m(\Sigma,s,M)$ of $s$-smooth Jacobian flows of order $m$ to consist of those maps $S$ of the form (\ref{eq:flow}) such that
\begin{enumerate}
    \item $\Sigma=(\Sigma^1,\ldots,\Sigma^m)$ are fixed orthogonal matrices, 
    \item $\sup_{x\in B_d(0,1)}\|U^j(x)\|_2\le 1$ for $j\in[m]$,
    \item $U^j\in\mathcal{T}(s,d,M)$ for $j\in[m]$.
\end{enumerate}
We also define $\Psi_m(\Sigma,s,M) = \{\psi_S:S\in\mathcal{J}_m(\Sigma,s,M)\}$.
\label{def:flow}
\end{definition}

By expanding our search to $\mathcal{J}_m(\Sigma,s,M)$, we are not targeting the KR map $S^*$ and we are no longer guaranteed uniqueness of a minimizer of the KL. Nevertheless, we can study the performance of estimators $S^n\in\mathcal{J}_m(\Sigma,s,M)$ as compared to minimizers of KL in $\overline{\mathcal{J}_m(\Sigma,s,M)}$, which are guaranteed to exist by compactness of $\overline{\mathcal{J}_m(\Sigma,s,M)}$ and lower semicontinuity of KL in $\|\cdot\|_{\infty,d}$.

Since the $\Sigma^j$ are orthogonal, we have $|\det(J\Sigma^j(x))| = |\det(\Sigma^j)|=1$, and therefore
\begin{align*}
\psi_S(x) 
= & \log [f(x)/g(S(x))] - \sum_{j=1}^m \sum_{k=1}^d \log D_kU_k^j(x^j),
\end{align*}
where we define
\[
x^j = \Sigma^j \circ U^{j-1} \circ \Sigma^{j-1}\circ\cdots\circ U^1\circ \Sigma^1(x), \quad j\in[m].
\]

Hence, with a loss decomposition mirroring that in Definition \ref{def:psi}, we can apply the methods of the preceding sections to establish quantitative limits on the loss incurred by the estimates $S^n$ in finite samples. See the Supplement for details.

\begin{theorem}
Suppose $f,g$ are $s$-smooth and supported on $B_d(0,1)$. 
Let $S^n$ be a near-optimizer of the functional $S\mapsto P_n\psi_S$ on $\mathcal{J}_m(\Sigma,s,M)$ with 
\[
R_n = P_n\psi_{S^n} - \inf_{S\in\mathcal{J}_m(\Sigma,s,M)} P_n\psi_S = o_P(1).
\]

Further, let $S^0$ be any minimizer of $S\mapsto P\psi_S$ on $\overline{\mathcal{J}_m(\Sigma,s,M)}$. It follows that
\[
P\psi_{S^n} \stackrel{p}{\to} P\psi_{S^0}.
\]
Moreover, if $R_n$ is
bounded in expectation as 
\[
\E[R_n] \lesssim \E\|P_n-P\|_{\Psi_m(\Sigma,s,M)},
\]
then the expected KL divergence of $S^n$ is bounded as
\[
\E[P\psi_{S^n}] - P\psi_{S^0}\lesssim \begin{cases}
n^{-1/2}, & d < 2s, \\
n^{-1/2}\log n, &d=2s, \\
n^{-s/d}, &d>2s.
\end{cases}
\]
\label{thm:flow}
\end{theorem}

\section{Numerical Illustrations}
\label{sec:experiments}
\begin{figure*}[tp]
\centering
\begin{tabular}{llll}
\hspace{-6mm}
\includegraphics[height=3.65cm,width=3.65cm]{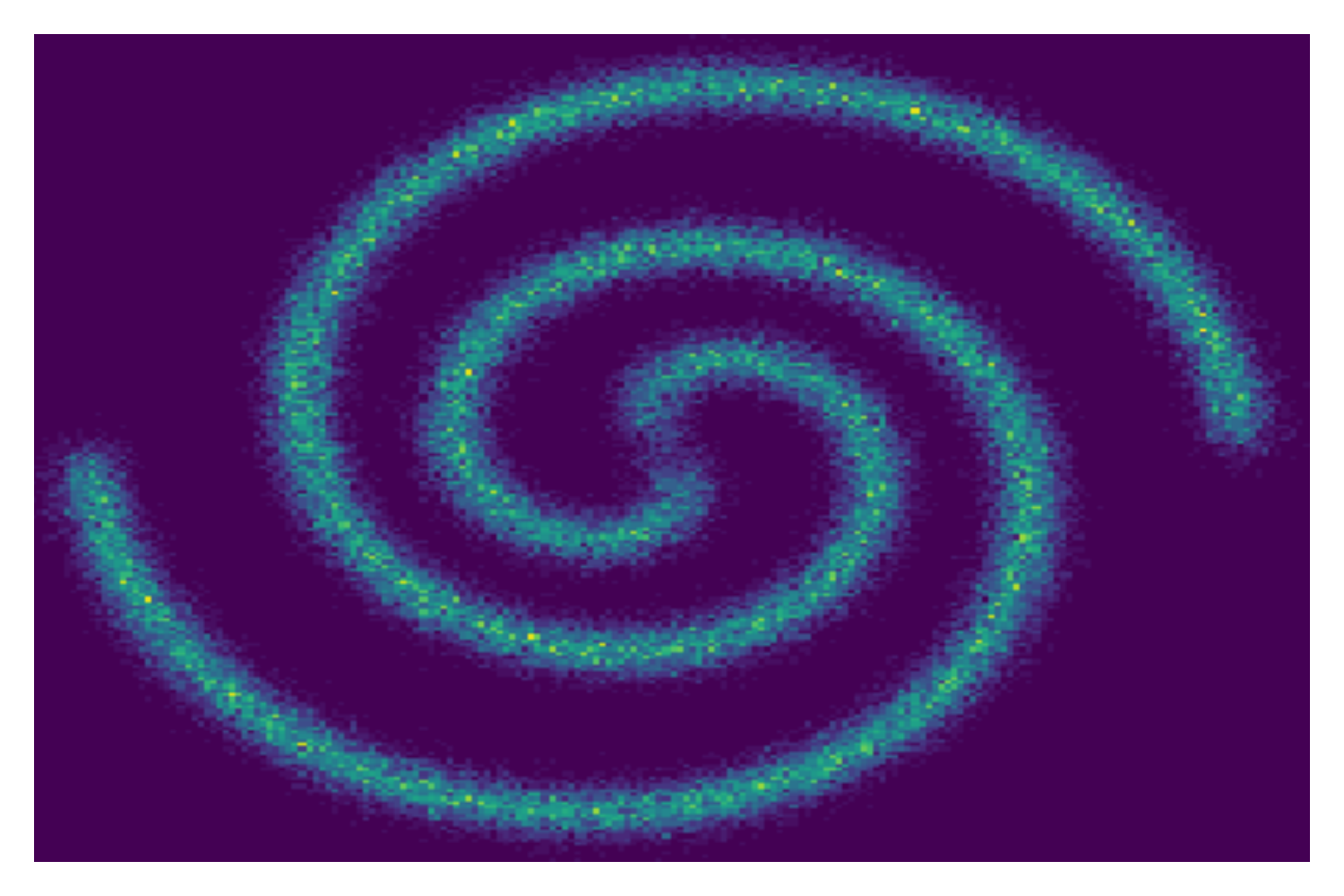}
& \hspace{-6mm}
\includegraphics[height=3.65cm,width=3.65cm]{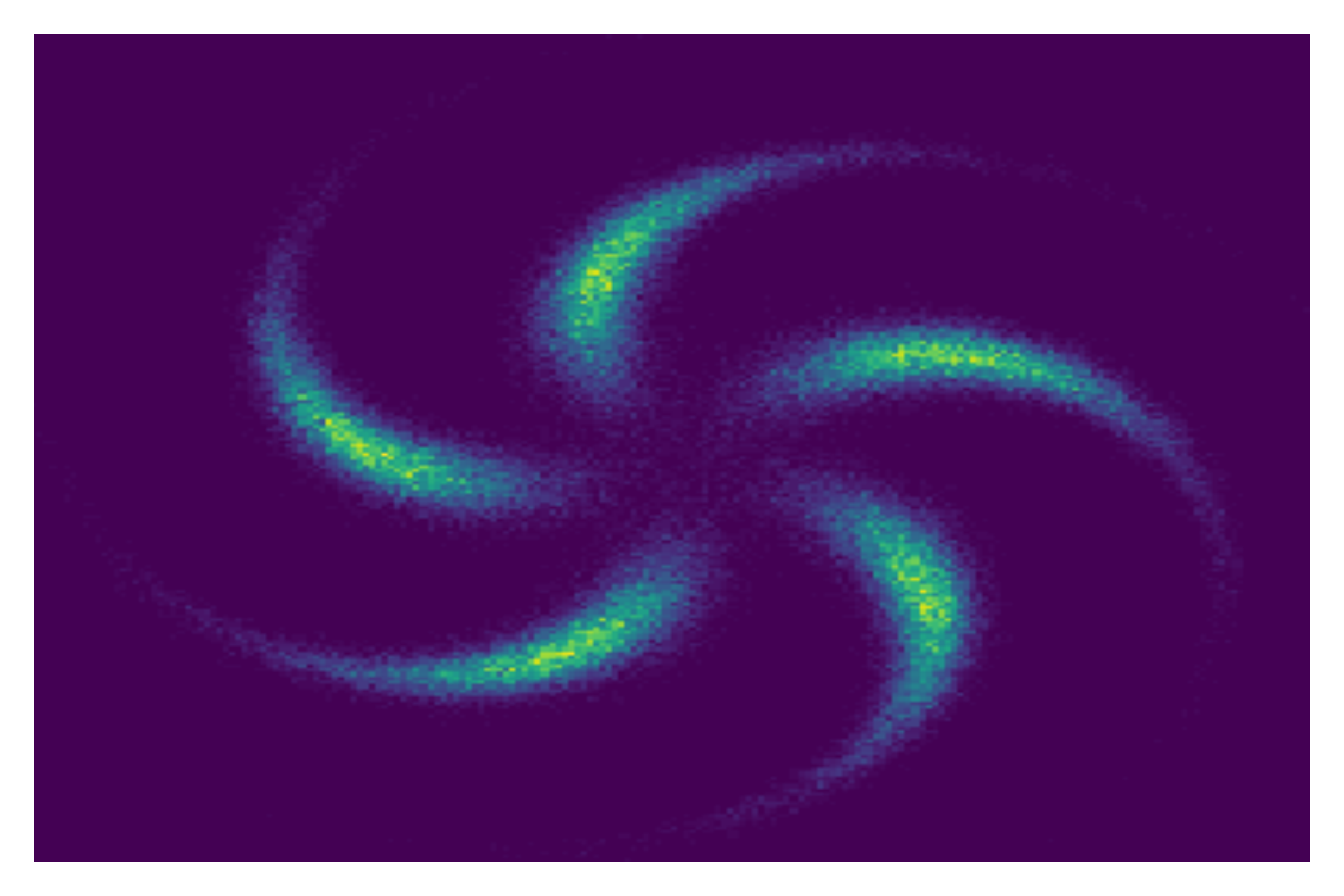}
& \hspace{-6mm}
\includegraphics[height=3.65cm,width=3.65cm]{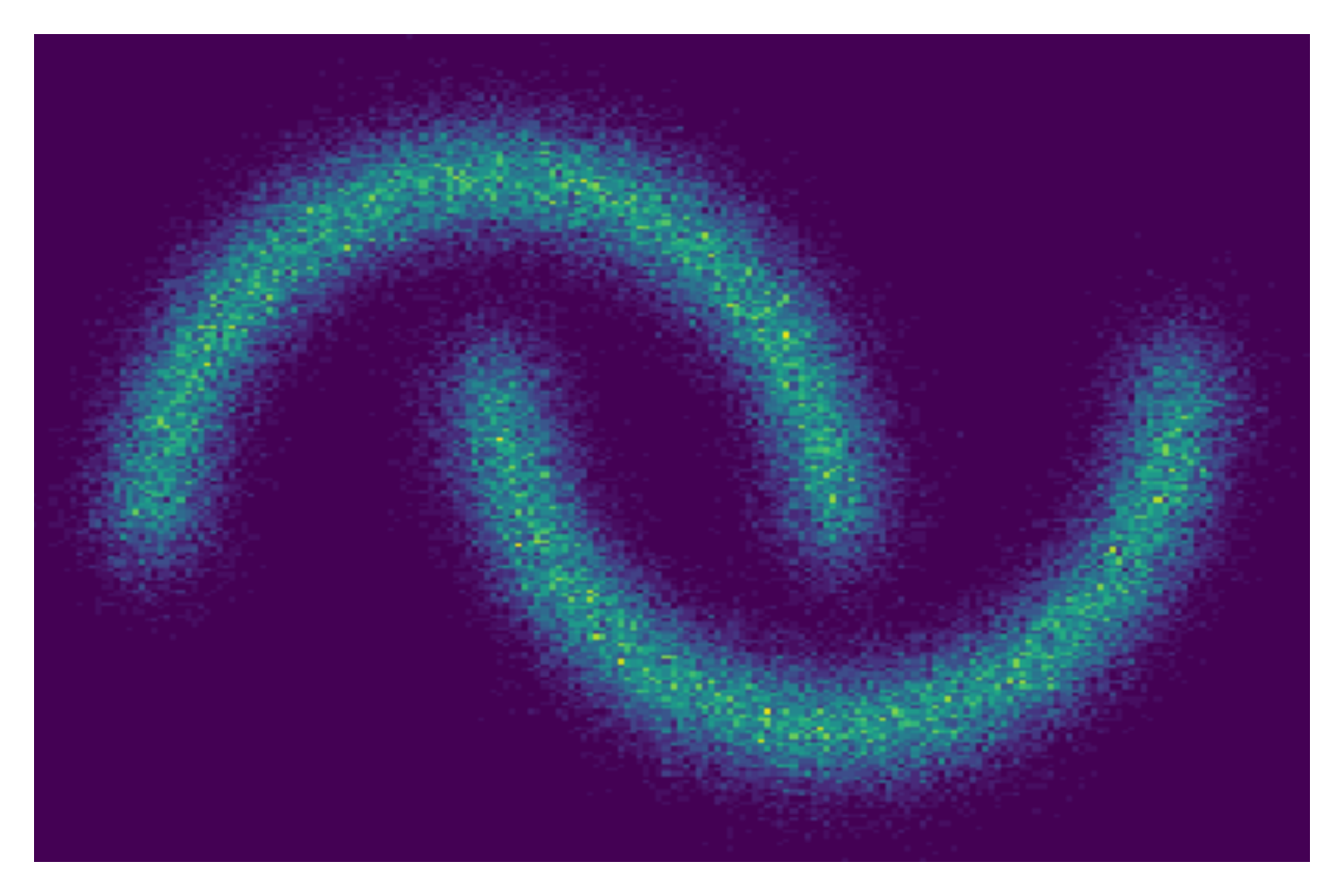}
& \hspace{-6mm}
\includegraphics[height=3.65cm,width=3.65cm]{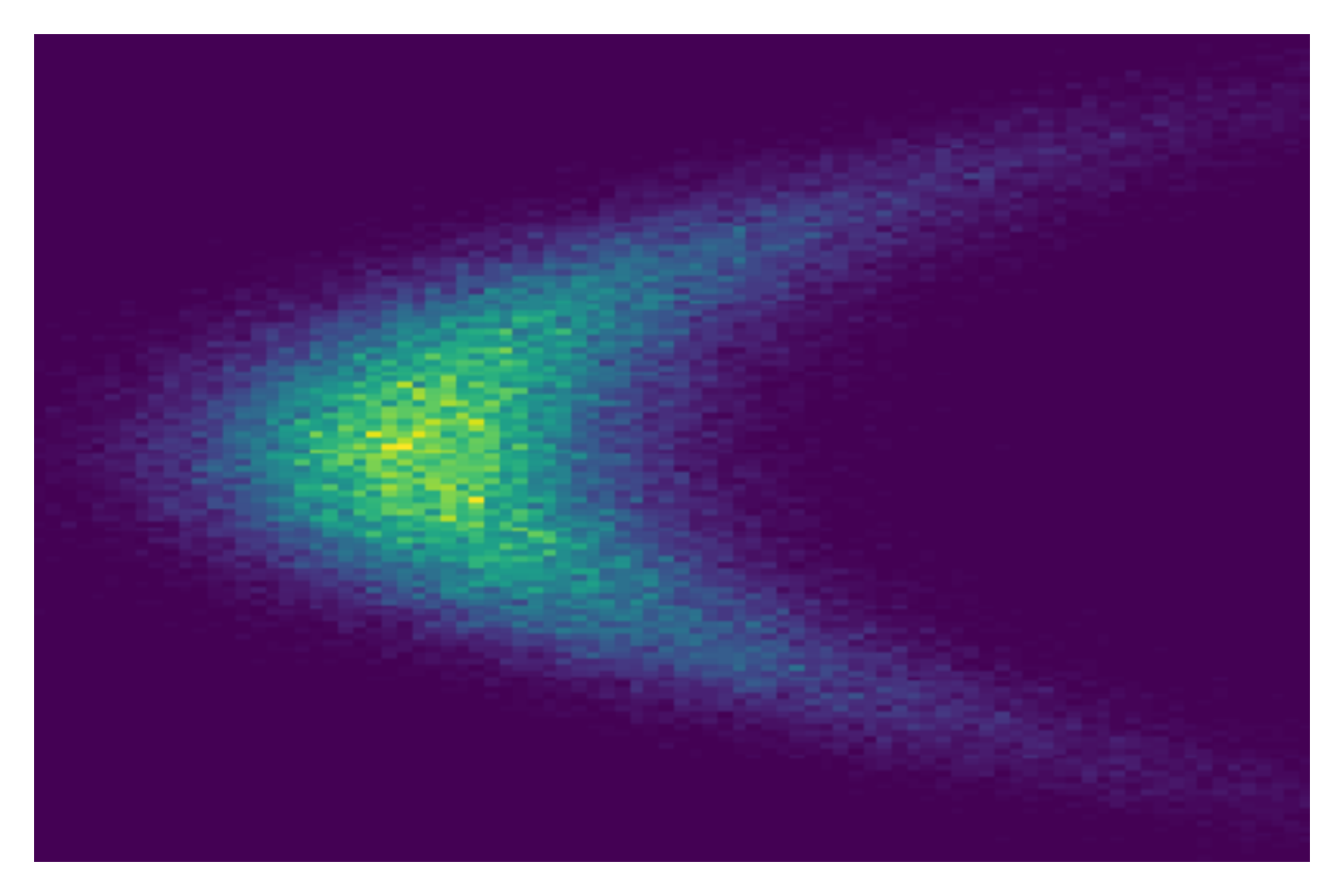} \\
\hspace{-6mm}
\includegraphics[height=4.5cm,width=4.15cm]{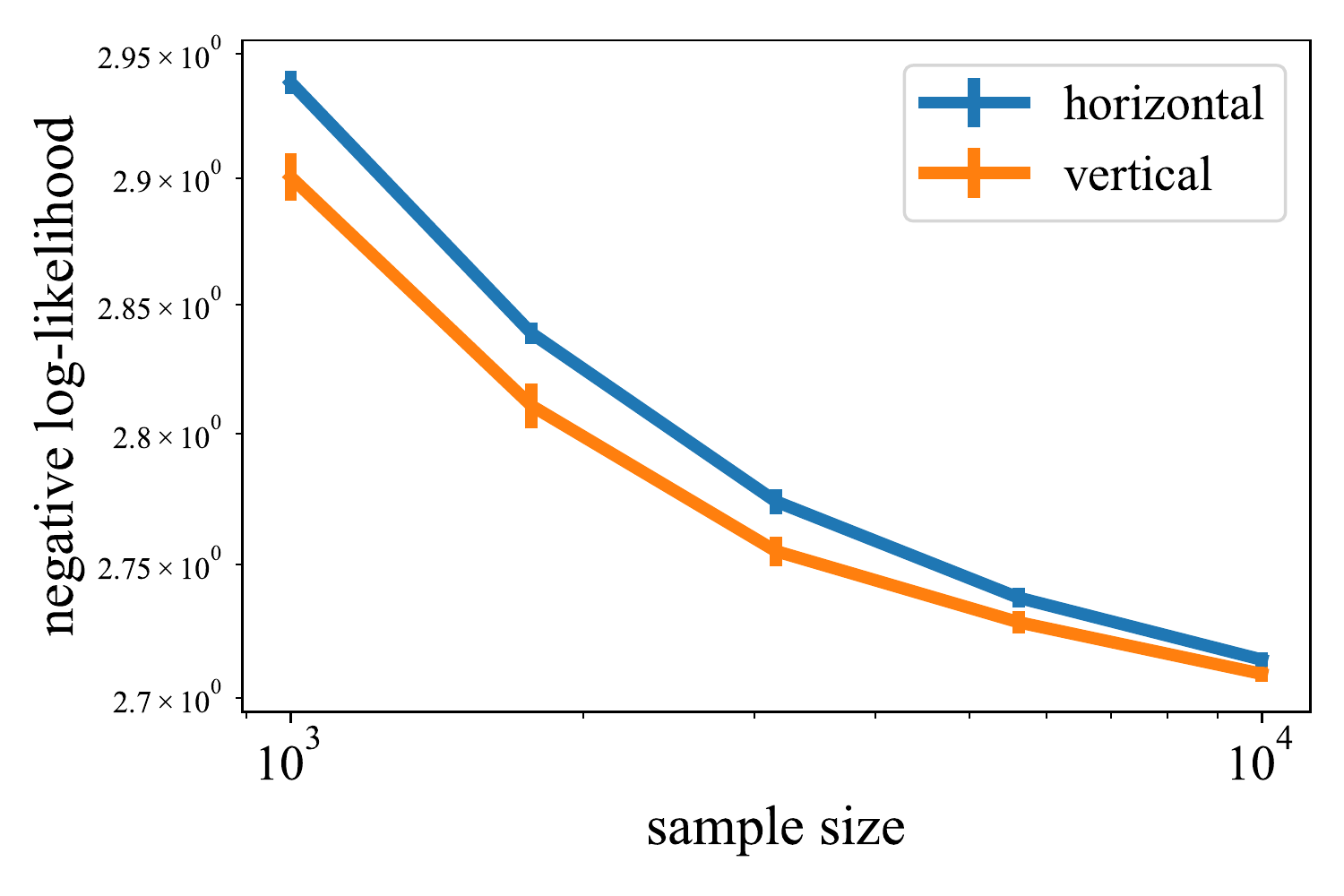}
&\hspace{-6mm}
\includegraphics[height=4.5cm,width=4.15cm]{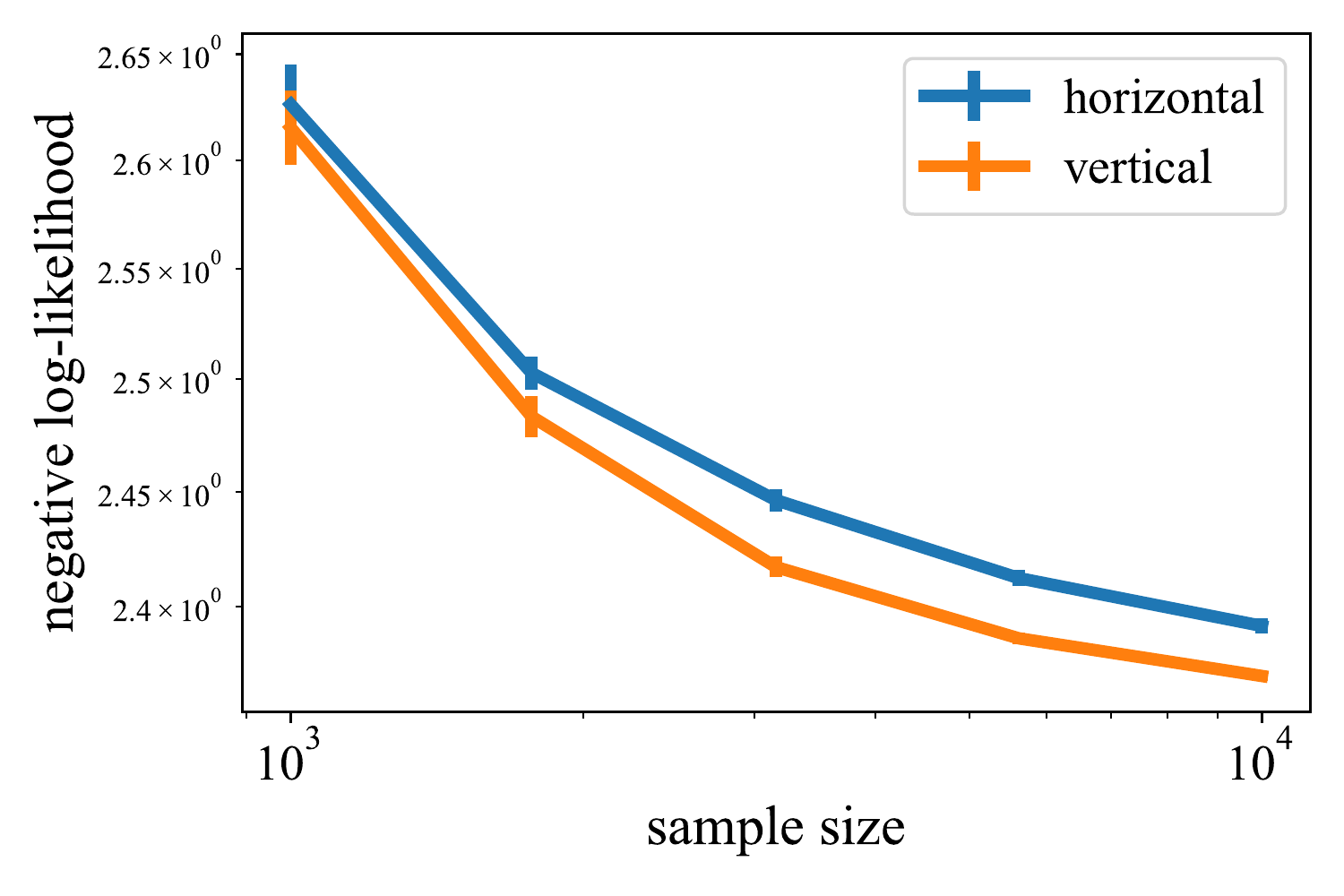}
&\hspace{-6mm}
\includegraphics[height=4.5cm,width=4.15cm]{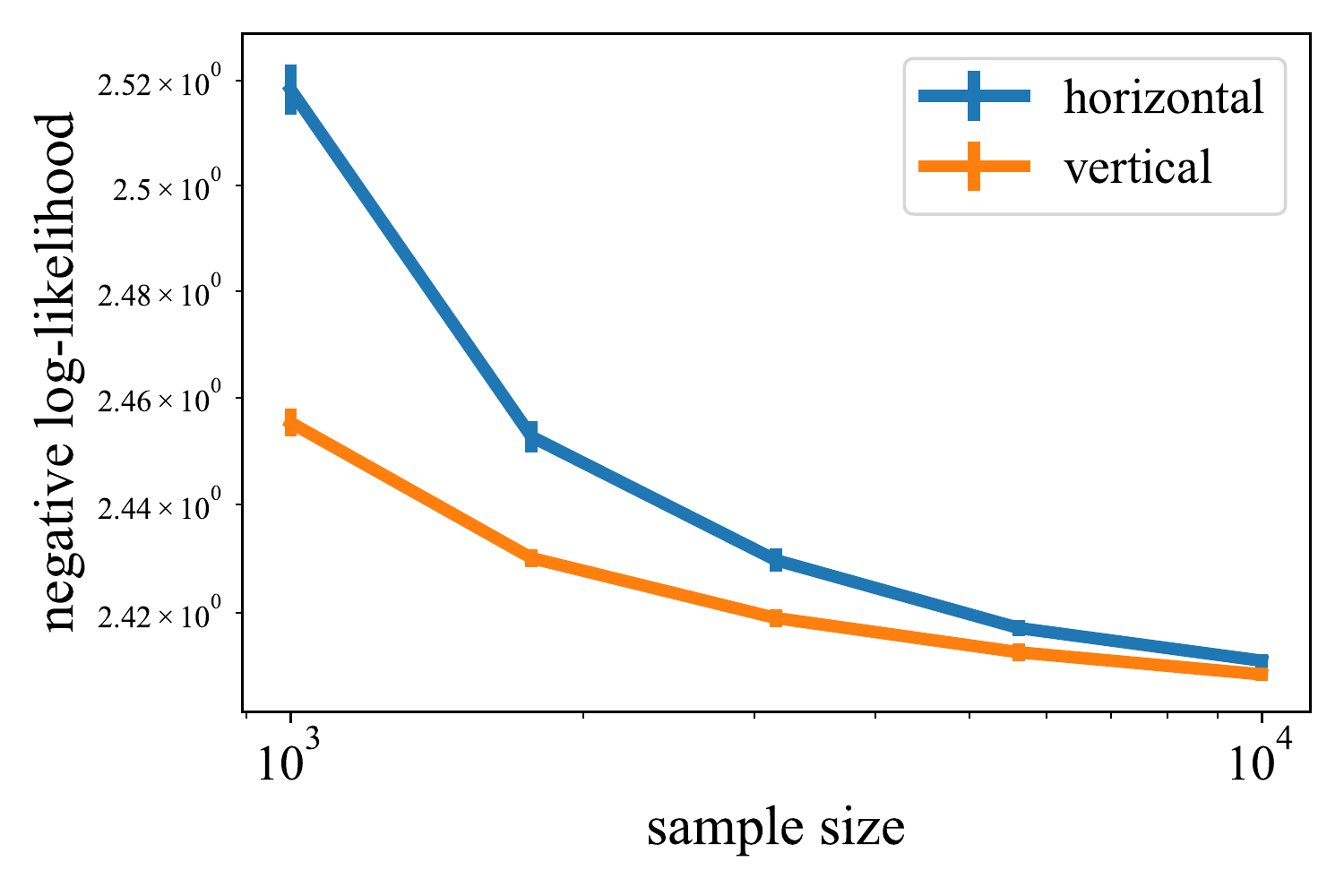}
&\hspace{-6mm}
\includegraphics[height=4.5cm,width=4.15cm]{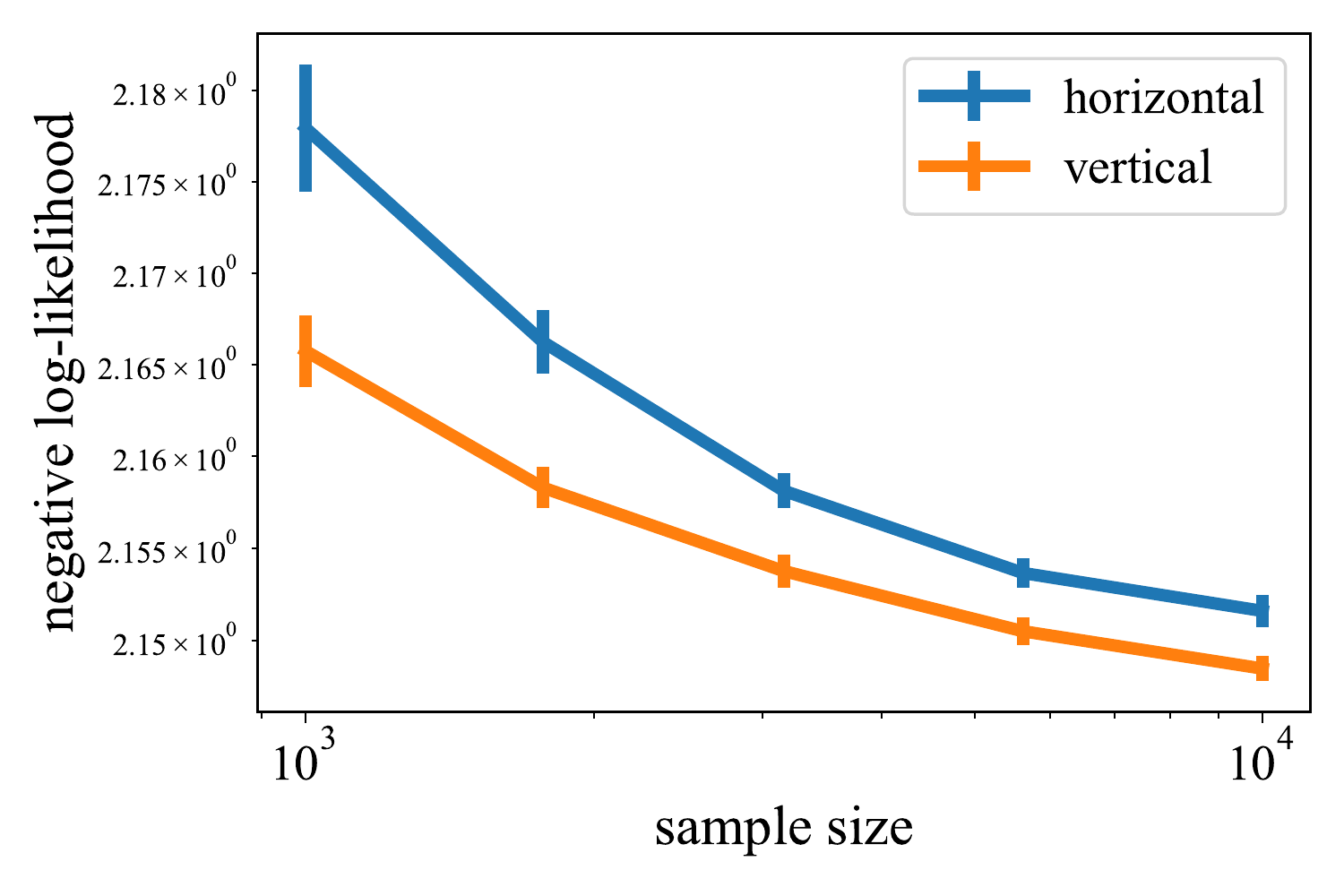} \\
\hspace{-6mm}
\includegraphics[height=3.65cm,width=3.65cm]{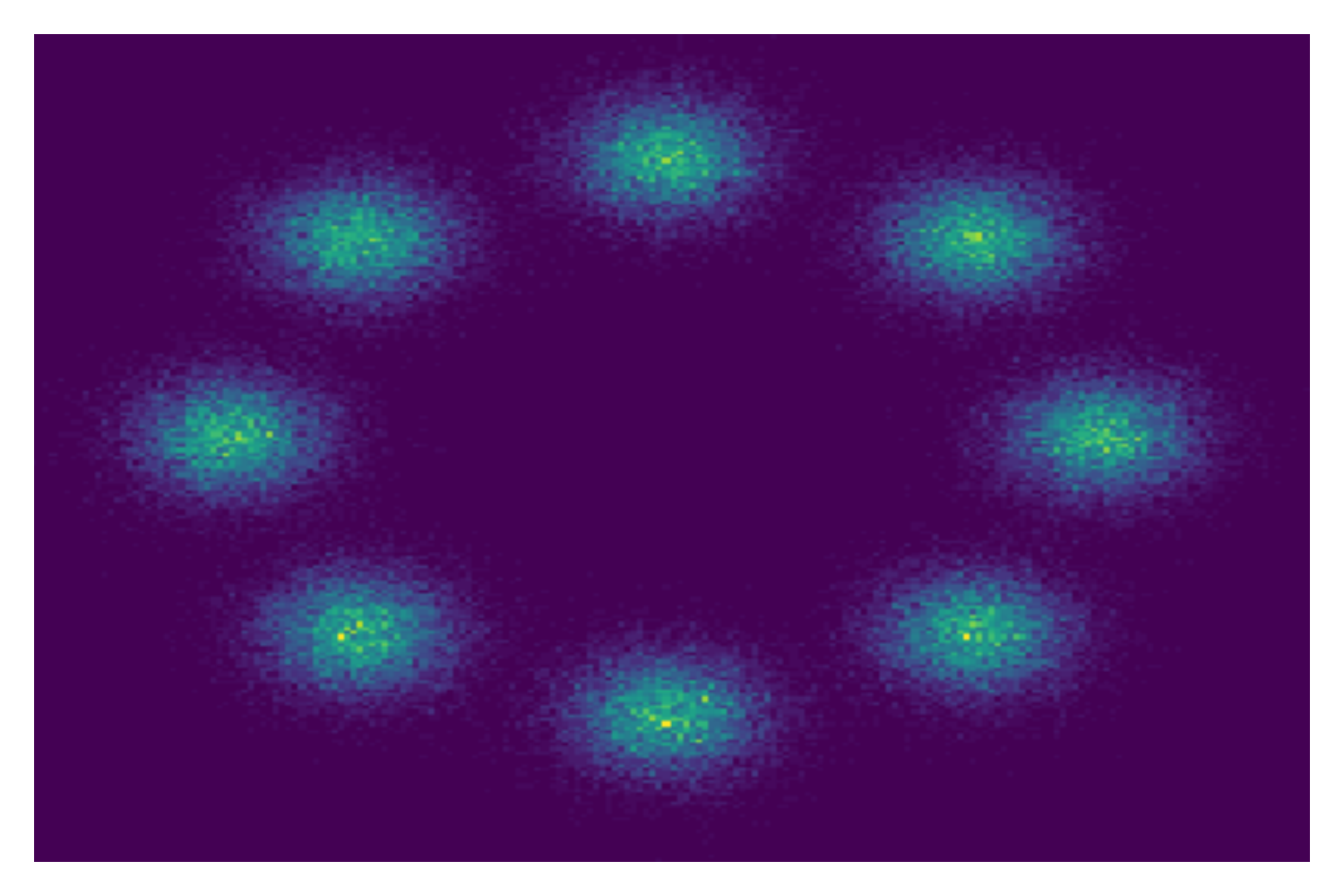}
& \hspace{-6mm}
\includegraphics[height=3.65cm,width=3.65cm]{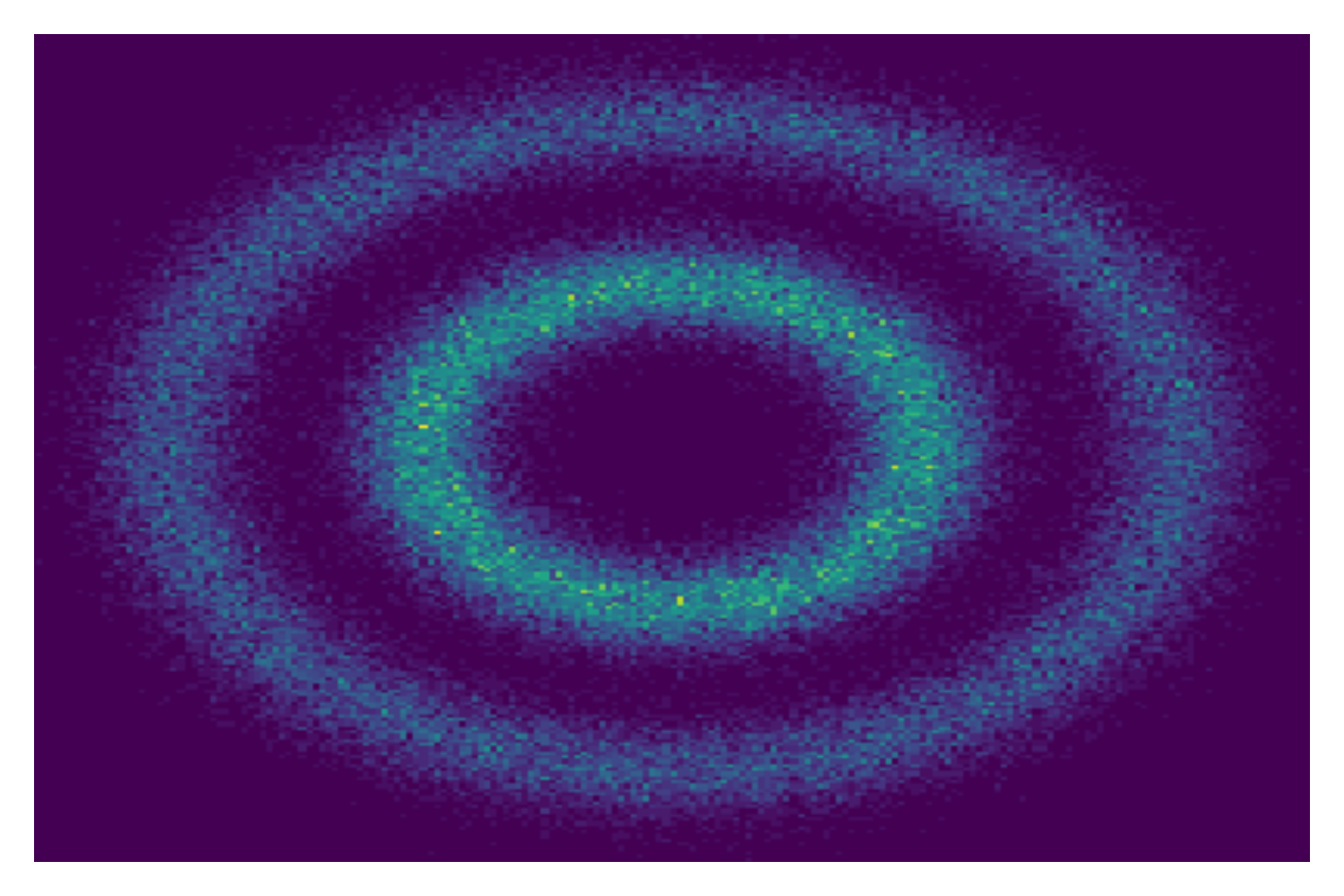}
& \hspace{-6mm}
\includegraphics[height=3.65cm,width=3.65cm]{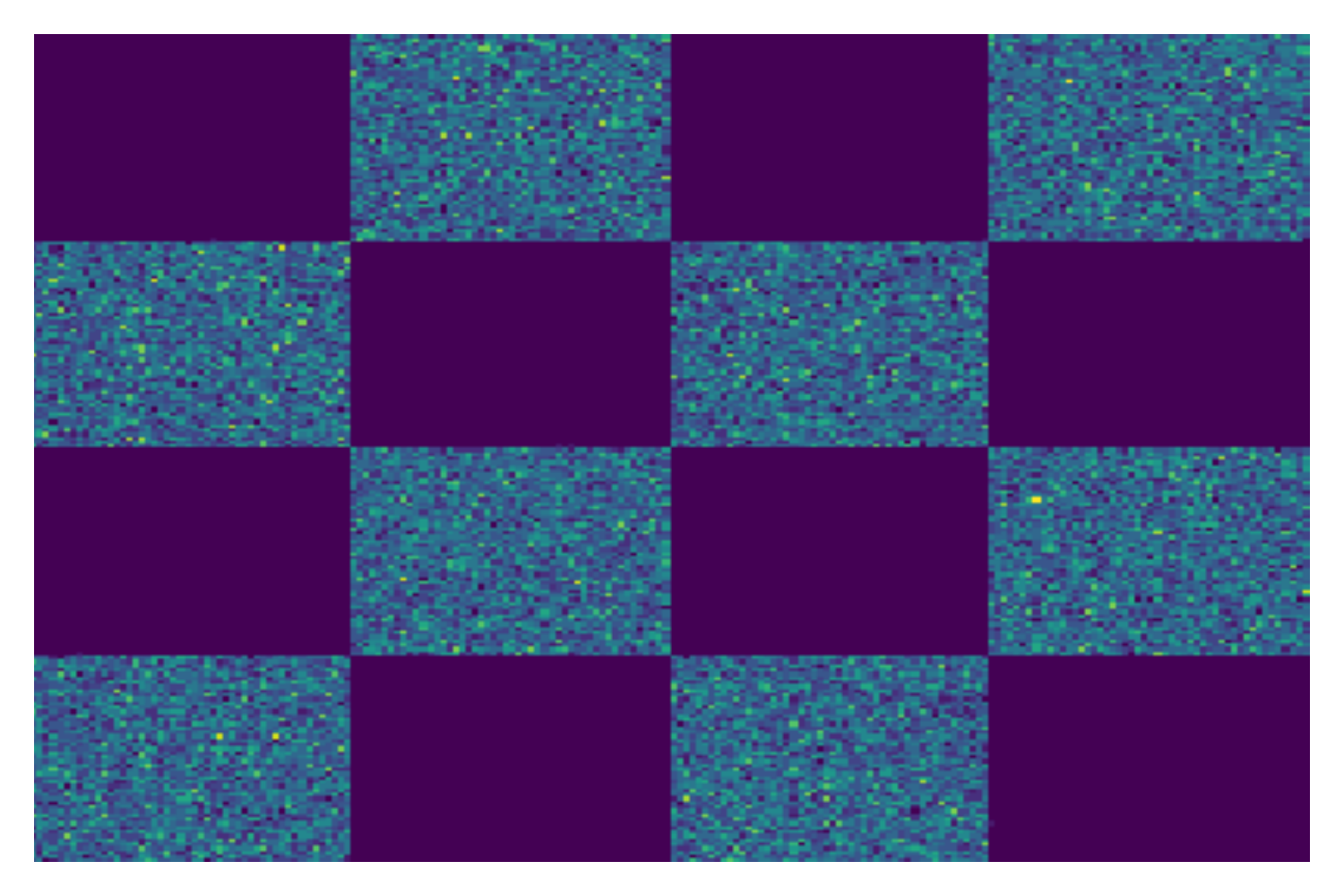}
&\hspace{-6mm}
\includegraphics[height=3.65cm,width=3.65cm]{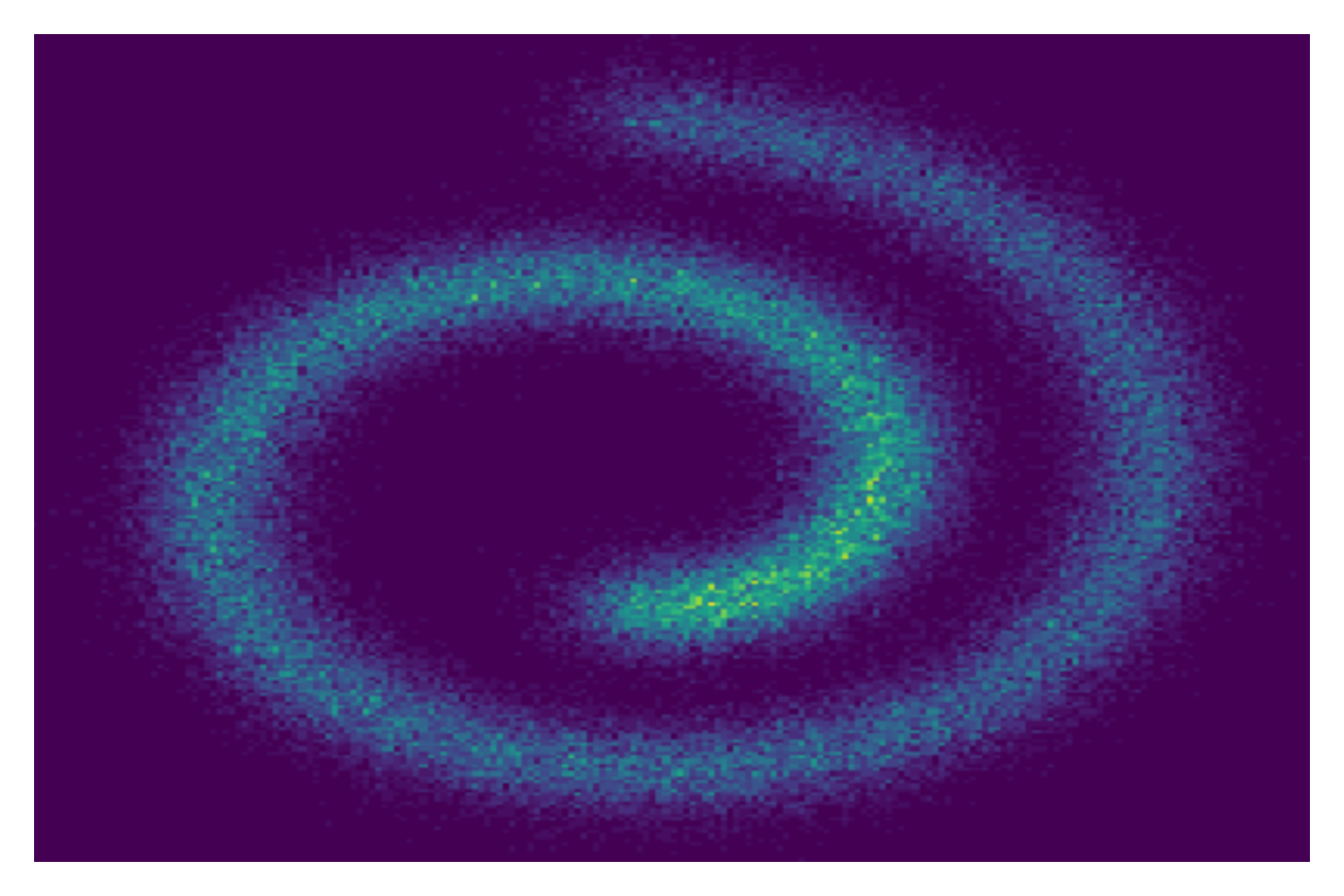} \\
\hspace{-6mm}
\includegraphics[height=4.5cm,width=4.15cm]{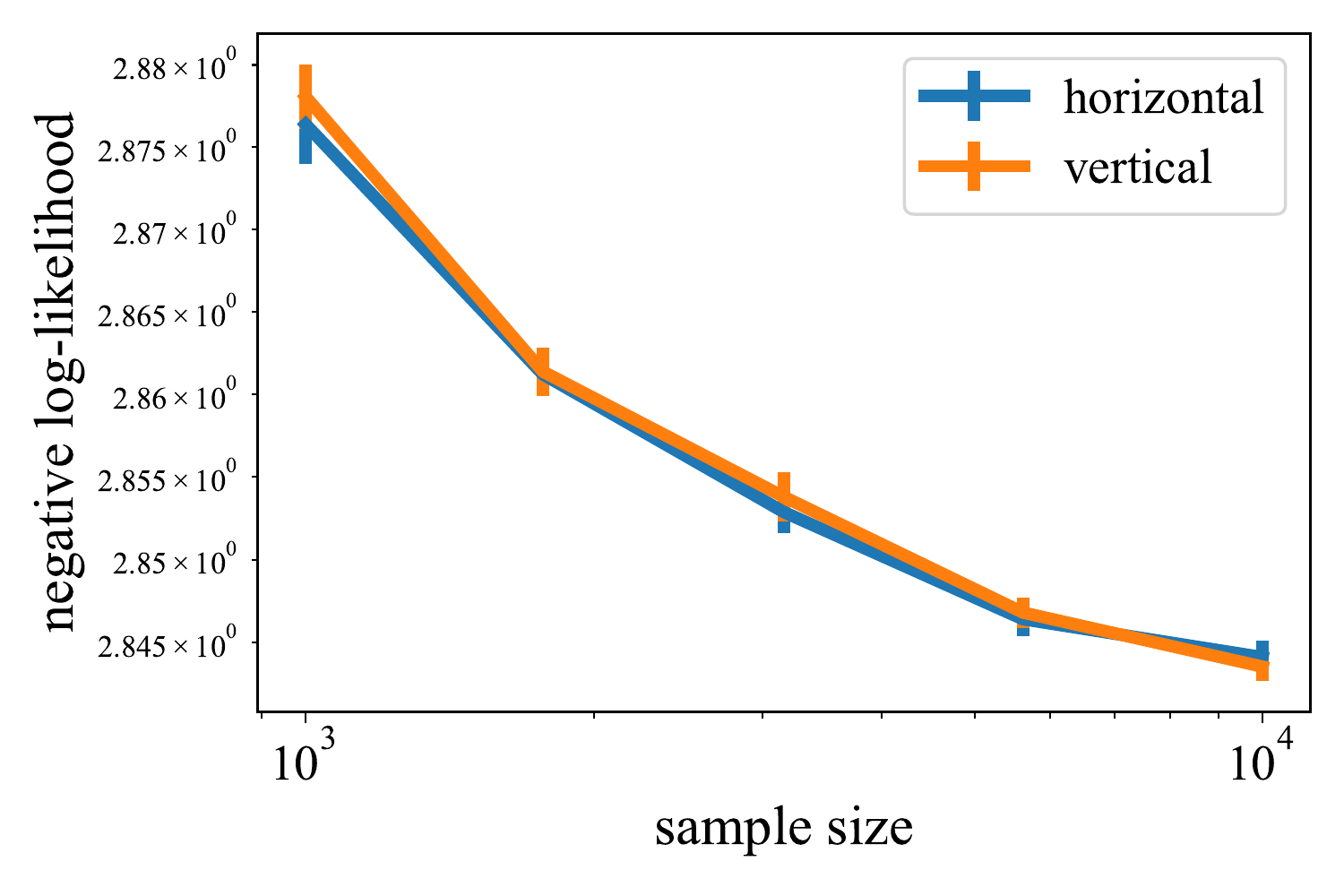}
& \hspace{-6mm}
\includegraphics[height=4.5cm,width=4.15cm]{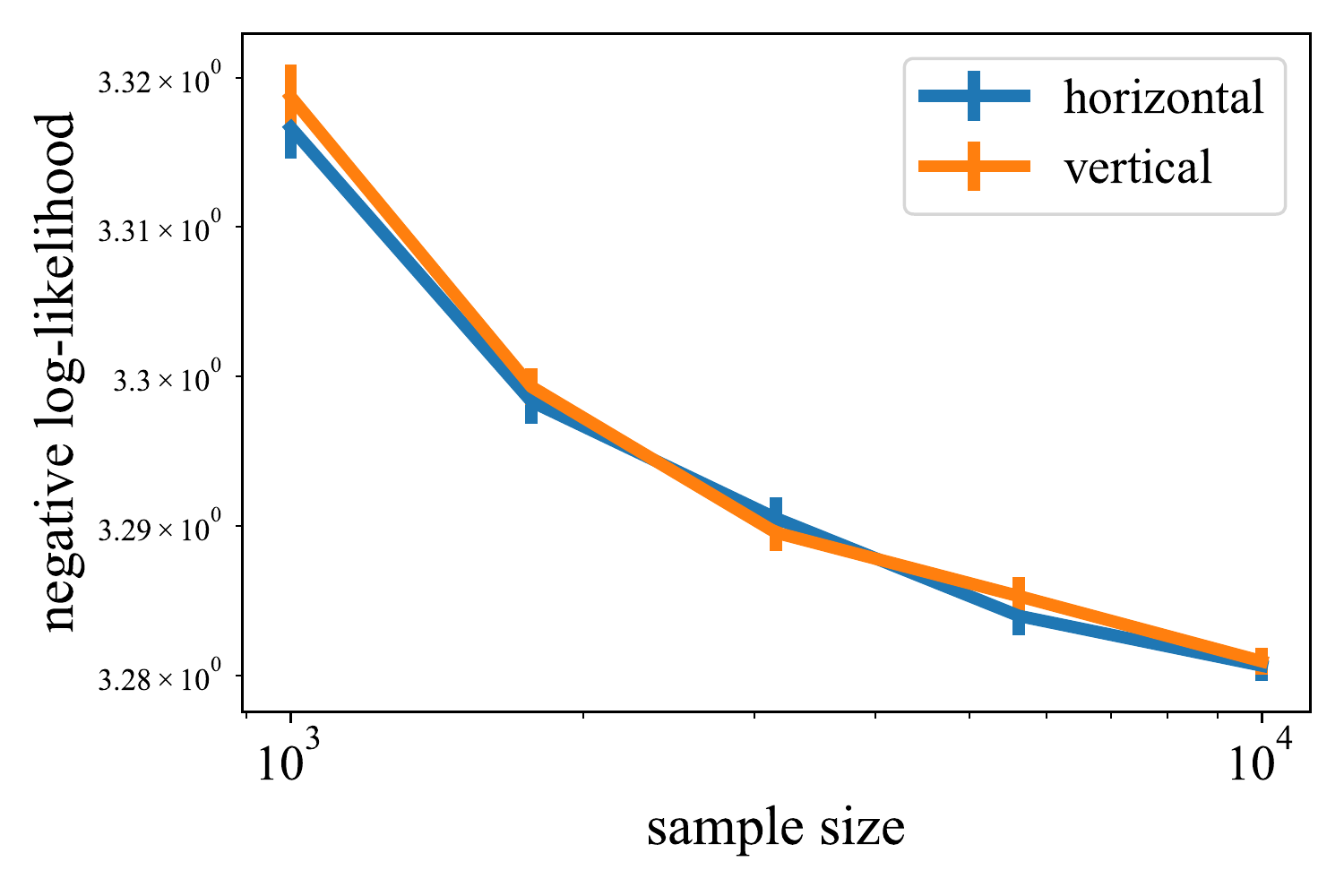}
& \hspace{-6mm}
\includegraphics[height=4.5cm,width=4.15cm]{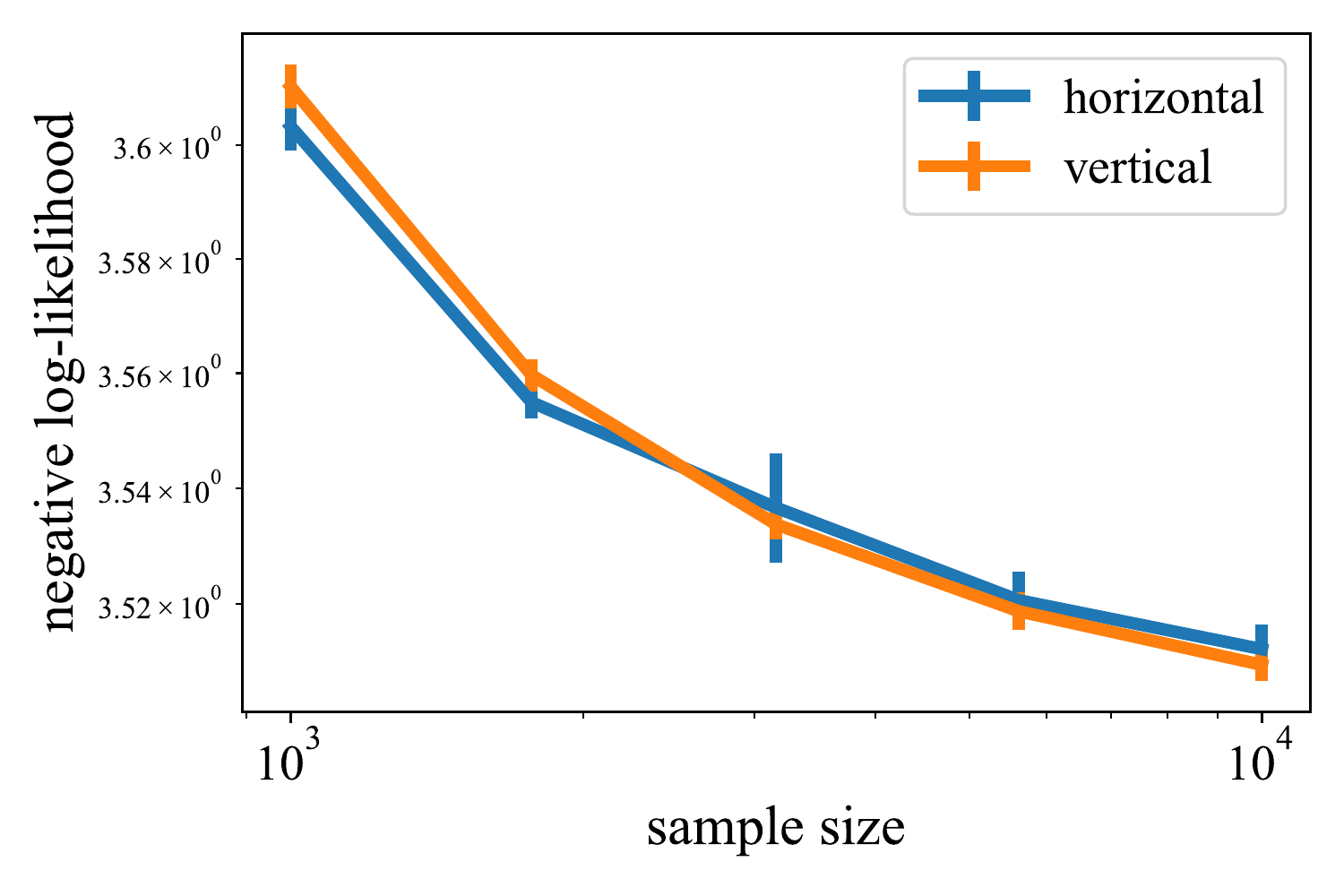}
&\hspace{-6mm}
\includegraphics[height=4.5cm,width=4.15cm]{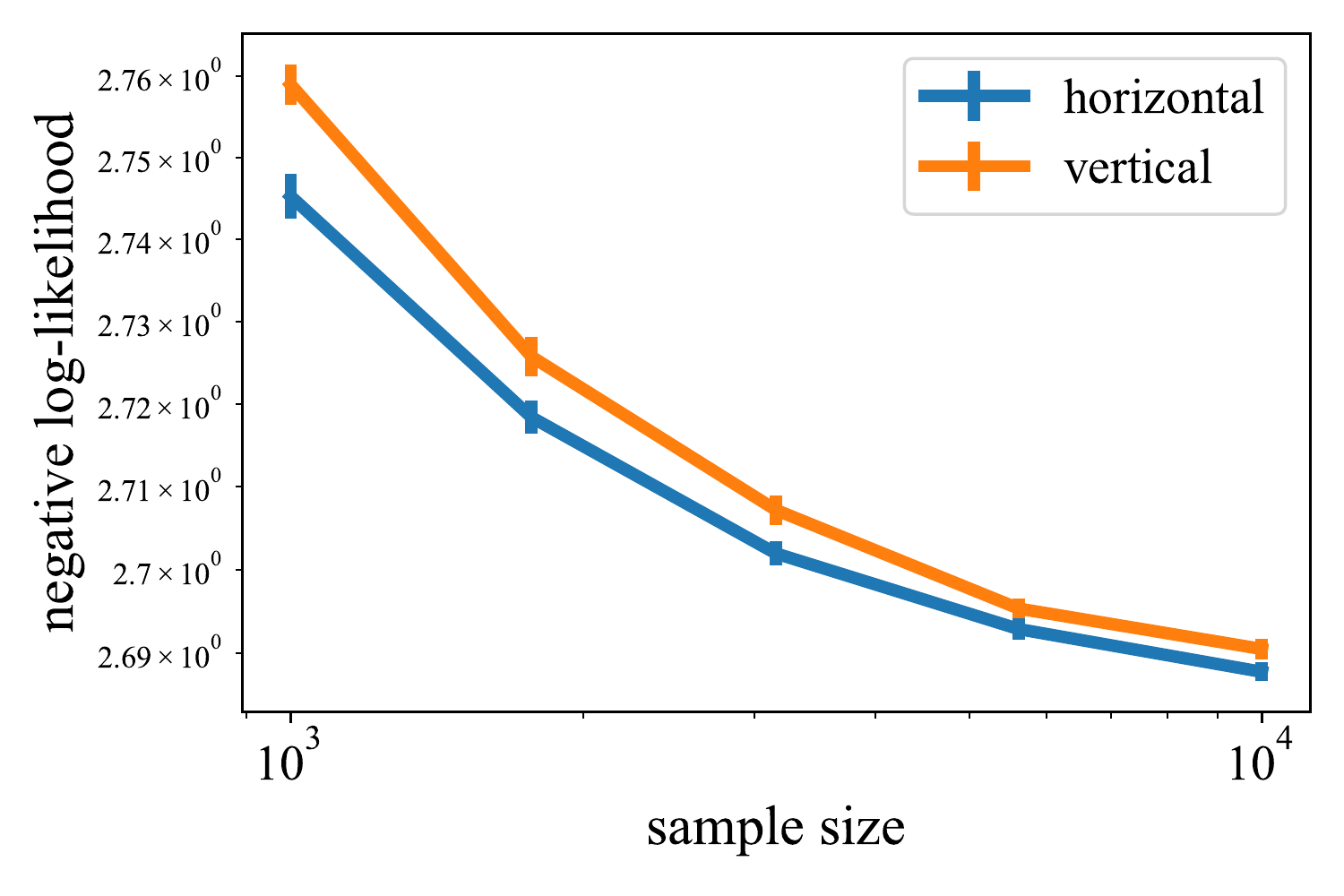}
\end{tabular}
\caption{Experimental comparison of convergence rates on 2D examples (top: density; bottom: rates) for the 8 datasets considered in \cite{mnn,ffjord}.}
\label{fig:umnn}
\end{figure*}

We conducted numerical experiments to illustrate our theoretical results. To estimate the KR map, we used Unconstrained Monotonic Neural Networks Masked Autoregressive Flows (UMNN-MAF), a particular triangular flow introduced in \cite{mnn}, with code to implement the model provided therein. 
UMNN-MAF learns an invertible monotone triangular map targeting the KR rearrangement via KL minimization
\[
S(x;\theta) = \begin{bmatrix}
S_1(x_1,\ldots,x_d;\theta) \\
S_1(x_2,\ldots,x_d;\theta) \\
\vdots \\
S_{d-1}(x_{d-1},x_d;\theta) \\
S_d(x_d;\theta)
\end{bmatrix},
\]
where each component $S_k(x_{k:d};\theta)$ is parametrized as the output of a neural network architecture that can learn arbitrary monotonic functions. Specifically, we have
\begin{equation}
S_k(x_{k:d},\theta) = \int_0^{x_k} f_k(t,h_k(x_{(k+1):d};\phi_k);\psi_k) dt + \beta_k(h_k(x_{(k+1):d};\phi_k);\psi_k), 
\label{eq:mnn}
\end{equation}
where $h_k(\cdot;\phi_k):\mathbb{R}^{d-k-1}\to\mathbb{R}^q$ is a $q$-dimensional neural embedding of $x_{(k+1):d}$ and $\beta_k(\cdot;\psi_k):\mathbb{R}^q\to\mathbb{R}$ is parametrized by a neural network. Each $f_k$ is a strictly positive function parametrized by a neural network, which guarantees that $\int_0^{x_k} f_k$ is increasing in $x_k$. Here the total parameter $\theta$ is defined as $\theta=\{(\psi_k,\phi_k)\}_{k=1}^d$. Further details of the model are provided in \cite{mnn}. We note that UMNN-MAF is captured by the model setup in our theoretical treatment of KR map estimation.

The model was trained via log-likelihood maximization using minibatch gradient descent with the Adam optimizer \cite{adam} with minibatch size 64, learning rate $10^{-4}$, and weight decay $10^{-5}$. The integrand network architecture defined in equation (\ref{eq:mnn}) consisted of 4 hidden layers of width 100. Following Wehenkel and Louppe \cite{mnn}, the architecture of the embedding networks is the best performing MADE network \cite{made} used in NAF \cite{naf}. We used 20 integration steps to numerically approximate the integral in equation (\ref{eq:mnn}). The source density $g$ is a bivariate standard normal distribution. The population negative log-likelihood loss, which differs from the KL objective by a constant factor (namely, the negative entropy of the target density $f$), was approximated by the empirical negative log-likelihood on a large independently generated test set of size $N=10^5$.

Figure \ref{fig:umnn} exhibits our results for UMNN-MAF trained on 8 two-dimensional datasets considered in \cite{mnn} and \cite{ffjord}. Heatmaps of the target densities are displayed in the top rows, while the bottom rows show the log-likelihood convergence rates as the sample size increases from $n=10^3$ to $n=10^4$ on a log-log scale. For each training sample size, we repeated the experiment 100 times. 
We report the mean of the loss over the 100 replicates with 95\% error bars.

These experiments highlight the impact of the ordering of coordinates on the convergence rate, as predicted by Theorem \ref{thm:order}. The blue curves correspond to first estimating the KR map along the horizontal $x_1$ axis, then the vertical axis conditional on the horizontal, $x_2|x_1$. The orange curves show the reverse order, namely estimating the KR map along the vertical $x_2$ axis first. The 5 densities in the top rows (2 spirals, pinwheel, moons, and banana) and the bottom right (swiss roll) are asymmetric in $(x_1,x_2)$ (i.e., $f(x_1,x_2)$ is not exchangeable in $(x_1,x_2)$). These densities exhibit different convergence rates depending on the choice of order. The remaining 3 densities in the bottom rows (8 gaussians, 2 circles, checkerboard) are symmetric in $(x_1,x_2)$ and do not exhibit this behavior, which is to be expected. We also note that a linear trend (on the log-log scale) is visible in the plots for smaller $n$, which aligns with the convergence rates established in Theorem \ref{thm:KLconv}. For larger $n$, however, approximation error dominates and the loss plateaus. 

As an illustrative example, we focus in on the top right panel of Figure \ref{fig:umnn}, which plots the \textit{banana} density $f(x_1,x_2)$ corresponding to the random variables
\begin{align*}
X_2 &\sim N(0,1) \\    
X_1|X_2 &\sim N(X_2^2/2,1/2).
\end{align*}
It follows that $f$ is given by
\begin{align*}
f(x_1,x_2) &= f(x_1|x_2)f(x_2) \\
&\propto \exp\left\{-\left(x_1-x_2^2/2\right)^2\right\}\cdot\exp\left\{-x_2^2/2\right\}.
\end{align*}
Intuitively, estimating the normal conditional $f(x_1|x_2)=N(x_2^2/2,1/2)$ and the standard normal marginal $f(x_2)$ should be easier than estimating $f(x_2|x_1)$ and $f(x_1)$. Indeed, as $x_1$ increases, we see 
that $f(x_2|x_1)$ transitions from a unimodal to a bimodal distribution. As such, 
we expect that estimating the KR map $S^{(21)}$ from $f(x_2,x_1)$ to the source density $g$ 
should be more difficult than estimating the KR map  $S^{(12)}$ from $f(x_1,x_2)$ to $g$. This is because the first component of $S^{(21)}$ targets the conditional distribution $f(x_2|x_1)$ and the second component targets $f(x_1)$, while the first component of $S^{(12)}$ targets $f(x_1|x_2)$ and the second component targets $f(x_2)$. Indeed, this is what we see in the top right panels of Figure \ref{fig:umnn}, which shows the results of fitting UMNN-MAF to estimate $S^{(12)}$ (orange) and $S^{(21)}$ (blue). 
As expected, we see that estimates of $S^{(21)}$ converge more slowly than those of $S^{(12)}$. These results are consistent with the findings of Papamakarios et al. (2017) \cite{maf}, who observed this behavior in estimating the banana density with MADE \cite{made}.

In the first 3 panels of Figure \ref{fig:kong-chaudhuri}, we repeat the experimental setup on the 3 normal mixture densities considered by Kong and Chaudhuri in \cite{pmlr-v108-kong20a}. The conclusions drawn from Figure \ref{fig:umnn} are echoed here. We observe no dependence in the convergence rates on the choice of variable ordering, since the target densities $f(x_1,x_2)$ are exchangeable in $(x_1,x_2)$. Furthermore, we observe a linear convergence rate as predicted by Theorem \ref{thm:KLconv}.

\begin{figure*}[tp]
\centering
\begin{tabular}{llll}
\hspace{-6mm}
\includegraphics[height=3.65cm,width=3.65cm]{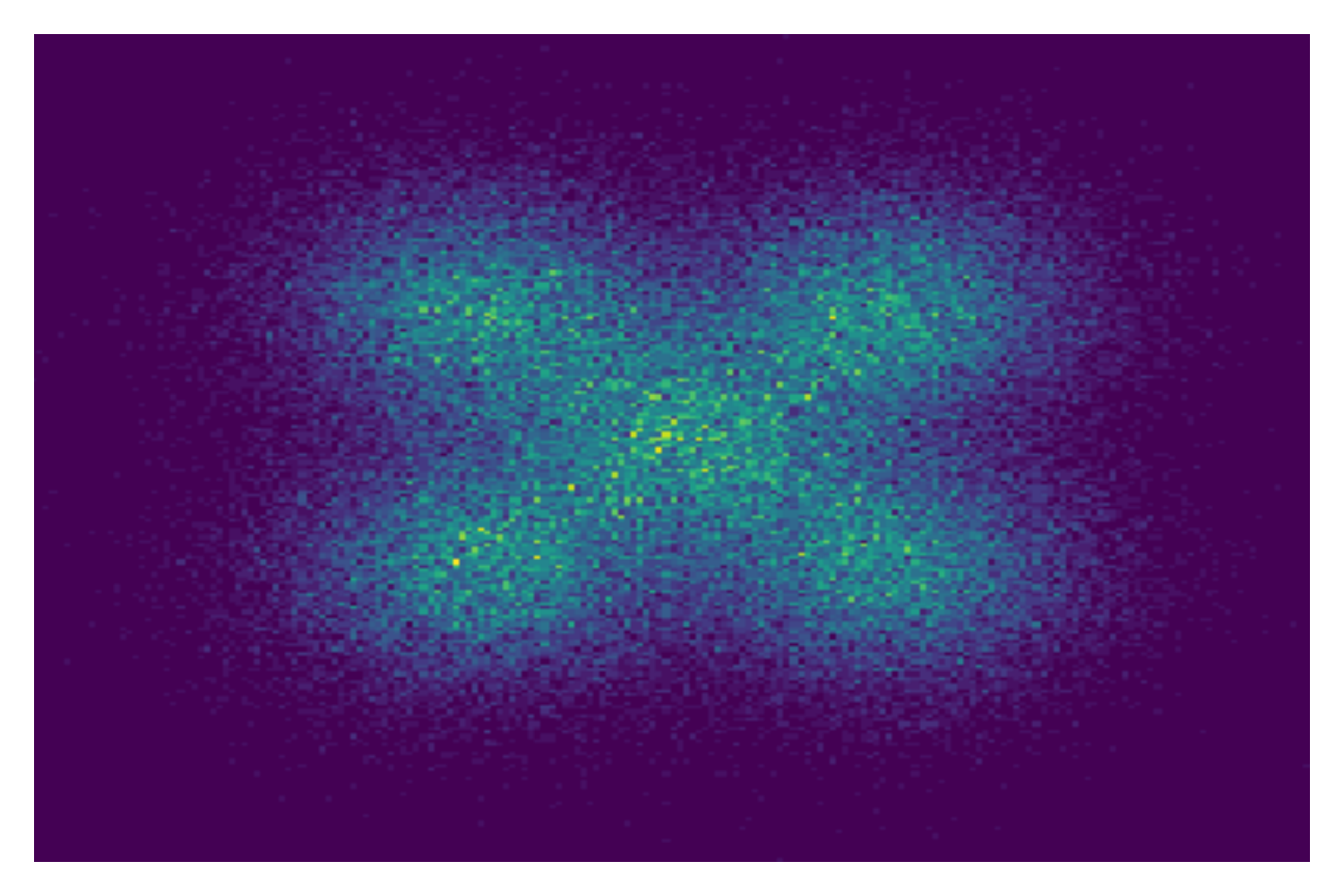}
& \hspace{-6mm}
\includegraphics[height=3.65cm,width=3.65cm]{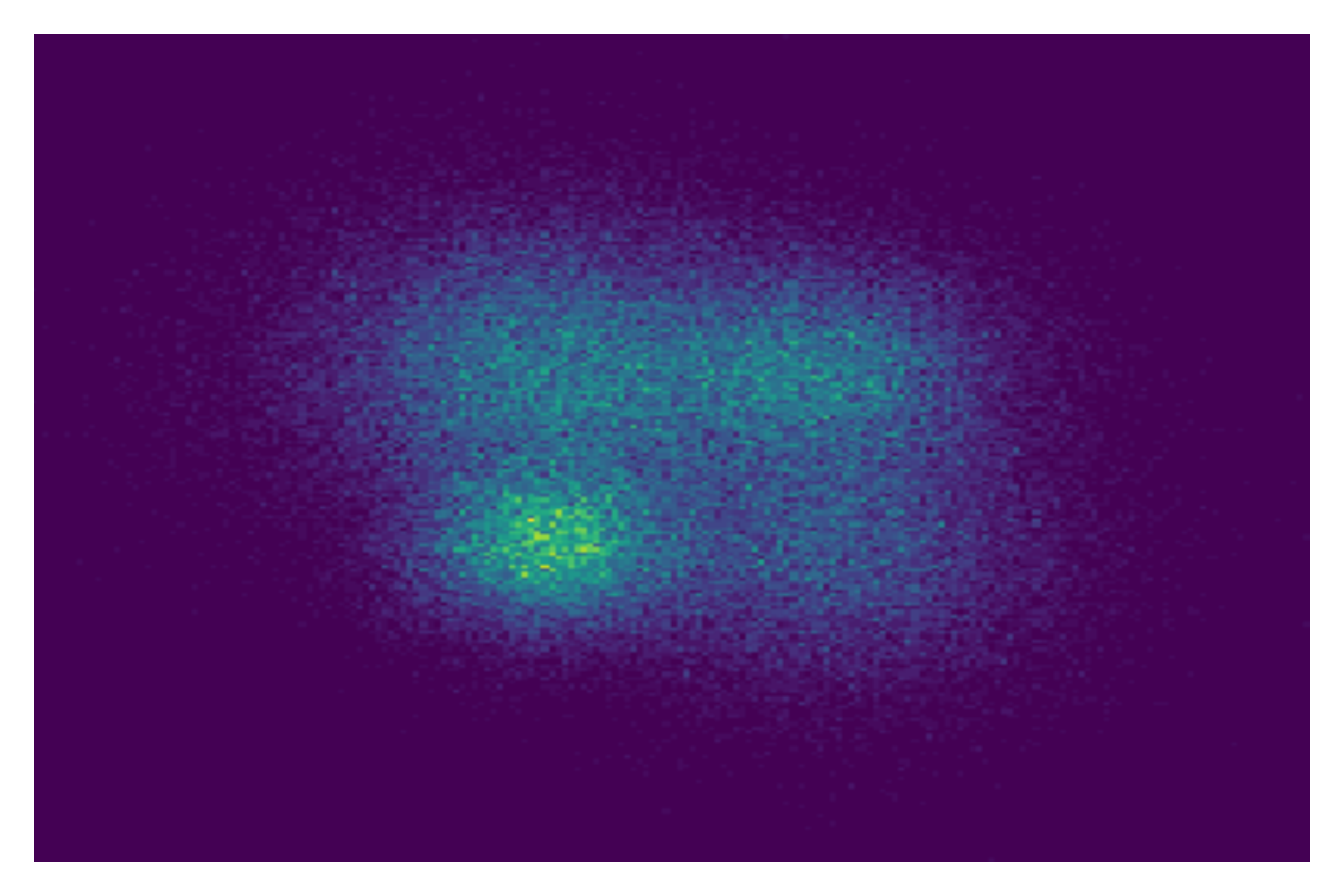}
& \hspace{-6mm}
\includegraphics[height=3.65cm,width=3.65cm]{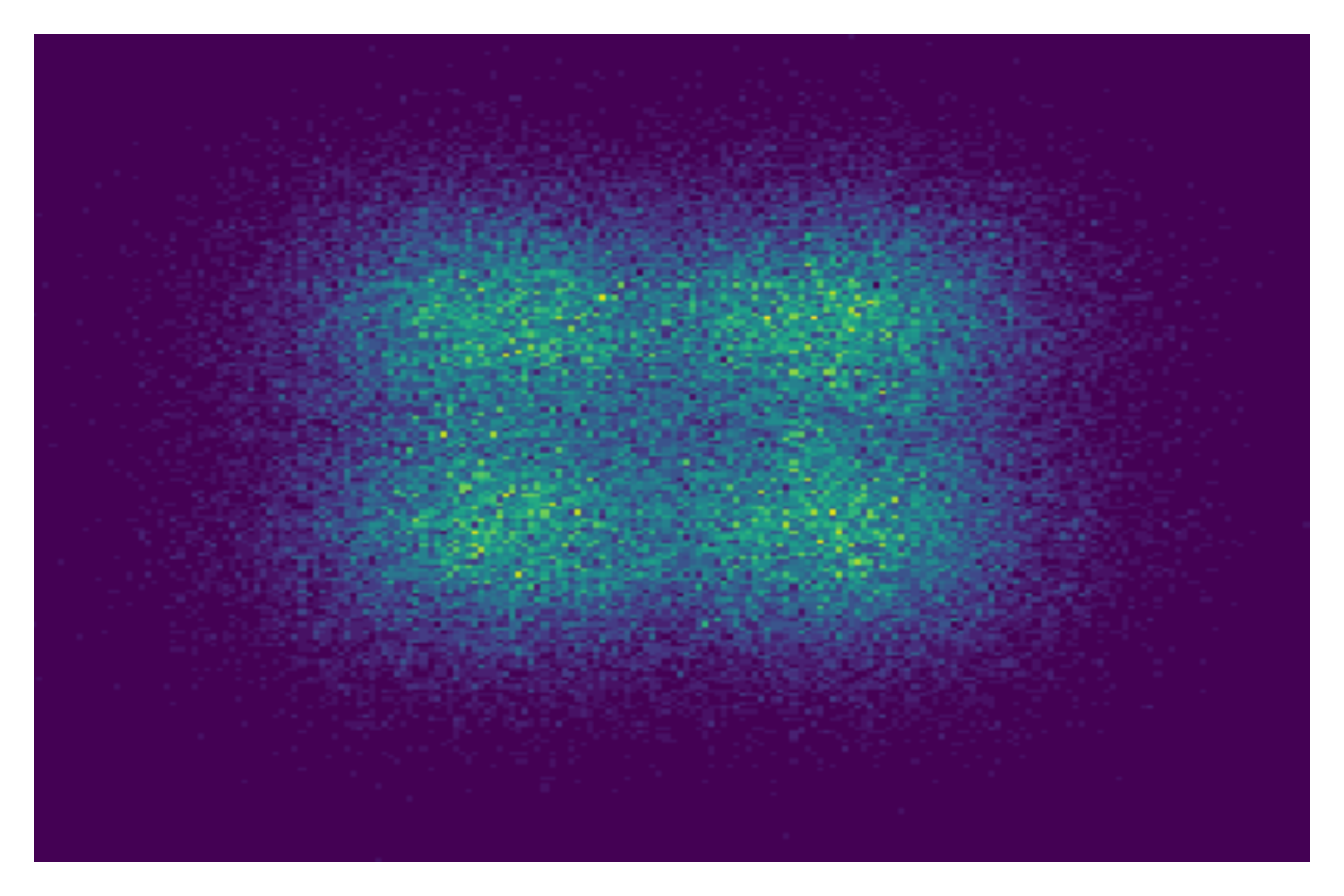} 
& \hspace{-6mm}
\includegraphics[height=3.65cm,width=3.65cm]{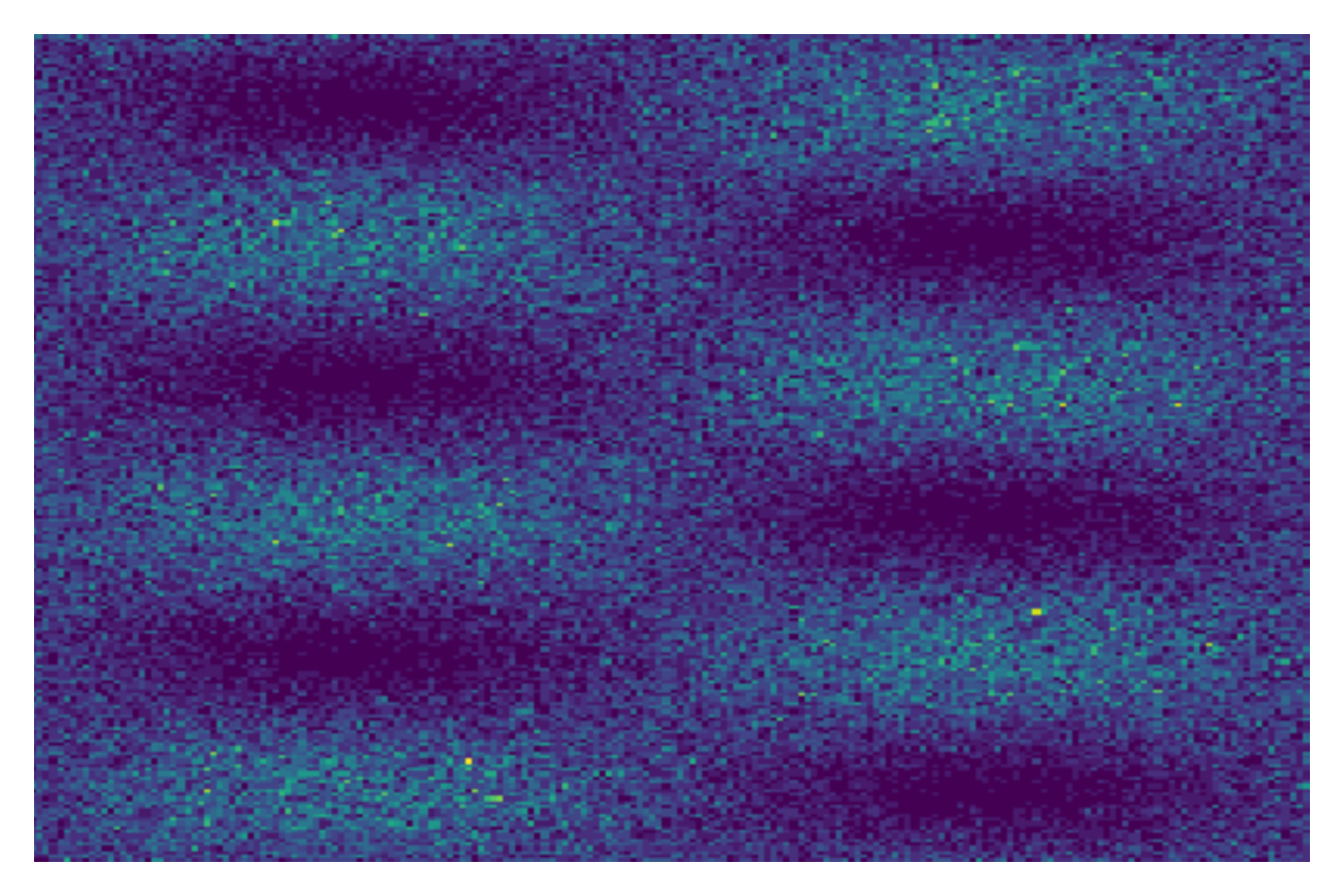} \\
\hspace{-6mm}
\includegraphics[height=4.5cm,width=4.15cm]{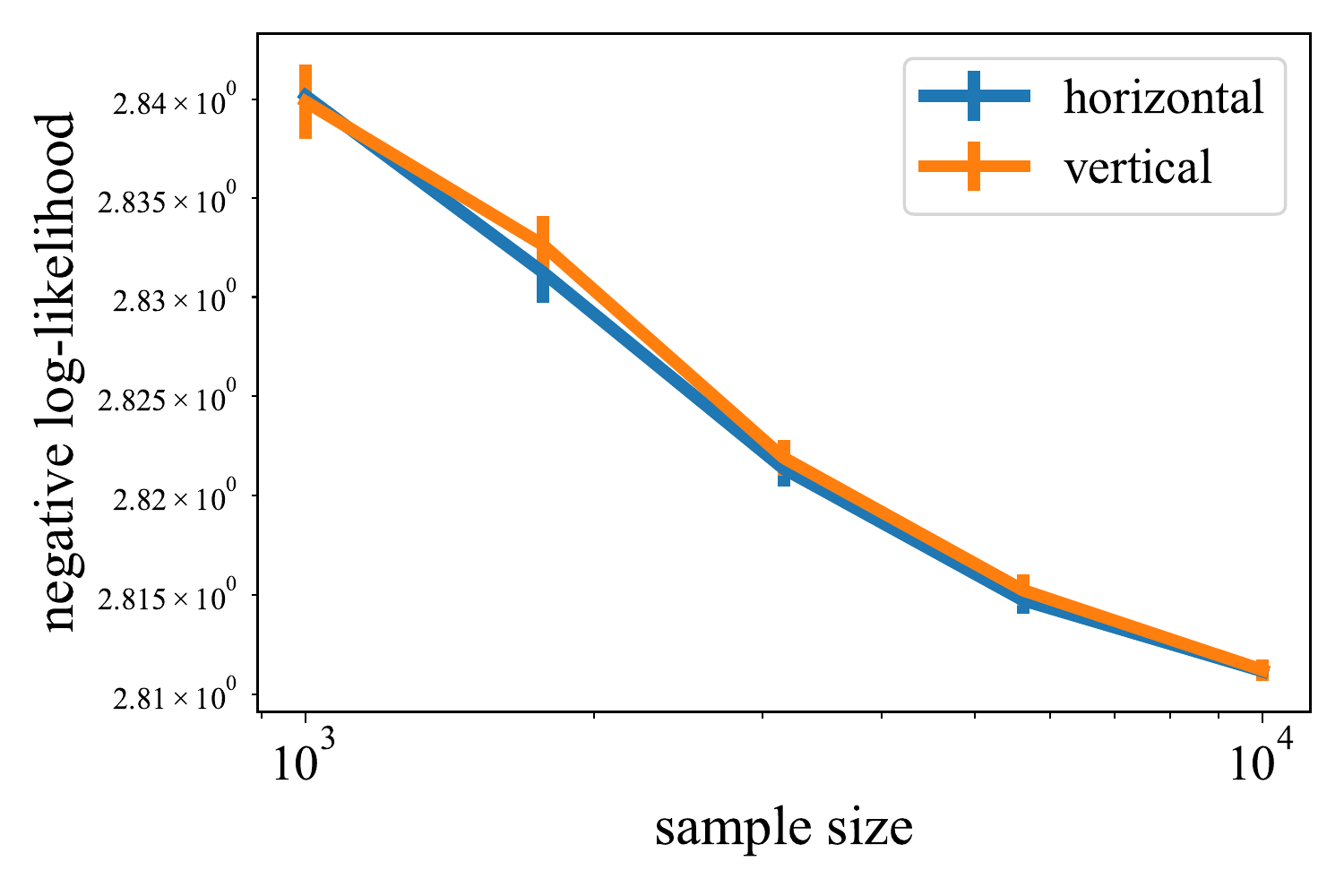}
& \hspace{-6mm}
\includegraphics[height=4.5cm,width=4.15cm]{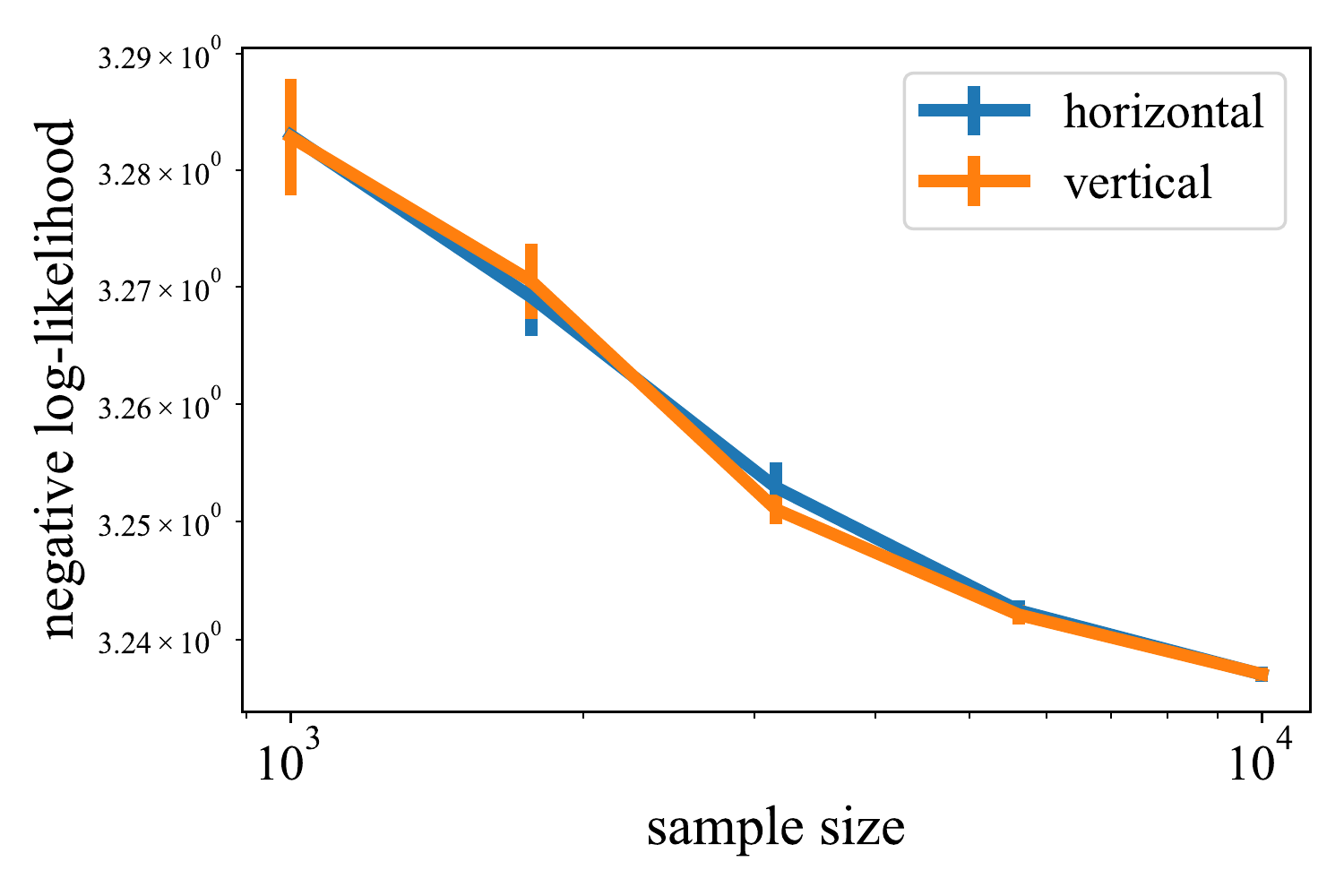}
&\hspace{-6mm}
\includegraphics[height=4.5cm,width=4.15cm]{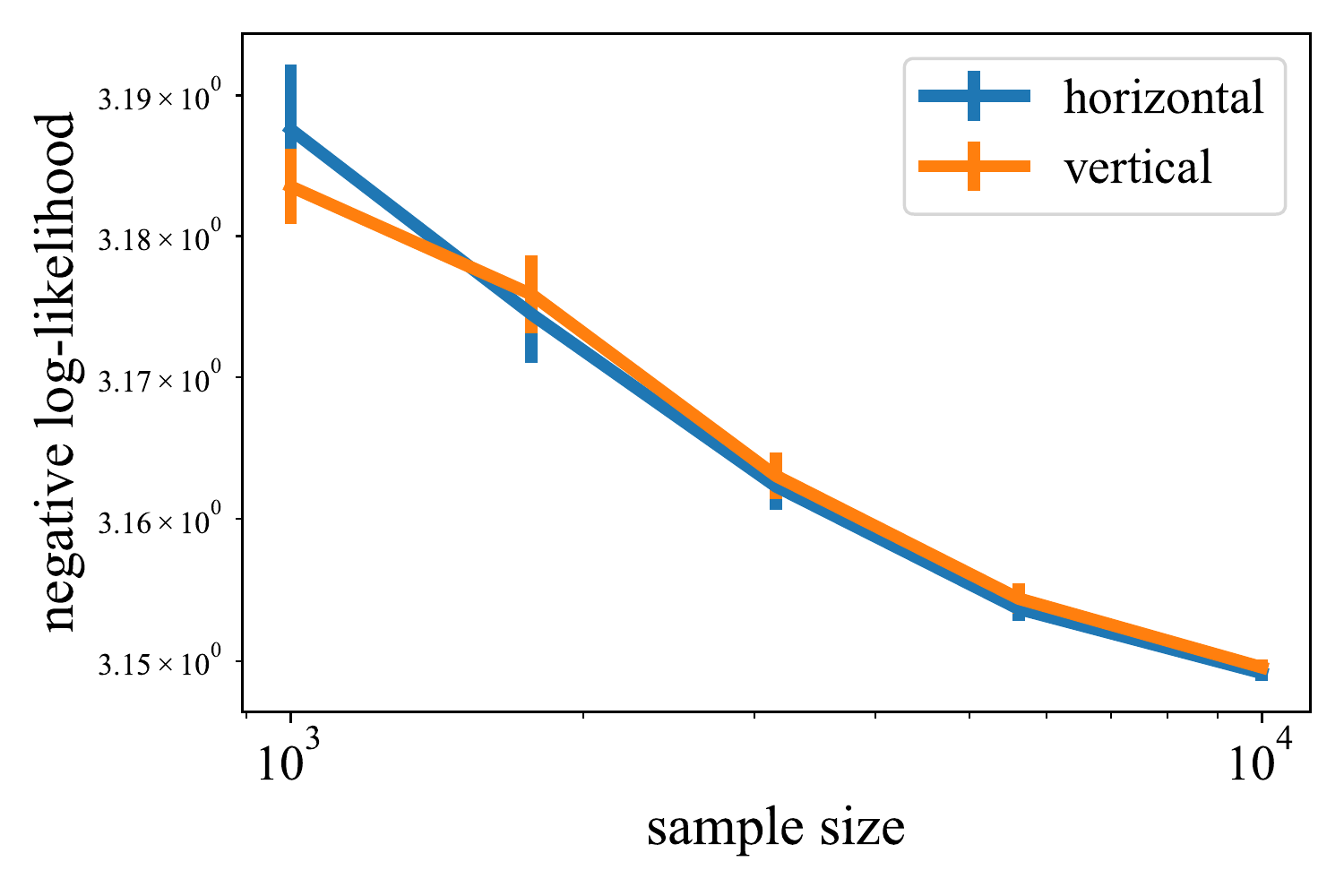}
&\hspace{-6mm}
\includegraphics[height=4.5cm,width=4.15cm]{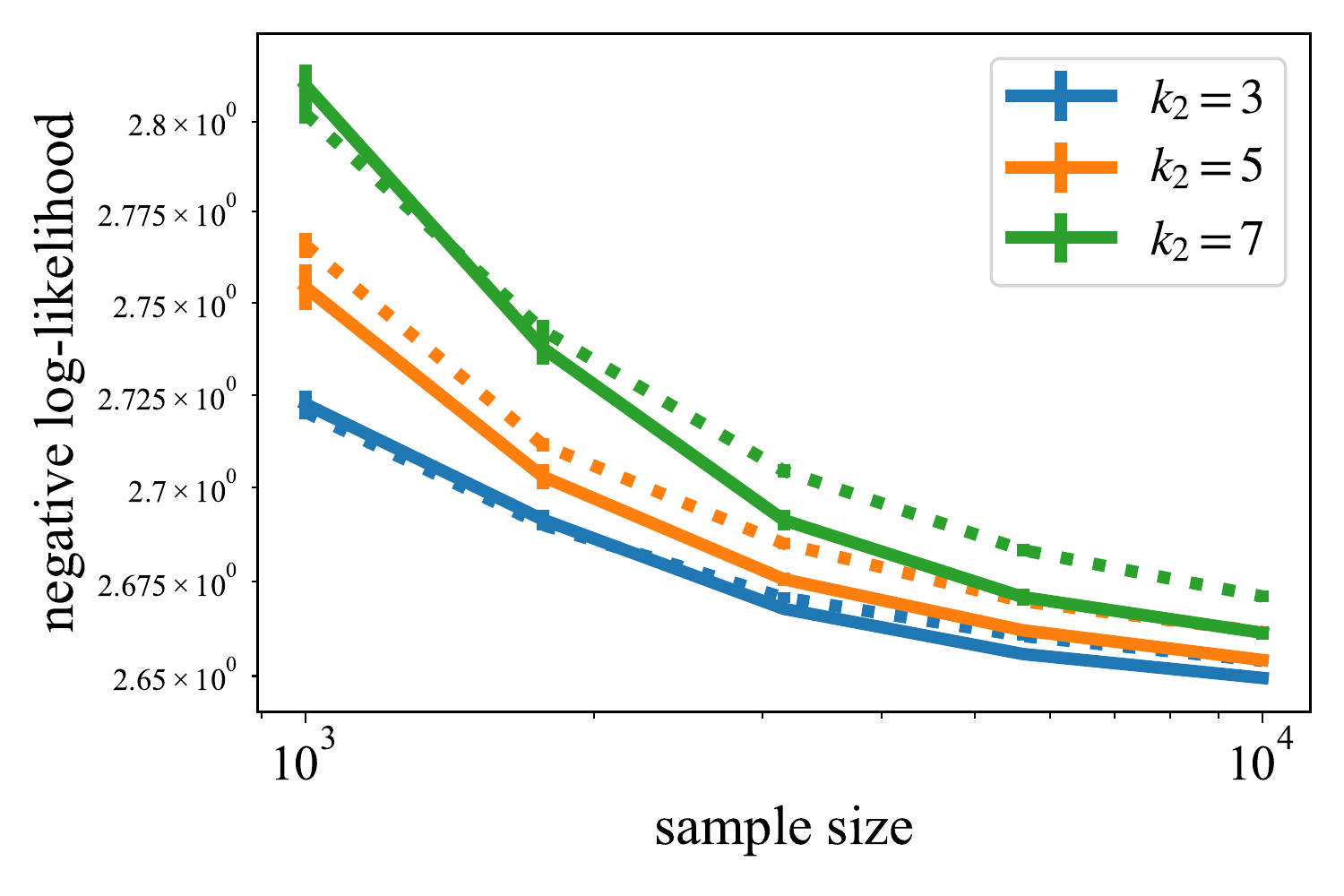} \\
\end{tabular}
\caption{Experimental comparison of convergence rates on 2D examples (top: density; bottom: rates) for the 3 datasets considered in \cite{pmlr-v108-kong20a} and the sine density with varying frequency $k_2\in\{3,5,7\}$.}
\label{fig:kong-chaudhuri}
\end{figure*}

Inspired by the pathological densities constructed in the proof of the ``no free lunch'' Theorem \ref{thm-slow-rates}, which are rapidly oscillating perturbations of the uniform density on the hypercube, we now consider the sine density on the hypercube, defined as
\[
f(x_1,\ldots,x_d) = 1+\prod_{j=1}^d \sin(2\pi k_jx_j)
\]
where 
$k_j\in\mathbb{Z}$ for $j\in[d]$, and
$x=(x_1,\ldots,x_d)\in[0,1]^d$. The smoothness of $f$, as measured by any $L^p$ norm of its derivative(s), decreases as the frequency $|k_j|$ increases. As such, $f$ parametrizes a natural family of functions to test our theoretical results concerning the statistical performance of KR map estimation as a function of the sample size $n$, the smoothness of the underlying target density, and the order of coordinates. The rightmost panels of Figure \ref{fig:kong-chaudhuri} plot the sine density with $k_1=1,k_2=3$ (top row) and convergence rates for the choices $k_1=1$ and $k_2\in\{3,5,7\}$ (bottom row). The dashed lines correspond to estimating the marginal $x_1$ first, followed by $x_2|x_1$; the solid lines indicate the reverse order. We again see an effect of coordinate ordering on convergence rates. It is also apparent that convergence slows down as $k_2$ increases and $f$ becomes less smooth.

\section{Discussion}
\label{sec:discussion}
\paragraph{Related work.}
Previous work on KR couplings has focused on existence and approximation questions in relation to universal approximation~\cite{bogachev2005triangular,alexandrova2006,naf}. This universal approximation property of KR maps has been a frequent motivation from a theoretical viewpoint of normalizing flows; see Sec. 3 in~\cite{kobyzev}. KR maps have been used for various learning and inference problems~\cite{kobyzev,spantini,sampling,moselhy}. 

Kong and Chaudhuri~\cite{pmlr-v108-kong20a} study the expressiveness of basic types of normalizing flows, such as planar flows, Sylvester flows, and Householder flows. In the one-dimensional setting, they show that such flows are universal approximators.
However, when the distributions lives in a $d$-dimensional space with $d\geq 2$, the authors provide a partially negative answer to the universal approximation power of these flows. For example, they exhibit cases where Sylvester flows cannot recover the target distributions. Their results can be seen as complementary to ours as we give examples of arbitrary slow statistical rates and we develop the consistency theory of KR-type flows~\cite{spantini}.

Jaini et al.~\cite{pmlr-v119-jaini20a} investigate the properties of the increasing triangular map required to push a tractable source density with known tails onto a desired target density.
Then they consider the general $d$-dimensional case and show similarly that imposing smoothness condition on the increasing triangular map will result in a target density with the same tail properties as the source. Such results suggest that without any assumption on the target distribution, the transport map might be too irregular to be estimated. These results echo our assumptions on the target to obtain fast rates and complement ours by focusing on the tail behavior while we focus on the consistency and the rates.

\paragraph{Conclusion.}
We have established the
uniform consistency and convergence rates of statistical estimators of the Kn\"othe-Rosenblatt rearrangement, highlighting the anisotropic geometry of function classes at play in triangular flows. Our results shed light on coordinate ordering and lead to statistical guarantees for Jacobian flows. 

\paragraph{Acknowledgements.}
The authors would like to thank the authors of~\cite{pmlr-v108-kong20a} for sharing their code to reproduce their results.
N. Irons is supported by a Shanahan Endowment Fellowship and a Eunice Kennedy Shriver National Institute of Child Health and Human Development training grant, T32 HD101442-01, to the Center for Studies in Demography \& Ecology at the University of Washington, and NSF CCF-2019844. M. Scetbon is supported by "Chaire d’excellence de l’IDEX Paris Saclay". S. Pal is supported by NSF DMS-2052239 and a PIMS CRG (PIHOT). Z. Harchaoui is supported by NSF CCF-2019844, NSF DMS-2134012,  NSF DMS-2023166, CIFAR-LMB, and faculty research awards. Part of this work was done while Z. Harchaoui was visiting the Simons Institute for the Theory of Computing. The authors would like to thank the Kantorovich Initiative of the Pacific Institute for the Mathematical Sciences (PIMS) for supporting this collaboration. This work was first presented at the Joint Statistical Meetings in August 2021.

\bibliography{biblio}

\begin{thebibliography}{10}

\bibitem{alexandrova2006}
D.~E. Alexandrova.
\newblock Convergence of triangular transformations of measures.
\newblock {\em Theory of Probability \& Its Applications}, 50(1):113--118,
  2006.

\bibitem{birge2}
L.~Birg\'{e}.
\newblock Approximation dans les espaces m\'{e}triques et th\'{e}orie de
  l'estimation.
\newblock {\em Z. Wahrscheinlichkeitstheor. Verw. Geb.}, 65:181--237, 1983.

\bibitem{birge}
L.~Birg\'{e}.
\newblock On estimating a density using {Hellinger} distance and some other
  strange facts.
\newblock {\em Probab. Th. Rel. Fields}, 71:271--291, 1986.

\bibitem{bogachev2005triangular}
V.~I. Bogachev, A.~V. Kolesnikov, and K.~V. Medvedev.
\newblock Triangular transformations of measures.
\newblock {\em Sbornik: Mathematics}, 196(3):309, 2005.

\bibitem{lugosi}
O.~Bousquet, S.~Boucheron, and G.~Lugosi.
\newblock {Introduction to statistical learning theory}.
\newblock In O.~Bousquet, U.~von Luxburg, and G.~R\"{a}tsch, editors, {\em
  Advanced Lectures on Machine Learning}, pages 169--207. Springer, Berlin,
  Heidelberg, 2004.

\bibitem{santambrogio}
G.~Carlier, A.~Galichon, and F.~Santambrogio.
\newblock From {Kn\"{o}the's} transport to {Brenier's} map and a continuation
  method for optimal transport.
\newblock {\em SIAM Journal on Mathematical Analysis}, 41(6):2554--2576, 2010.

\bibitem{devroye3}
L.~Devroye.
\newblock On arbitrarily slow rates of global convergence in density
  estimation.
\newblock {\em {Z. Wahrscheinlichkeitstheorie verw. Gebiete}}, 62:475--483,
  1983.

\bibitem{devroye2}
L.~Devroye.
\newblock Another proof of a slow convergence result of {Birg\'{e}}.
\newblock {\em {Statistics and Probability Letters}}, 23:63--67, 1995.

\bibitem{devroye}
L.~Devroye, L.~Gy\"{o}rfi, and G.~Lugosi.
\newblock {\em A Probabilistic Theory of Pattern Recognition}.
\newblock {Springer}, 1996.

\bibitem{nice}
L.~Dinh, D.~Krueger, and Y.~Bengio.
\newblock {NICE: Non-linear Independent Components Estimation}.
\newblock In {\em ICLR Workshop}, 2015.

\bibitem{nvp}
L.~Dinh, J.~Sohl-Dickstein, and S.~Bengio.
\newblock {Density estimation using Real NVP}.
\newblock In {\em ICLR}, 2017.

\bibitem{donsker-varadhan}
M.~Donsker and S.~Varadhan.
\newblock {Asymptotic evaluation of certain markov process expectations for
  large time, I}.
\newblock {\em {Communications on Pure and Applied Mathematics}}, 28(1):1--47,
  1975.

\bibitem{dudley67}
R.~Dudley.
\newblock {The sizes of compact subsets of {Hilbert} space and continuity of
  Gaussian processes}.
\newblock {\em {Journal of Functional Analysis}}, 1:290--330, 1967.

\bibitem{dudley68}
R.~Dudley.
\newblock {The speed of mean {Glivenko-Cantelli} convergence}.
\newblock {\em {Ann. Math. Statist.}}, 40:40--50, 1968.

\bibitem{moselhy}
T.~A. El~Moselhy and Y.~M. Marzouk.
\newblock {Bayesian inference with optimal maps}.
\newblock {\em Journal of Computational Physics}, 231(23):7815--7850, 2012.

\bibitem{made}
M.~Germain, K.~Gregor, I.~Murray, and H.~Larochelle.
\newblock {MADE: Masked Autoencoder for Distribution Estimation}.
\newblock In {\em ICML}, 2015.

\bibitem{ffjord}
W.~Grathwohl, R.~Chen, J.~Bettencourt, I.~Sutskever, and D.~Duvenaud.
\newblock {FFJORD: Free-form continuous dynamics for scalable reversible
  generative models}.
\newblock In {\em ICLR}, 2019.

\bibitem{gyorfi2002distribution}
L.~Gy{\"o}rfi, M.~Kohler, A.~Krzy{\.z}ak, and H.~Walk.
\newblock {\em A distribution-free theory of nonparametric regression},
  volume~1.
\newblock Springer, 2002.

\bibitem{naf}
C.-W. Huang, D.~Krueger, A.~Lacoste, and A.~Courville.
\newblock Neural autoregressive flows.
\newblock In {\em ICML}, 2018.

\bibitem{rigollet}
J.-C. H\"{u}tter and P.~Rigollet.
\newblock {Minimax estimation of smooth optimal transport maps}.
\newblock {\em The Annals of Statistics}, 49(2):1166--1194, 2021.

\bibitem{pmlr-v119-jaini20a}
P.~Jaini, I.~Kobyzev, Y.~Yu, and M.~Brubaker.
\newblock Tails of {Lipschitz} triangular flows.
\newblock In {\em International Conference on Machine Learning}, pages
  4673--4681. PMLR, 2020.

\bibitem{adam}
D.~Kingma and J.~Ba.
\newblock {Adam: A method for stochastic optimization}.
\newblock In {\em ICLR}, 2015.

\bibitem{glow}
D.~Kingma and P.~Dhariwal.
\newblock {Glow: Generative flow with invertible 1x1 convolutions}.
\newblock In {\em Advances in Neural Information Processing Systems}, pages
  10215--–10224, 2018.

\bibitem{knothe}
H.~Kn\"othe.
\newblock {Contributions to the theory of convex bodies}.
\newblock {\em {The Michigan Mathematical Journal}}, 4(1):39--52, 1957.

\bibitem{kobyzev}
I.~Kobyzev, S.~Prince, and M.~Brubaker.
\newblock Normalizing flows: An introduction and review of current methods.
\newblock {\em IEEE Transactions on Pattern Analysis and Machine Intelligence},
  2020.

\bibitem{pmlr-v108-kong20a}
Z.~Kong and K.~Chaudhuri.
\newblock The expressive power of a class of normalizing flow models.
\newblock In {\em International Conference on Artificial Intelligence and
  Statistics}, pages 3599--3609. PMLR, 2020.

\bibitem{sampling}
Y.~Marzouk, T.~Moselhy, M.~Parno, and A.~Spantini.
\newblock An introduction to sampling via measure transport.
\newblock {\em arXiv: 1602.05023}, 2016.

\bibitem{nickl}
R.~Nickl and B.~P\"otscher.
\newblock Bracketing metric entropy rates and empirical central limit theorems
  for function classes of {Besov-} and {Sobolev-Type}.
\newblock {\em {J Theor Probab}}, 20:177--199, 2007.

\bibitem{maf}
G.~Papamakarios, T.~Pavlakou, and I.~Murray.
\newblock Masked autoregressive flow for density estimation.
\newblock In {\em NeurIPS}, 2017.

\bibitem{rosenblatt}
M.~Rosenblatt.
\newblock {Remarks on a multivariate transformation}.
\newblock {\em {The Annals of Mathematical Statistics}}, 23(3):470–472, 1952.

\bibitem{rudin}
W.~Rudin.
\newblock {\em Principles of Mathematical Analysis}.
\newblock McGraw-Hill, 3rd edition, 1976.

\bibitem{otam}
F.~Santambrogio.
\newblock {\em Optimal Transport for Applied Mathematicians}.
\newblock {Birkh\"auser}, 2015.

\bibitem{spantini}
A.~Spantini, D.~Bigoni, and Y.~Marzouk.
\newblock Inference via low-dimensional couplings.
\newblock {\em The Journal of Machine Learning Research}, 19(1):2639--2709,
  2018.

\bibitem{vdv-jaw}
A.~van~der Vaart and J.~Wellner.
\newblock {\em Weak Convergence and Empirical Processes}.
\newblock Springer, New York, 1996.

\bibitem{villani}
C.~Villani.
\newblock {\em Topics in Optimal Transportation}.
\newblock {American Mathematical Society}, 2003.

\bibitem{hds}
M.~Wainwright.
\newblock {\em {High-Dimensional Statistics: A Non-Asymptotic Viewpoint}}.
\newblock {Cambridge University Press}, 2019.

\bibitem{mnn}
A.~Wehenkel and G.~Louppe.
\newblock Unconstrained monotonic neural networks.
\newblock In {\em Advances in Neural Information Processing Systems}, pages
  1543--1553, 2019.

\end{thebibliography}
\bibliographystyle{abbrv}

\clearpage
\appendix

\clearpage

This Supplement collects the detailed proofs of the theoretical results stated in the main text and additional materials. 
Sec.~\ref{sec:eq-proof} details the derivations of the Kullback-Leibler objective. Sec.~\ref{sec:slow-proof} details the proof of the slow rates in Sec.~\ref{sec:krmap} of the main text. Sec.~\ref{sec:entropy-proof} provides estimates of metric entropy from Sec.~\ref{sec:consistency} of the main text. Sec.~\ref{sec:consistency-proof} provides the proofs of the statistical consistency from Sec.~\ref{sec:consistency} of the main text. Sec.~\ref{sec:concave-proof} provides the proofs of the Sobolev rates under log-concavity of $g$ from Sec.~\ref{sec:fast} of the main text. Sec.~\ref{sec:ordering-proof} expands on the dimension ordering from Sec.~\ref{sec:fast} of the main text. Sec.~\ref{sec:flows-proof} expands on the extension to Jacobian flows from Sec.~\ref{sec:fast} of the main text. Sec. \ref{sec:separability} discusses conditions under which the KL objective (\ref{min:sample}) defining the KR estimator $S^n=(S^n_1,\ldots,S^n_d)$ is separable in the components $S^n_k$. 
\section{Detailed proofs}

\subsection{Kullback-Leibler objective}
\label{sec:eq-proof}
\begin{proof}[Derivation of (\ref{min:pop})]
By the change of variables formula (\ref{eq:COV}), the density $S\# f$ is given by
\[
(S\# f)(y) = f(S^{-1}(y))|\det (J(S^{-1})(y))|.
\]
Consequently, $\text{KL}(S\# f|g)$ rewrites as
\begin{align*}
\text{KL}(S\# f|g) &= \int_Y (S\# f)(y)\log\left(\frac{(S\# f)(y)}{g(y)}\right)dy \\
&= \int_{\mathcal{Y}} f(S^{-1}(y))|\det(J(S^{-1})(y))|\log\left(\frac{f(S^{-1}(y))|\det(J(S^{-1})(y))|}{g(y)}\right)dy  \\
&= \int_{\mathcal{Y}} f(S^{-1}(y))|\det(JS(S^{-1}(y)))|^{-1}\log\left(\frac{f(S^{-1}(y))}{g(y)|\det(JS(S^{-1}(y)))|}\right)dy \tag{inverse function theorem} \\
&= \int_{\mathcal{X}} f(x)|\det(JS(x))|^{-1}\log\left(\frac{f(x)}{g(S(x))|\det(JS(x))|}\right)\cdot|\det(JS(x))|\ dx \tag{$x:=S^{-1}(y)$}  \\
&= \int_{\mathcal{X}} f(x)\log\left(\frac{f(x)}{g(S(x))|\det(JS(x))|}\right) dx \\
&= \E_{X\sim f}\left\{\log f(X)-\log g(S(X))-\log|\det(JS(x))|\right\} \\
&= \E_{X\sim f}\left\{\log f(X)-\log g(S(X))-\sum_{k=1}^d\log D_kS_k(X)\right\}. 
\end{align*}
The last line follows because $S$ is assumed to be upper triangular and monotone non-decreasing, and therefore
\[
|\det(JS(x))| = \prod_{k=1}^d D_kS_k(x).
\]
Note that in the above calculation, we have also established that
\[
\text{KL}(S\# f|g)=\text{KL}(f|S^{-1}\# g),
\]
for any diffeomorphism $S:\mathbb{R}^d\to\mathbb{R}^d$, since
\begin{align*}
\text{KL}(S\# f|g) &= \int_{\mathcal{X}} f(x)\log\left(\frac{f(x)}{g(S(x))|\det(JS(x))|}\right) dx \\
&= \int_{\mathcal{X}} f(x)\log\left(\frac{f(x)}{(S^{-1}\# g)(x)}\right) dx \tag{change of variables}\\
&= \text{KL}(f|S^{-1}\# g).
\end{align*}
This completes the proof.
\end{proof}

\subsection{Slow rates}
\label{sec:slow-proof}
\subsubsection{Proof of Theorem \ref{thm-slow-rates}}

Our proof follows the argument in Section V of \cite{birge}.

\begin{proof}[Proof of Theorem~\ref{thm-slow-rates}]
For fixed $\epsilon\in(0,1)$, let $\tilde{h}(x;\epsilon)$ be a $C^\infty$ bump function on $\mathbb{R}$ satisfying
\begin{enumerate}
\item $0\le \tilde{h}(x;\epsilon)\le 1 \qquad\forall x\in\mathbb{R}$,
\item $\tilde{h}(x;\epsilon)=1$ on the interval $[\epsilon/4,1/2-\epsilon/4]$,
\item $\tilde{h}(x;\epsilon)=0$ outside of the interval $[0,1/2]$.
\end{enumerate}
Now for $r\in\mathbb{N}$, define the function $h_{\epsilon,r}:[0,1]^d\to[-1,1]$ by
\[
h_{\epsilon,r}(x) = \tilde{h}(x_1r;\epsilon) - \tilde{h}(x_1r-1/2;\epsilon). 
\]
It is clear that $h_{\epsilon,r}$ is smooth, $\sup_{x\in[0,1]^d}|h_{\epsilon,r}(x)|\le 1$, and $\int h_{\epsilon,r}(x)dx = 0$. Therefore, $1+h_{\epsilon,r}\in\mathcal{F}$ is a smooth Lebesgue density on $[0,1]^d$ uniformly bounded by 2. Also note that $|h_{\epsilon,r}(x)|=1$ whenever $[\epsilon/4\le x_1r\le1/2-\epsilon/4]$ or $[1/2+\epsilon/4\le x_1r\le1-\epsilon/4]$. Furthermore, the support of $h_{\epsilon,r}$ is contained in the set $[0,1/r]\times[0,1]^{d-1}$. It follows that
\begin{align*}
\text{TV}(1+h_{\epsilon,r},1-h_{\epsilon,r}) &= \frac{1}{2}\int |(1+h_{\epsilon,r}(x))-(1-h_{\epsilon,r}(x))|dx \\
&= \int |h_{\epsilon,r}(x)|dx \\
&\ge \int_{[\epsilon/4\le x_1r\le 1/2-\epsilon/4]\cup[1/2+\epsilon/4\le x_1r\le 1-\epsilon/4]} 1\ dx_1 \\
&= \frac{1-\epsilon}{r},
\end{align*}
and
\begin{align*}
\text{TV}(1+h_{\epsilon,r},1-h_{\epsilon,r}) &= \int |h_{\epsilon,r}(x)|dx \\
&\le \int_{[0,1/r]} 1\ dx_1 \\
&=1/r.
\end{align*}
As we will see, these bounds on the total variation imply that the perturbations $1+h_{\epsilon,r},1-h_{\epsilon,r}$ are sufficiently similar to make identification a challenging task, but sufficiently different to incur significant loss when mistaken for each other.

Now define the translates $h_i(x;\epsilon,r)=h_{\epsilon,r}\left(x_1-\frac{i-1}{r},x_2,\ldots,x_d\right)$,  which are disjointedly supported with support contained in $H_i=[(i-1)/r,i/r]\times[0,1]^{d-1}$ for $i=1,\ldots,r$. Hereafter, we suppress dependence of $h_i$ on $\epsilon,r$ for notational convenience. Consider the family of densities 
\[
\mathcal{F}(\epsilon,r)=\left\{1+\sum_{i=1}^r \delta_i h_i : \delta_i=\pm 1\right\}\subset\mathcal{F}
\] 
with cardinality $2^r$. For $\delta\in\{\pm1\}^r$ we write
\[
f_\delta = 1+\sum_{i=1}^r \delta_i h_i.
\]
The worst-case KL risk on $\mathcal{F}$ of any density estimate $f_n=T^n\# g$ derived from a KR map estimate $T^n$ can be bounded below as
\begin{align*}
\sup_{f\in\mathcal{F}}\E_f[\text{KL}(f|f_n)] &\ge \sup_{f\in\mathcal{F}(\epsilon,r)} \E_f[\text{KL}(f|f_n)] \\
&\ge \sup_{f\in\mathcal{F}(\epsilon,r)} \E_f[2\text{TV}(f,f_n)^2] \tag{Pinsker's inequality} \\
&\ge \sup_{f\in\mathcal{F}(\epsilon,r)} 2\E_f[\text{TV}(f,f_n)]^2 \tag{Jensen's inequality}.
\end{align*}
We aim to lower bound the total variation risk on $\mathcal{F}(\epsilon,r)$. We have 
\begin{align*}
\sup_{f\in\mathcal{F}(\epsilon,r)} \E_f[\text{TV}(f,f_n)] &\ge 2^{-r}\sum_{f\in\mathcal{F}(\epsilon,r)} \E_f[\text{TV}(f,f_n)] \\
&= 2^{-r}\sum_{\delta\in\{\pm 1\}^r} \E_{f_\delta}\left[\text{TV}\left(f_\delta,f_n\right)\right],
\end{align*}
i.e, the worst-case risk is larger than the Bayes risk associated to the uniform prior on $\mathcal{F}(\epsilon,r)$.
Now note that
\begin{align*}
\text{TV}\left(f_\delta,f_n\right) &= \frac{1}{2}\int |f_n(x)-f_\delta(x)|dx\\
&= \frac{1}{2}\int\left|f_n(x)-\left(1+\sum_{i=1}^r \delta_i h_i(x)\right)\right| dx \\
&= \frac{1}{2}\sum_{i=1}^r \int_{H_i}|f_n(x)-(1+\delta_i h_i(x))|dx \tag{$\{H_i\}$ are disjoint}.
\end{align*}
Define
\[
\ell_i(f_n) = \frac{1}{2}\int_{H_i}|f_n(x)-(1+ h_i(x))|dx, \qquad \ell_i'(f_n) = \frac{1}{2}\int_{H_i}|f_n(x)-(1-h_i(x))|dx
\]
and note that, by the triangle inequality,
\[
\ell_i(f_n)+\ell_i'(f_n) \ge \frac{1}{2}\int_{H_i}|(1+h_i(x))-(1-h_i(x))|dx =\int|h_i(x)|dx \ge \frac{1-\epsilon}{r}.
\]
Writing $F_\delta^n$ to denote the cdf of an iid sample of size $n$ from $f_\delta$, we then have
\begin{align*}
2^{-r}\sum_{\delta} \E_{f_\delta}\left[\text{TV}\left(f_\delta,f_n\right)\right] &= 2^{-r}\sum_{i=1}^r\left\{\sum_{\delta_i=1} \E_{f_\delta}[\ell_i(f_n)]+ \sum_{\delta_i=-1}\E_{f_\delta}[\ell_i'(f_n)]\right\} \\
&=\frac{1}{2}\sum_{i=1}^r \left\{\int \ell_i(f_n) d\left[2^{1-r}\sum_{\delta_i=1}F_\delta^n\right]+\int \ell_i'(f_n) d\left[2^{1-r}\sum_{\delta_i=-1}F_\delta^n\right]
\right\} \\
&\ge \frac{1}{2}\sum_{i=1}^r\left\{\int[\ell_i(f_n)+\ell_i'(f_n)]d\left(2^{1-r}\sum_{\delta_i=1}F_\delta^n\wedge 2^{1-r}\sum_{\delta_i=-1}F_\delta^n\right)\right\} \\
&\ge \frac{1-\epsilon}{2r}\sum_{i=1}^r \int d\left(2^{1-r}\sum_{\delta_i=1}F_\delta^n\wedge 2^{1-r}\sum_{\delta_i=-1}F_\delta^n\right) \\
&:= \frac{1-\epsilon}{2r}\sum_{i=1}^r \pi\left(2^{1-r}\sum_{\delta_i=1}F_\delta^n, 2^{1-r}\sum_{\delta_i=-1}F_\delta^n\right),
\end{align*}
where $x\wedge y=\min(x,y)$. In the last line we defined the testing affinity $\pi$ between two distribution functions $F_p,F_q$ with Lebesgue densities $p,q$,
\[
\pi(F_p,F_q) =\int d(F_p\wedge F_q) = \int (p\wedge q)dx,
\]
which satisfies the well known identity
\[
\pi(F_p,F_q)= 1-\text{TV}(p,q).
\]
Since $\min$ is concave, Jensen's inequality implies that
\begin{align*}
\pi\left(2^{1-r}\sum_{\delta_i=1}F_\delta^n, 2^{1-r}\sum_{\delta_i=-1}F_\delta^n\right) &\ge 2^{1-r}\sum_{(\delta,\delta')\in\Delta_i}\pi(F_\delta^n,F_{\delta'}^n),
\end{align*}
where $\Delta_i = \{(\delta,\delta'):\delta_i=1,\delta'_i=-1,\delta_j=\delta'_j \quad\forall j\neq i\}$. For any $(\delta,\delta')\in\Delta_i$, we have
\begin{align*}
\pi(F_\delta^n,F_{\delta'}^n) &= 1-\text{TV}(f_\delta,f_\delta') \\
&= 1-\frac{1}{2}\int |f_\delta(x)-f_{\delta'}(x)| dx \\
&= 1-\int |h_i(x)| dx \\
&\ge 1-\frac{1}{r}.
\end{align*}
Hence, we conclude that
\[
\pi\left(2^{1-r}\sum_{\delta_i=1}F_\delta^n, 2^{1-r}\sum_{\delta_i=-1}F_\delta^n\right) \ge 2^{1-r}\sum_{\Delta_i}\left(1-\frac{1}{r}\right)=1- \frac{1}{r}.
\]
Thus, putting this all together, we have shown that
\begin{align*}
\sup_{f\in\mathcal{F}(\epsilon,r)} \E_f[\text{TV}(f,f_n)] &\ge \frac{1-\epsilon}{2r}\sum_{i=1}^r\left(1-\frac{1}{r}\right) \\
&= \frac{1-\epsilon}{2}\left(1-\frac{1}{r}\right).
\end{align*}
Sending $\epsilon\to 0$ and $r\to\infty$, it follows that
\[
\sup_{f\in\mathcal{F}(\epsilon,r)} \E_f[\text{TV}(f,f_n)] \ge 1/2,
\]
and hence
\[
\sup_{f\in\mathcal{F}}\E_f[\text{KL}(f|f_n)] \ge \sup_{f\in\mathcal{F}(\epsilon,r)} 2\E_f[\text{TV}(f,f_n)]^2 \ge 1/2.
\]
This completes the proof.
\end{proof}

\subsubsection{Proof of Theorem \ref{thm:slow2}}

\begin{proof}
Let $h_n$ denote the first marginal of the density estimate $f_n=T^n\# g$. Note that $h_n$ is a density on $\mathbb{R}$. Defining $\pi_1:\mathbb{R}^d\to\mathbb{R}$ to be the projection along the first factor, we have 
\[
h_n = \pi_1\# f_n = (\pi_1\circ T^n)\# g.
\]
By Problem 7.5 in Devroye et al. (1996) \cite{devroye}, for any positive sequence $1/16\ge b_1\ge b_2\cdots$ converging to zero and any density estimate $h_n$ there exists a density $h$ on $\mathbb{R}$ such that
\[
\E\left\{\int |h(x)-h_n(x)|dx\right\} \ge b_n.
\]
Letting TV denote the total variation distance, this inequality can be rewritten as 
\[
2\E[\text{TV}(h,h_n)]\ge b_n.
\]
Setting $a_n=b_n^2/2$, which satisfies $a_1\le 1/512$, we find that
\begin{align*}
\E[\text{KL}(h|h_n))] &\ge \E[2\cdot \text{TV}(h,h_n)^2] \tag{Pinsker's inequality} \\
&\ge 2\E[\text{TV}(h,h_n)]^2 \tag{Jensen's inequality} \\
&\ge b_n^2/2 \\
&= a_n.
\end{align*}
Finally, let $f=h^{\otimes d}$ be the density on $\mathbb{R}^d$ defined as a $d$-fold product of $h$. By the chain rule of relative entropy, it follows that
\begin{align*}
\E[\text{KL}(f|f_n)] \ge \E[\text{KL}(h|h_n)] \ge a_n.
\end{align*}
This completes the proof.
\end{proof}

\subsection{Upper bounds on metric entropy}
\label{sec:entropy-proof}
We begin by defining relevant Sobolev function spaces, for which metric entropy bounds are known.

\begin{definition}
For $\mathcal{X}\subseteq\mathbb{R}^d$, define the function space
    \[
    D_s(\mathcal{X}) = \{\phi:D^\alpha \phi\text{ are uniformly continuous for all }|\alpha|\le s\}.
    \]
    and its subset
    \[
    C_s(\mathcal{X}) =\left\{\phi:\mathcal{X}\to\mathbb{R}:\sum_{0\le|\alpha|\le s} \|D^\alpha \phi\|_\infty < \infty\right\}\cap D_s(\mathcal{X}).
    \]
    endowed with the Sobolev norm 
    \[
    \|\phi\|_{H^{s,\infty}(\mathcal{X})} = \sum_{|\alpha|\le s}\|D^\alpha\phi\|_\infty.
    \]
\end{definition}

\begin{prop}[Corollary 3, Nickl and P\"otscher (2007) \cite{nickl}]
Assume $\mathcal{X}\subset\mathbb{R}^d$ is compact and let $\mathcal{F}$ be a bounded subset of $C_s(\mathcal{X})$ with respect to $\|\cdot\|_{H^{s,\infty}(\mathcal{X})}$ for some $s>0$.
The metric entropy of $\mathcal{F}$ in the $L^\infty$ norm is bounded as
\[
H(\epsilon, \mathcal{F},\|\cdot\|_\infty) 
\lesssim \epsilon^{-d/s}.
\]
\label{prop:nickl}
\end{prop}

With this result in hand, we can proceed to the proof of Proposition \ref{prop:entbound}.

\begin{proof}[Proof of Proposition~\ref{prop:entbound}]
This is a direct consequence of Proposition \ref{prop:nickl}.
Indeed, under Assumptions \ref{assumption:compact}-\ref{assumption:smooth} and Definition \ref{def:tsdm}, for every $S\in\mathcal{T}(s,d,M)$, every term in $\exp(\psi_S(x))$ is bounded away from 0 and $s$-smooth with uniformly bounded derivatives. Since $\log$ is smooth away from 0, it follows that that every $\psi_S\in \Psi(s,d,M)$ is $s$-smooth with uniformly bounded derivatives. Consequently, $\Psi(s,d,M)$ is a bounded subset of $C_s(\mathcal{X})$ for every $M>0$.
\end{proof}

We also provide a metric entropy bound for the space $\mathcal{T}(\mathfrak{s},d,M)$, which in particular establishes compactness of $\overline{\mathcal{T}(\mathfrak{s},d,M)}$ with respect to $\|\cdot\|_{\infty,d}$ (although this can also be proved with the Arzel\`{a}-Ascoli theorem).

\begin{prop}
Let $\mathfrak{s}=(s_1,\ldots,s_d)\in\mathbb{Z}^d_+$, $d_k = d-k+1$, and $\tilde{\sigma}_k=d_k\left((s_k+1)^{-1}+\sum_{j=k+1}^d s_j^{-1}\right)^{-1}$ for $k\in[d]$. Under Assumptions \ref{assumption:compact}, \ref{assumption:convex}, and \ref{assumption:multismooth} the space $\mathcal{T}(\mathfrak{s},d,M)$ is totally bounded (and therefore precompact) with respect to the uniform norm $\|\cdot\|_{\infty,d}$
with metric entropy satisfying
\begin{align*}
H(\epsilon,\mathcal{T}(\mathfrak{s},d,M),\|\cdot\|_{\infty,d})
\le& \sum_{k=1}^{d}c_k(\epsilon/2M)^{-d_k/\tilde{\sigma}_k}
\end{align*}
for some positive constants $c_k,k\in[d]$ independent of $\epsilon$ and $M$. 
\label{prop:KRmultientbound}
\end{prop}

This result relies on known metric entropy bounds for anisotropic smoothness classes.

\begin{prop}[Prop. 2.2, Birg\'{e} (1986) \cite{birge}]
Let $\mathfrak{s}=(s_1,\ldots,s_d)\in\mathbb{Z}^d_+$ and $\sigma=d\left(\sum_{j=1}^d s_j^{-1}\right)^{-1}$. Assume that $\Phi$ is a family of functions $\mathbb{R}^d\to\mathbb{R}$ with common compact convex support of dimension $d$ and satisfying
\[
\sup_{\phi\in\Phi,\alpha\preceq\mathfrak{s}}\|D^\alpha \phi\|_\infty <\infty.
\]
The metric entropy of $\Phi$ in the $L^\infty$ norm is bounded as
\[
H(\epsilon,\Phi,\|\cdot\|_\infty) 
\lesssim \epsilon^{-d/\sigma}.
\]
\label{prop:birge}
\end{prop}

We now proceed to the proof of Proposition \ref{prop:KRmultientbound}.

\begin{proof}[Proof of Proposition \ref{prop:KRmultientbound}]
For every $k\in[d]$, define the set of functions $\mathcal{X}_{k:d}\to\mathbb{R}$ given by
\[
\mathcal{T}_k = \{S_k:S\in\mathcal{T}(\mathfrak{s},d,M)\}.
\]
By Definition \ref{def:tsdm}, for each $k\in[d]$ we have that $\mathcal{T}_k/M$
satisfies the assumptions of Proposition \ref{prop:birge}, and hence
\[
n_k:=H(\epsilon, \mathcal{T}_k,L^\infty)
\le c_k
(\epsilon/M)^{-d_k/\tilde{\sigma}_k},
\]
for some $c_k>0$ independent of $\epsilon$ and $M$.

Now note that $\mathcal{T}(\mathfrak{s},d,M)\subseteq \prod_{k=1}^d \mathcal{T}_k$. For each $k\in[d]$, let $\{g_{k,1},\ldots, g_{k,n_k}\}$ be a minimal $\epsilon$-cover of $\mathcal{T}_k$ with respect to the $L^\infty$ norm, and define the subset 
\[
\mathcal{E}=\{f_{i_1,\ldots,i_d}=(g_{1,i_i},\ldots,g_{d,i_d}):1\le i_k\le n_k\}\subset \prod_{k=1}^d \mathcal{T}_k
\]
which has cardinality $\prod_{k=1}^d n_k$. Now fix an arbitrary $f\in\prod_{k=1}^d \mathcal{T}_k$. For each $k\in[d]$, we can find some $g_{k,i_k}$ such that $\|f_k-g_{k,i_k}\|_\infty\le\epsilon$. It follows that
\[
\|f-f_{i_1,\ldots,i_d}\|_{\infty,d}=\max_{k\in[d]}\|f_k-g_{k,i_k}\|_\infty\le\epsilon.
\]
Hence, $\mathcal{E}$ is an $\epsilon$-cover of $\prod_{k=1}^d\mathcal{T}_k$ and so
\[
H\left(\epsilon,\prod_{k=1}^d\mathcal{T}_k,\|\cdot\|_{\infty,d}\right) \le \log\left(\prod_{k=1}^d n_k\right) =\sum_{k=1}^d n_k.
\]

To conclude the proof, we claim that $\mathcal{T}(\mathfrak{s},d,M)\subseteq \prod_{k=1}^d \mathcal{T}_k$ implies that
\[
H\left(\epsilon,\mathcal{T}(\mathfrak{s},d,M),\|\cdot\|_{\infty,d}\right) \le H\left(\epsilon/2,\prod_{k=1}^d\mathcal{T}_k,\|\cdot\|_{\infty,d}\right).
\]
Indeed, suppose $\{f_1,\ldots,f_m\}\subseteq\prod_{k=1}^d\mathcal{T}_k$ is a finite $(\epsilon/2)$-cover of $\prod_{k=1}^d\mathcal{T}_k$. For each $j=1,\ldots,m$, if $\mathcal{T}(\mathfrak{s},d,M)\cap B_\infty(f_j,\epsilon/2)$ is non-empty, we define $g_j$ to be an element of this intersection. Here $B_\infty(f_j,\epsilon/2)$ denotes the ball of radius $\epsilon/2$ in $\prod_{k=1}^d\mathcal{T}_k$ centered at $f_j$ with respect to the norm $\|\cdot\|_{\infty,d}$. Now let $g\in\mathcal{T}(\mathfrak{s},d,M)$ be arbitrary. Since $g\in\prod_{k=1}^d\mathcal{T}_k$, there is some $j\in\{1,\ldots,m\}$ such that $\|g-f_j\|_{\infty,d}\le\epsilon/2$. This implies that $g_j$ is defined and hence
\[
\|g-g_j\|_{\infty,d} \le \|g-f_j\|_{\infty,d}+\|f_j-g_j\|_{\infty,d}\le \epsilon/2+\epsilon/2=\epsilon.
\]
It follows that $\{g_j\}$ is a finite $\epsilon$-cover of $\mathcal{T}(\mathfrak{s},d,M)$ with respect to $\|\cdot\|_{\infty,d}$, which establishes the claim and completes the proof. 

\end{proof}

\subsection{Statistical consistency}
\label{sec:consistency-proof}
\subsubsection{Proof of Lemma \ref{lemma:EPrate}}

\begin{proof}[Proof of Lemma \ref{lemma:EPrate}]
Assume $d<2s$. By Theorem 2.14.2 in van der Vaart and Wellner (1996) \cite{vdv-jaw}, since functions in $\Psi(s,d,M)$ are by definition uniformly bounded there exist positive constants $C_1,C_2$ such that
\begin{align*}
\E\|P_n-P\|_{\Psi(s,d,M)} 
&\le \frac{C_1C_2}{\sqrt{n}}\int_0^1 \sqrt{1+H(C_2\epsilon,\Psi(s,d,M),\|\cdot\|_\infty)}d\epsilon \\
&\lesssim\frac{1}{\sqrt{n}}\int_0^1 \sqrt{1+\epsilon^{-d/s}} d\epsilon \tag{Proposition \ref{prop:entbound}}\\
&\lesssim n^{-1/2}. \tag{$d<2s$}
\end{align*}
The last line follows since $d<2s$ implies that the integral on the right side is finite.

When $d\ge 2s$, the metric entropy integral above is no longer finite. In this case, we appeal to Dudley's metric entropy integral bound \cite{dudley67} (see also Theorem 5.22 in \cite{hds}), which states that there exists positive constants $C_3,D>0$ for which
\begin{align*}
\E\|P_n-P\|_{\Psi(s,d,M)} &\le \min_{\delta\in[0,D]}\left\{\delta+\frac{C_3}{\sqrt{n}}\int_\delta^D \sqrt{H(\epsilon,\Psi(s,d,M),\|\cdot\|_\infty)}d\epsilon\right\} \\
&\le \min_{\delta\in[0,D]}\left\{\delta+\frac{C_3\sqrt{c}}{\sqrt{n}}\int_\delta^D \epsilon^{-d/2s}d\epsilon\right\}. \tag{Proposition \ref{prop:entbound}}
\end{align*}
First assume $d=2s$. Evaluating the integral in Dudley's bound, we obtain
\begin{align*}
\E\|P_n-P\|_{\Psi(s,d,M)} &\le \min_{\delta\in[0,D]}\left\{\delta+\frac{C_3\sqrt{c}}{\sqrt{n}}[\log D-\log\delta]\right\}.    
\end{align*}
To minimize the expression on the right side in $\delta$, we differentiate with respect to $\delta$ and find where the derivative vanishes. The bound is optimized by choosing $\delta$ proportional to $n^{-1/2}$, which implies that $\E\|P_n-P\|_{\Psi(s,d,M)}\lesssim n^{-1/2}\log n$.

Now assume $d>2s$. Evaluating the integral in Dudley's bound, we obtain
\begin{align*}
\E\|P_n-P\|_{\Psi(s,d,M)} &\le \min_{\delta\in[0,D]}\left\{\delta+\frac{C_3\sqrt{c}}{\sqrt{n}}\int_\delta^\infty \epsilon^{-d/2s}d\epsilon\right\} \\
&= \min_{\delta\in[0,D]}\left\{\delta+\frac{C_4}{\sqrt{n}}\delta^{1-d/2s}\right\}.
\end{align*}
We optimize this bound by choosing $\delta$ proportional to $n^{-s/d}$, which implies that $\E\|P_n-P\|_{\Psi(s,d,M)}\lesssim n^{-s/d}$. 
\end{proof}

\subsubsection{Proof of Theorem \ref{thm:KLconv}}

Before continuing on to the proof of Theorem \ref{thm:KLconv}, we first prove that indeed $S^*\in\mathcal{T}(s,d,M^*)$ for $M^*>0$ sufficiently large, and similarly for $T^*$.

\begin{lemma}
Let $\mathfrak{s}=(s_1,\ldots,s_d)\in\mathbb{Z}^d_+$. 
Under Assumptions \ref{assumption:compact}, \ref{assumption:convex}, and \ref{assumption:multismooth}, there exists some $M^*>0$ such that the KR map $S^*$ from $f$ to $g$ lies in $\mathcal{T}(\mathfrak{s},d,M)$ for all $M\ge M^*$.
\label{lemma:KRmultismooth}
\end{lemma}

By symmetry, if we switch the roles of $f$ and $g$ in Assumption \ref{assumption:multismooth}, we see that the same is true for the KR map $T^*$ from $g$ to $f$. In particular, if $\mathfrak{s}=(s,\ldots,s)$ for some $s\in\mathbb{N}$, then we are in the case of Assumption \ref{assumption:smooth}. It then follows from Lemma \ref{lemma:KRmultismooth} that $S^*,T^*\in\mathcal{T}(s,d,M^*)$ for $M^*>0$ sufficiently large.

The proof of Lemma \ref{lemma:KRmultismooth} requires an auxiliary result establishing the regularity of marginal and conditional pdfs and quantile functions associated to a smooth density. 

\begin{lemma}
Let $h$ be a 
density on $\mathbb{R}^d$ with compact support 
$\mathcal{Z}$
. Assume further that $h$ is strictly positive on $\mathcal{Z}$ and $s$-smooth in the sense of Assumption \ref{assumption:multismooth}, where $s=(s_1,\ldots,s_d)\in\mathbb{Z}^d_+$. 
The following hold.
\begin{enumerate}
    \item For any index set $J\subseteq[d]$, the marginal density $h(x_J)=\int_{\mathbb{R}^{d-|J|}} h(x_J,x_{-J}) dx_{-J}$ is $s_J$-smooth. Similarly, the conditional density $h(x_{J}|x_{-J})=\frac{h(x_J,x_{-J})}{h(x_{-J})}$ is $s$-smooth as a function of $x=(x_J,x_{-J})$. 
    
    \item For any $J\subseteq[d],k
    \in J^c$, the univariate conditional cdf $H(x_k|x_J)=\int_{-\infty}^{x_k} h(y_k|x_J)dy_k$ is $(s_k+1,s_J)$-smooth as a function of $(x_k,x_J)$. 
    
    \item Assume further that $\mathcal{Z}$ is convex.
    Then the conditional quantile function $H^{-1}(p_k|x_J)$ 
    is $(s_k+1,s_J)$-smooth as a function of $(p_k,x_J)$ for any $J\subseteq[d],k\in J^c,p_k\in[0,1]$.
\end{enumerate}
\label{lemma:multismooth}
\end{lemma}

\begin{proof}
For (1), the regularity assumptions on $h$ imply that we can differentiate under the integral. For any multi-index $\alpha\preceq s$ satisfying $\alpha_{-J}=0$ we have
\begin{align*}
D^\alpha h(x_J) &= \int D^\alpha  h(x_J,x_{-J})dx_{-J}.
\end{align*}
Since $D^\alpha  h(x_J,x_{-J})$ is continuous and compactly supported by hypothesis, the dominated convergence theorem implies that for any sequence $x_{J,n}\to x_J$ we have
\[
D^\alpha h(x_{J,n}) = \int D^\alpha  h(x_{J,n},x_{-J})dx_{-J} \to \int D^\alpha  h(x_J,x_{-J})dx_{-J} = D^\alpha h(x_J).
\]
This proves that $h(x_J)$ is $s_J$-smooth. Since
$h>0$ 
on its support and
$x\mapsto 1/x$ is a smooth function for $x>0$, this implies that $1/h(x_{-J})$ is $s_{-J}$-smooth. As products of differentiable functions are differentiable, we conclude that $h(x_J|x_{-J})=h(x_J,x_{-J})/h(x_{-J})$ is $s$-smooth.

For (2), the result from (1) implies that the univariate conditional density $h(x_k|x_J)$ is $(s_k,s_J)$-smooth. For any multi-index $\alpha\preceq s$ satisfying $\alpha_k=0$, we can repeat the differentiation-under-the-integral argument above, combined with the dominated convergence theorem, to see that $D^\alpha H(x_k|x_J)$ is continuous. If $\alpha_k\neq 0$, we apply the fundamental theorem of calculus to take care of one derivative in $x_k$ and then recall that $h(x_k|x_J)$ is $(s_k,s_J)$-smooth to complete the proof, appealing to Clairaut's theorem to exchange the order of differentiation. Note that this shows that $H(x_k|x_J)$ is $(s_k+1,s_J)$-smooth as a function of $(x_k,x_J)$.

For (3), the assumption that $\mathcal{Z}$ is convex, combined with the fact that $h>0$ 
on its support, 
implies that $H(x_k|x_J)$ is strictly increasing in $x_k$. As such, $H^{-1}(p_k|x_J)$ is defined, strictly increasing, and continuous in $p_k$. Since $H(x_k|x_J)$ is differentiable in $x_k$, we have
\[
\frac{\partial}{\partial p_k}H^{-1}(p_k|x_J) = \frac{1}{\frac{\partial}{\partial x_k}H(H^{-1}(p_k|x_J)|x_J)} = \frac{1}{h(H^{-1}(p_k|x_J)|x_J)},
\]
which is continuous in $p_k$ as a composition of continuous functions. Continuing in this way, we see that $H^{-1}(p_k|x_J)$ is $(s_k+1)$-smooth in $p_k$. 

For the other variables, note that the following relation holds generally:
\begin{align*}
H(H^{-1}(p_k|x_J)|x_J)&=p_k.    
\end{align*}
Defining the function
\[
u(p_k,x_J,x_k) = H(x_k|x_J)-p_k,
\]
we see that $u$ is $(\infty,s_J,s_k+1)$-smooth in 
its arguments 
$(p_k,x_J,x_k)$ and satisfies
\[
u(p_k,x_J,H^{-1}(p_k|x_J))=0.
\]
Furthermore $\frac{\partial}{\partial x_k}u>0$ by assumption, since $h>0$ 
on its support. 
Hence we can appeal to the implicit function theorem (stated precisely below) to conclude that $H^{-1}(p_k|x_J)$ is $(s_k+1,s_J)$-smooth in $(p_k,x_J)$. For example, for any $j\in J$ we can evaluate the partial derivative in $x_j$ by implicitly differentiating the above relation. Letting $x_k(p_k,x_J) = H^{-1}(p_k|x_J)$ we obtain
\begin{align*}
\frac{\partial}{\partial x_j}H(x_k(p_k,x_J)|x_J) &= \frac{\partial H}{\partial x_k}(x_k(p_k,x_J)|x_J)\cdot \frac{\partial x_k}{\partial x_j}(p_k,x_J) + \frac{\partial H}{\partial x_j}(x_k|x_J) \\
&= \frac{\partial p_k}{\partial x_j} \\
&= 0.
\end{align*}
Rearranging terms and noting that $\frac{\partial x_k}{\partial x_j}(p_k,x_J)=\frac{\partial}{\partial x_j}H^{-1}(p_k|x_J)$, we find that
\[
\frac{\partial}{\partial x_j}H^{-1}(p_k|x_J) = -\frac{\frac{\partial H}{\partial x_j}(H^{-1}(p_k|x_J)|x_J)}{h(H^{-1}(p_k|x_J)|x_J)}.
\]
\end{proof}

\begin{theorem}[Implicit function theorem \cite{rudin}]
Let $u:\mathbb{R}^{n+m}\to\mathbb{R}^m$ be continuously differentiable and suppose $u(a,b)=0$ for some $(a,b)\in\mathbb{R}^n\times\mathbb{R}^m$. 
For $(x,y)\in\mathbb{R}^n\times\mathbb{R}^m$ define the partial Jacobians
\[
J_nu(x,y)= \left[\frac{\partial}{\partial x_j}u_i(a,b)\right]_{(i,j)\in[m]\times[n]}\qquad J_m u(x,y) = \left[\frac{\partial}{\partial y_j}u_i(a,b)\right]_{(i,j)\in[m]\times[m]},
\]
and suppose that $J_m u(a,b)$ is invertible. Then there exists an open neighborhood $V\subseteq\mathbb{R}^n$ of $a$ such that there exists a unique continuously differentiable function $v:V\to\mathbb{R}^m$ satisfying $v(a)=b$, $u(x,v(x))=0$ for all $x\in V$, and 
\[
Jv(x) = -J_m u(x,v(x))^{-1}J_n u(x,v(x)).
\]
\end{theorem}

We now proceed to the proof of Lemma \ref{lemma:KRmultismooth}.

\begin{proof}[Proof of Lemma \ref{lemma:KRmultismooth}]
By definition of the KR map, for each $k\in[d]$ we have
\[
S^*_k(x_k|x_{(k+1):d}) = G_k^{-1}(F_k(x_k|x_{(k+1):d})|S^*_{(k+1):d}(x)).
\]
The requisite smoothness in Definition \ref{def:tsdm} of $S^*_{k:d}$ then follows from the chain rule of differentiation, appealing to Assumption \ref{assumption:multismooth}, and an application of Lemma \ref{lemma:multismooth}. Note also that $x_k \mapsto G_k^{-1}(F_k(x_k|x_{(k+1):d})|S^*_{(k+1):d}(x))$ is strictly increasing as a composition of strictly increasing functions, since $f,g$ are bounded away from 0 on their supports by assumption.
Defining $\mathcal{T}(\mathfrak{s},d)\subset\mathcal{T}$ as the subset of strictly increasing triangular maps that are $\mathfrak{s}$-smooth (although not necessarily with uniformly bounded derivatives), we see that $S^*\in\mathcal{T}(\mathfrak{s},d)$. Since $\mathcal{T}(\mathfrak{s},d)=\cup_{M>0}\mathcal{T}(\mathfrak{s},d,M)$, there exists some $M^*>0$ for which $S^*\in\mathcal{T}(\mathfrak{s},d,M^*)$. Since $\mathcal{T}(\mathfrak{s},d,M_1)\subseteq\mathcal{T}(\mathfrak{s},d,M_2)$ for all $M_1\le M_2$ by definition, we conclude that $S^*\in\mathcal{T}(\mathfrak{s},d,M)$ for all $M\ge M^*$.
\end{proof}

With Lemma \ref{lemma:KRmultismooth} in hand, we can now prove Theorem \ref{thm:KLconv}.

\begin{proof}[Proof of Theorem \ref{thm:KLconv}]
The theorem is a direct result of inequality (\ref{eq:lossbound}) and Lemma \ref{lemma:EPrate}. Indeed, since $S^*\in\mathcal{T}(s,d,M^*)$ by Lemma \ref{lemma:KRmultismooth}, we have $\inf_{S\in\mathcal{T}(s,d,M^*)}P\psi_S=P\psi_{S^*}=0$, and therefore
\begin{align*}
\E[P\psi_{S^n}] &= \E\left\{P\psi_{S^n}-\inf_{S\in\mathcal{T}(s,d,M^*)} P\psi_{S}\right\}\\
&\le \E\left\{2\|P_n-P\|_{\Psi(s,d,M^*)}+ R_n\right\} \tag{inequality (\ref{eq:lossbound})} \\
&\lesssim \E\|P_n-P\|_{\Psi(s,d,M^*)}.
\end{align*}
Appealing to Lemma \ref{lemma:EPrate} establishes the stated bound on $\E[P\psi_{S^n}]$. To conclude that $P\psi_{S^n} \stackrel{p}{\to}0$, we apply Markov's inequality:
\[
P(P\psi_{S^n} \ge \epsilon) \le \epsilon^{-1}\E[P\psi_{S^n}] \to 0.
\]
As $\epsilon > 0$ was arbitrary, this completes the proof.
\end{proof}

\subsubsection{Proof of KL lower semicontinuity}

Although Theorem \ref{thm:KLconv} only establishes a weak form of consistency in terms of the KL divergence, we leverage this result to prove strong consistency, in the sense of uniform convergence of $S^n$ to $S^*$ in probability, in Theorem \ref{thm:uniformconv}. The proof requires understanding the regularity of the KL divergence with respect to the topology induced by the $\|\cdot\|_{\infty,d}$ norm. Lemma \ref{lemma:lsc} establishes that KL is lower semicontinuous with respect to this topology. It relies on the weak lower semicontinuity of KL proved by Donsker and Varadhan using their dual representation in Lemma 2.1 of \cite{donsker-varadhan}.

\begin{lemma}
Under Assumptions \ref{assumption:compact}-\ref{assumption:smooth}, the functional $S\mapsto P\psi_S$ on the domain $\mathcal{T}$ is lower semicontinuous with respect to the uniform norm $\|\cdot\|_{\infty,d}$.
\label{lemma:lsc}
\end{lemma}
\begin{proof}
Assume $\|S^n-S\|_{\infty,d}\to 0$ for fixed maps $S^n,S\in\mathcal{T}$. We first claim that $\nu_n := S^n\#\mu\rightsquigarrow S\#\mu =:\nu$ for any probability measure $\mu$. We will show that $\int h\ d\nu_n \to \int h\ d\nu $
for all bounded continuous functions $h\in C_b(\mathbb{R}^d)$. Note that $\int h\ d\nu_n = \int h\circ S^n\ d\mu$.
By assumption, $h\circ S^n$ is bounded and measurable with $\|h\circ S^n\|_\infty \le \|h\|_\infty$ for all $n$. Furthermore, $h\circ S^n\to h\circ S$ pointwise, since $h$ is continuous. Then by the dominated convergence theorem,
$\int |h\circ S^n-h\circ S|\ d\mu \to 0$.
In particular, this implies that $\int h\ d\nu_n = \int h\circ S^n\ d\mu \to \int h\circ S\ d\mu = \int h\ d\nu$.
This proves the claim.

Combining this fact with the weak lower semicontinuity of KL divergence \cite{donsker-varadhan} proves the lemma. Indeed, since $\|S^n-S\|_{\infty,d}\to 0$ implies that $S^n\#\mu\rightsquigarrow S\#\mu$ for any $\mu$, it follows that
\begin{align*}
P\psi_S &=\text{KL}(S\#f|g) \\
&\le \liminf_{n\to\infty} \text{KL}(S^n\#f|g) \tag{KL is weakly lower semicontinuous} \\
&=\liminf_{n\to\infty} P\psi_{S^n}.
\end{align*}
Thus, $S\mapsto P\psi_S$ on $\mathcal{T}$ is lower semicontinuous with respect to $\|\cdot\|_{\infty,d}$.
\end{proof}

As a corollary, we also obtain an existence result for minimizers of $S\mapsto P\psi_S$ on $\overline{\mathcal{T}(s,d,M)}$.

\begin{corollary}
Under Assumptions \ref{assumption:compact}-\ref{assumption:smooth}, for any $M>0$ the minimum
$
\inf_{S\in\overline{\mathcal{T}(s,d,M)}} P\psi_S
$
is attained.
\end{corollary}
\begin{proof}
This follows from the direct method of calculus of variations, since $\overline{\mathcal{T}(s,d,M)}$ is compact with respect to $\|\cdot\|_{\infty,d}$ (Proposition \ref{prop:KRmultientbound})
and $S\mapsto P\psi_S$ is bounded below and lower semicontinuous with respect to $\|\cdot\|_{\infty,d}$. 
\end{proof}

\subsubsection{Uniform consistency of the inverse map}

We have proved consistency of the estimator $S^n$ of the sampling map $S^*$, which pushes forward the target density $f$ to the source density $g$. However, we require knowledge of the direct map $T^*=(S^*)^{-1}$ to generate new samples from $f$ by pushing forward samples from $g$ under $T^*$. In this section we prove consistency and a rate of convergence of the estimator $T^n=(S^n)^{-1}$ of the direct map $T^*$. 

First note that $\text{KL}(S\# f|g)=\text{KL}(f|S^{-1}\# g)$ for any diffeomorphism $S:\mathbb{R}^d\to\mathbb{R}^d$, as we proved in Section \ref{sec:eq-proof}. As such, the consistency and rate of convergence of $\text{KL}(S^n\# f|g)$ obtained in Theorem \ref{thm:KLconv} yield the same results for the estimator $T^n$ in terms of $\text{KL}(f|T^n\# g)$ under identical assumptions. We can also establish uniform consistency of $T^n$ as we did for $S^n$ in Theorem \ref{thm:uniformconv}, although the proof of this fact requires a bit more work.

\begin{theorem}
Suppose Assumptions \ref{assumption:compact}-\ref{assumption:smooth} hold.
Let $S^n$ be any near-optimizer of the functional $S\mapsto P_n\psi_S$ on $\mathcal{T}(s,d,M^*)$, i.e., suppose
\[
P_n\psi_{S^n} = \inf_{S\in\mathcal{T}(s,d,M^*)} P_n\psi_S+o_P(1).
\]
Let $T^n=(S^n)^{-1}$. Then $\|T^n-T^*\|_{\infty,d}\stackrel{p}{\to}0$, i.e, $T^n$ is a uniformly consistent estimator of $T^*$.
\label{thm:invconsist}
\end{theorem}

The proof of Theorem \ref{thm:invconsist} relies heavily on the bounds on $S^n$ and its derivatives posited in Definition \ref{def:tsdm}, which we utilize in conjunction with the inverse function theorem to uniformly bound the Jacobian $JT^n(y)=(JS^n(T^n(y)))^{-1}$ over $y\in\mathcal{Y}$ and $n\in \mathbb{N}$, thereby establishing uniform equicontinuity of the family of estimators $\{T^n\}_{n=1}^\infty$. We combine this uniform equicontinuity with the uniform consistency of $S^n$ from Theorem \ref{thm:uniformconv} to complete the proof.

First we establish a lemma that allows us to bound the derivatives of the inverse map estimates $T^n$.

\begin{lemma}
Suppose $A\in\mathbb{R}^{d\times d}$ is an invertible upper triangular matrix satisfying 
\[
\max_{i,j\in[d], i<j}|A_{ij}|\le L \qquad\text{and}\qquad \min_{j\in[d]}|A_{jj}|\ge 1/M
\]
for some positive $L,M>0$. Then $A^{-1}$ is upper triangular and the diagonal entries are bounded as
\[
\max_{j\in[d]}|A^{-1}_{jj}|\le M.
\]
Furthermore, the superdiagonal terms $i,j\in[d]$ with $i<j$ are bounded as
\[
|A_{ij}^{-1}|\le M^2L(ML+1)^{j-i-1}.
\]
\label{lemma:matrixbound}
\end{lemma}

\begin{proof}
Let $D=\text{diag}(A)$ denote the matrix with diagonal entries $D_{jj}=A_{jj}$ for $j\in[d]$ and zeros elsewhere. Note that $D$ is invertible since $A$ is, and $D^{-1}$ is diagonal with entries $D^{-1}_{jj}=1/A_{jj}$ for $j\in[d]$. Let $U=A-D$ denote the strictly upper triangular part of $A$. Now note that
\[
A^{-1} = (D+U)^{-1} =[D(I+D^{-1}U)]^{-1} = (I+D^{-1}U)^{-1}D^{-1}.
\]
To calculate $(I+D^{-1}U)^{-1}$ we make use of the matrix identity
\[
(I-X)\left[\sum_{k=0}^{d-1}X^k\right] = I-X^d.
\]
We plug in $X=-D^{-1}U$ and note that $D^{-1}U$ is strictly upper triangular, which implies that it is nilpotent of degree at most $d$, i.e., $(D^{-1}U)^d=0$. Hence, we obtain
\[
(I+D^{-1}U)\left[\sum_{k=0}^{d-1}(-D^{-1}U)^k\right] = I-(-D^{-1}U)^d=I,
\]
which implies that $(I+D^{-1}U)^{-1} = \sum_{k=0}^{d-1}(-D^{-1}U)^k$. It follows that
\[
A^{-1} = \sum_{k=0}^{d-1}(-D^{-1}U)^k D^{-1}.
\]
This shows that $A^{-1}$ is upper triangular as a sum of products of upper triangular matrices. Note also that the terms in the sum with $k>0$ are strictly upper triangular. Hence, we see that $\text{diag}(A^{-1})=D^{-1}$ and therefore the diagonal entries of $A^{-1}$ satisfy $
|A^{-1}_{jj}| = \frac{1}{|A_{jj}|}\le M
$
by hypothesis.

Now we bound the superdiagonal entries. By repeated application of the triangle inequality and appealing to the bounds on $A_{ij}$, we can bound $A^{-1}_{ij}$ for $i<j$ as
\begin{align*}
|A^{-1}_{ij}| &= \left|\sum_{k=0}^{d-1} [(D^{-1}U)^k D^{-1}]_{ij}\right| \\
&= \left|\sum_{k=1}^{d-1} [(D^{-1}U)^k D^{-1}]_{ij}\right| \tag{$i<j$}\\
&\le \sum_{k=1}^{d-1} |[(D^{-1}U)^k D^{-1}]_{ij}| \\
&\le \sum_{k=1}^{d-1} |[(MU)^k M]_{ij}| \tag{$|D^{-1}_{jj}|=|A^{-1}_{jj}|\le M$}\\
&= \sum_{k=1}^{d-1} M^{k+1}|(U^k)_{ij}|.
\end{align*}
Now let $V$ denote the $d\times d$ matrix with ones strictly above the diagonal and zeros elsewhere, i.e., 
\[
V_{ij} = 
\begin{cases}
1, &i < j, \\
0, & i\ge j.
\end{cases}
\]
We now claim that for $k\ge 1$,
\[
(V^k)_{ij} =
\begin{cases}
\binom{j-i-1}{k-1}, & k\le j-i \\
0, &\text{else}.
\end{cases}
\]
Our proof of the claim proceeds by induction. The statement is clearly true for the base case $k=1$ by definition of $V$. Suppose the claim holds up to $k$. It follows that 
\begin{align*}
(V^{k+1})_{ij} &= (V^k\cdot V)_{ij} \\
&= \sum_{\ell=1}^d (V^k)_{i\ell}V_{\ell j} \\
&= \sum_{\ell=1}^d \binom{\ell-i-1}{k-1}1_{[k\le \ell-i]}\cdot 1_{[\ell<j]} \tag{inductive step}\\
&= 1_{[k+1\le j-i]}\sum_{\ell=i+k}^{j-1} \binom{\ell-i-1}{k-1} \\
&= 1_{[k+1\le j-i]}\sum_{p=k-1}^{j-i-2} \binom{p}{k-1} \tag{$p:=\ell-i-1$}\\
&= 1_{[k+1\le j-i]}\left\{1+\sum_{p=k}^{j-i-2} \binom{p}{k-1}\right\} \\
&= 1_{[k+1\le j-i]}\left\{1+\sum_{p=k}^{j-i-2} \left[\binom{p+1}{k}-\binom{p}{k}\right]\right\} \tag{Pascal's formula}\\
&= 1_{[k+1\le j-i]}\left\{1+\binom{j-i-1}{k}-\binom{k}{k}\right\} \tag{telescoping sum}\\
&= 1_{k+1\le j-i}\binom{j-i-1}{k}.
\end{align*}
This proves the claim for $k+1$. The result then follows by induction.

Note now that $|U_{ij}|\le LV_{ij}$ for all $i,j\in[d]$ by hypothesis. Again applying the triangle inequality and the assumption $\max_{i<j}|A_{ij}|\le L$, it follows that
\begin{align*}
|A^{-1}_{ij}| &\le \sum_{k=1}^{d-1} M^{k+1}|(U^k)_{ij}| \\
&\le \sum_{k=1}^{d-1} M^{k+1}|((LV)^k)_{ij}| \\
&= \sum_{k=1}^{d-1}M^{k+1}L^k\binom{j-i-1}{k-1}1_{[k\le j-i]} \\
&= \sum_{k=1}^{j-i}M^{k+1}L^k\binom{j-i-1}{k-1}\\
&= M^2L\sum_{\ell=0}^{j-i-1}(ML)^{\ell}1^{j-i-1-\ell}\binom{j-i-1}{\ell} \tag{$\ell:=k-1$}\\
&= M^2L(ML+1)^{j-i-1}. \tag{binomial formula}
\end{align*}
This completes the proof.
\end{proof}

We now use this lemma to establish strong consistency of the direct map estimator.

\begin{proof}[Proof of Theorem \ref{thm:invconsist}]
First note that convexity of $\mathcal{X}$ along with positivity $f>0$ on $\mathcal{X}$ implies that the inverse KR map $T^*=(S^*)^{-1}$
is well-defined and continuous, as noted in Lemma \ref{lemma:KRmultismooth}. Furthermore, by definition of $\mathcal{T}(s,d,M^*)$, the inverse maps $T^n=(S^n)^{-1}$ exist and are continuous also.

We first prove that the sequence $\{T^n\}_{n=1}^\infty$ is uniformly equicontinuous with respect to the $\ell_\infty$ norm $\|\cdot\|_\infty$ on $\mathbb{R}^d$.
For a function $S:\mathcal{X}\to\mathcal{Y}$ let $JS(x)$ denote the Jacobian matrix at $x\in\mathcal{X}$. For a matrix $A:\mathbb{R}^d\to\mathbb{R}^d$ let $\|A\|_{p}$ denote the operator norm induced by the $\ell_p$ norm on $\mathbb{R}^d$, i.e., 
\[
\|A\|_p = \sup\left\{\|Ax\|_p:x\in\mathbb{R}^d, \|x\|_p=1\right\}.
\]
When $p=\infty$, this norm is simply the maximum absolute row sum of the matrix:
\[
\|A\|_\infty = \max_{1\le i\le d}\sum_{j=1}^d |A_{ij}|.
\]
We aim to bound
$
\|JT^n(y)\|_\infty
$
uniformly in $y\in\mathcal{Y},n\in\mathbb{N}$. Since $\{S^n\}\subset\mathcal{T}(s,d,M^*)$, the Jacobian matrix $JS^n(x)$ is upper triangular for every $x\in\mathcal{X}$ and
\[
\max_{i<j}\sup_{x\in\mathcal{X}}|[JS^n(x)]_{ij}|=\max_{i<j}\|D_jS^n_i\|_\infty\le M^*.
\]
Similarly, we also have that
\[
\min_{j\in [d]}\inf_{x\in\mathcal{X}}|[JS^n(x)]_{jj}|=\min_{j\in [d]}\inf_{x\in\mathcal{X}}|D_jS^n_j(x)|\ge 1/M^*.
\]
It follows that for each $x\in\mathcal{X}$, the Jacobian $JS^n(x)$ satisfies the hypotheses of Lemma \ref{lemma:matrixbound} with $L=M^*$. Applying the inverse function theorem, we conclude from the lemma that
the entries of $JT^n(y)=(JS^n(T^n(y)))^{-1}$ are bounded as
\begin{align*}
\max_{i<j}\sup_{y\in\mathcal{Y}} |[JT^n(y)]_{ij}|&\le (M^*)^3((M^*)^2+1)^{j-i-1}, \\
\max_{j\in [d]}\sup_{y\in\mathcal{Y}} |[JT^n(y)]_{jj}|&\le M^*.
\end{align*}
It follows that the $\ell_\infty$ operator norm is bounded as
\begin{align*}
\sup_{y\in\mathcal{Y}}\|JT^n(y)\|_\infty \le&\ M^*+\sum_{j=2}^d (M^*)^3((M^*)^2+1)^{j-2} \\
=&\begin{cases}
M^*+(M^*)^3\cdot\frac{1-((M^*)^2+1)^{d-2}}{1-((M^*)^2+1)}, & d\ge 2, \\
M^*, &d =1.
\end{cases} \tag{partial geometric series}\\
=& \begin{cases}
M^*((M^*)^2+1)^{d-2}, & d\ge 2, \\
M^*, & d=1.
\end{cases} \\
=&\ M^*\cdot \max\{1,((M^*)^2+1)^{d-2}\}\\
:=&\ C(d,M^*).
\end{align*}
Here we have used the convention that $\sum_{j=2}^{1}a_j=0$ for any sequence $\{a_j\}$.
Now we apply the mean value inequality for vector-valued functions to deduce that the $\{T^n\}$ are uniformly equicontinuous. Indeed, for any $y_1,y_2\in\mathcal{Y}$, we have
\begin{align*}
\|T^n(y_1)-T^n(y_2)\|_\infty &\le \sup_{y\in\mathcal{Y}}\|JT^n(y)\|_\infty\|y_1-y_2\|_\infty \\
&\le C(d,M^*)\|y_1-y_2\|_\infty.
\end{align*}

Now let $y\in\mathcal{Y}$ and note that $y=S^*(x)$ for some $x\in\mathcal{X}$. 
We then have
\begin{align*}
\|T^n(y)-T^*(y)\|_\infty &= \|T^n(y)-T^*(S^*(x))\|_\infty  \\
&= \|T^n(y)-x\|_\infty  \\
&= \|T^n(S^*(x))-T^n(S^n(x))\|_\infty \\
&\le C(d,M^*)\|S^*(x)-S^n(x)\|_\infty \\
&\le C(d,M^*)\|S^*-S^n\|_{\infty,d}.
\end{align*}
Since $y\in\mathcal{Y}$ was arbitrary, we can take the supremum over $y$ on the left side to obtain the desired result:
\[
\|T^n-T^*\|_{\infty,d} \le C(d,M^*)\|S^*-S^n\|_{\infty,d} \stackrel{p}{\to} 0. \tag{Theorem \ref{thm:uniformconv}}
\]
\end{proof}

\subsection{Sobolev rates under log-concavity}
\label{sec:concave-proof}
\subsubsection{Proof of Theorem \ref{thm:KRrate}}

Suppose the source density $g$ is log-concave, which implies that $S\mapsto P\psi_S$ and $S\mapsto P_n\psi_S$ are strictly convex functionals. 
Since $\mathcal{T}(s,d,M)$ is convex, 
\[
\min_{S\in\mathcal{T}(s,d,M)} P\psi_S, \qquad \min_{S\in\mathcal{T}(s,d,M)} P_n\psi_S
\]
are convex optimization problems. If in addition $g$ is strongly log-concave, we obtain strong convexity of the objective, as we show in Lemma \ref{lemma:strongcvx}.

\begin{lemma}
Suppose Assumptions \ref{assumption:compact}-\ref{assumption:smooth} hold. Assume further that the source density $g$ is $m$-strongly log-concave for some $m>0$:
\[
[\nabla \log g(y_1)-\nabla \log g(y_2)]^T(y_1-y_2) \le m\|y_1-y_2\|_2^2 \qquad\forall y_1,y_2\in\mathcal{Y}.
\]
Then the map $S\mapsto P\psi_S$ on $\mathcal{T}(s,d,M)$ is $\min\{m,M^{-2}\}$-strongly convex with respect to the $L^2$ Sobolev-type norm
\begin{align*}
\|S\|_{H^{1,2}_f(\mathcal{X})}^2 :=&\  \|S\|^2_{L^2_f(\mathcal{X})}+\sum_{k=1}^d\|D_kS_k\|_{L^2_f(\mathcal{X})}^2.
\end{align*}
\label{lemma:strongcvx}
\end{lemma}

\begin{proof}
We first calculate the G\^ateaux derivative of $S\mapsto P\psi_S$ in the direction $A\in\mathcal{T}(s,d,M)$.
\begin{align*}
\nabla P\psi_S(A) &= \lim_{t\to 0}\frac{P\psi_{S+tA}-P\psi_S}{t} \\
&=\lim_{t\to 0}-t^{-1}\E\bigg\{\left[\log g((S+tA)(X))-\log g(S(X))\right] \\
&\hspace{3cm}+\sum_{k=1}^d\left[\log D_k(S_k+tA_k)(X)-\log D_kS_k(X)\right]\bigg\} \\
&= -\E\left\{\nabla \log g(S(X))^TA(X)+\sum_k\frac{D_kA_k(X)}{D_kS_k(X)}\right\}.
\end{align*}
We can differentiate under the integral by the dominated convergence theorem, since the integrand is smooth and compactly supported by Assumptions \ref{assumption:compact}-\ref{assumption:smooth}.

Now note that $\nabla P\psi_S(A)$ is a bounded linear operator in $A$. Furthermore, since the KR map $S^*$ is the global minimizer of $S\mapsto P\psi_S$, we have $\nabla P\psi_{S^*}(A)= 0$ for all $A\in\mathcal{T}(s,d,M)$ satisfying $S^*+tA\in\mathcal{T}(s,d,M)$ for all $t$ sufficiently small.
We now check the strong convexity condition. Assume $A,B\in\mathcal{T}(s,d,M)$ for some $M>0$. We have
\begin{align*}
(\nabla P\psi_A-\nabla & P\psi_B)(A-B) = \nabla P\psi_A(A-B)-\nabla P\psi_B(A-B) \\
&= \E\{[\nabla\log g(B(X))-\nabla\log g(A(X))]^T(A(X)-B(X))\\
&\qquad+\sum_k\left[ \frac{D_k(A_k-B_k)(X)}{D_kB_k(X)}-\frac{D_k(A_k-B_k)(X)}{D_kA_k(X)}\right]\}\\
&\ge \E\left\{m\|A(X)-B(X)\|_2^2+\sum_k\left[ \frac{D_k(A_k-B_k)(X)}{D_kB_k(X)}-\frac{D_k(A_k-B_k)(X)}{D_kA_k(X)}\right]\right\} \tag{strong log-concavity of $g$}\\
&= \E\left\{m\|A(X)-B(X)\|_2^2+\sum_k \frac{[D_k(A_k-B_k)(X)]^2}{D_kA_k(X)D_kB_k(X)}\right\} \\
&\ge \E\left\{m\|A(X)-B(X)\|_2^2+\frac{1}{M^2}\sum_k [D_k(A_k-B_k)(X)]^2\right\} \tag{$A,B\in\mathcal{T}(s,d,M)$}\\
&= m\|A-B\|_{L^2(f)}^2+\frac{1}{M^2}\sum_k\|D_k(A_k-B_k)\|_{L^2(f)}^2 \\
&\ge \min\{m,M^{-2}\}\|A-B\|_{H^1(f)}^2. 
\end{align*}
Hence, $S\mapsto P\psi_S$ satisfies the first-order strong convexity condition.
\end{proof}

We now proceed to prove Theorem \ref{thm:KRrate}.

\begin{proof}[Proof of Theorem \ref{thm:KRrate}]
By Lemma \ref{lemma:strongcvx}, strong convexity of $S\mapsto P\psi_S$ with respect to $\|\cdot\|_{H^{1,2}_f(\mathcal{X})}$ implies that
\begin{align*}
P\psi_{S^n} &= P\psi_{S^n}-P\psi_{S^*} \\&\ge \nabla P\psi_{S^*}(S^n-S^*)+\frac{\min\{m,(M^*)^{-2}\}}{2}\|S^n-S^*\|_{H^{1,2}_f(\mathcal{X})}^2 \\
&= \frac{\min\{m,(M^*)^{-2}\}}{2}\|S^n-S^*\|_{H^{1,2}_f(\mathcal{X})}^2,
\end{align*}
since $\nabla P\psi_{S^*}(S^n-S^*)= 0$, as $S^*$ minimizes $P\psi_S$. We complete the proof by appealing to the bound on $\E[P\psi_{S^n}]$ established in Theorem \ref{thm:KLconv}.
\end{proof}

\subsubsection{Sobolev rates for the inverse map}

Now we prove a rate of convergence of the inverse map estimator in the $L^2$ Sobolev norm $\|\cdot\|_{H^{1,2}_g(\mathcal{Y})}$ assuming strong log-concavity, as in Theorem \ref{thm:KRrate}.

\begin{theorem}
Suppose Assumptions \ref{assumption:compact}-\ref{assumption:smooth} hold. Assume further that $g$ is
$m$-strongly log-concave. 
Let $S^n$ be a near-optimizer of the functional $S\mapsto P_n\psi_S$ on $\mathcal{T}(s,d,M^*)$ with remainder $R_n$ satisfying
\[
\E[R_n] = \E\left\{P_n\psi_{S^n} - \inf_{S\in\mathcal{T}(s,d,M^*)} P_n\psi_S\right\} \lesssim \E\|P_n-P\|_{\Psi(s,d,M^*)}.
\]
Then $T^n=(S^n)^{-1}$ converges to $T^*=(S^*)^{-1}$ with respect to the norm $\|\cdot\|_{H^{1,2}_g(\mathcal{Y})}$ with rate
\[
\frac{\min\{m,(M^*)^{-2}\}}{2}\E\|T^n-T^*\|_{H^1(g)}^2 \lesssim \begin{cases}
n^{-1/2}, & d < 2s, \\
n^{-1/2}\log n, &d=2s, \\
n^{-s/d}, &d>2s.
\end{cases}
\]
\label{thm:invrate}
\end{theorem}

\begin{proof}
We aim to bound $\|T^n-T^*\|_{H^{1,2}_g(\mathcal{Y})}$ by a multiple of $\|S^n-S^*\|_{H^{1,2}_f(\mathcal{X})}$, which will establish the same rate of convergence for the inverse map (up to constant factors) as derived for $S^n$ in Theorem \ref{thm:KRrate}.

In the proof of Theorem \ref{thm:invconsist} we showed that we can bound the $\ell_\infty$ matrix norm of the Jacobian $JT^n(y)$ uniformly as
\[
\sup_{y\in\mathcal{Y}}\|JT^n(y)\|_\infty \le C(d,M^*) := M^*\cdot \max\{1,((M^*)^2+1)^{d-2}\}.
\]
Since $JT^n(y)$ is upper triangular, we can use the exact same argument to arrive at the same bound on the $\ell_1$ matrix norm, which equals the maximum absolute column sum of the matrix,
\[
\|A\|_1 = \max_{1\le j\le d}\sum_{i=1}^d |A_{ij}|.
\]
Hence we conclude that $\sup_{y\in\mathcal{Y}}\|JT^n(y)\|_1 \le C(d,M^*)$.
Now we apply H\"older's inequality for matrix norms to conclude that
\[
\sup_{y\in\mathcal{Y}}\|JT^n(y)\|_2 \le \sup_{y\in\mathcal{Y}}\sqrt{\|JT^n(y)\|_1\|JT^n(y)\|_\infty} \le C(d,M^*).
\]
Consequently, we can bound the first term in $\|\cdot\|_{H^{1,2}_g(\mathcal{Y})}$ as
\begin{align*}
\|T^n-T^*\|_{L^2_g(\mathcal{Y})}^2 &= \int \|T^n(y)-T^*(y)\|_2^2 \ g(y)dy \\
&= \int \|T^n(S^*(x))-T^*(S^*(x))\|_2^2\ |\det JS^*(x)| g(S^*(x))dx \tag{Change of variables} \\
&= \int \|T^n(S^*(x))-T^*(S^*(x))\|_2^2\ f(x)dx \tag{$g=T_\# f$} \\
&= \int \|T^n(S^*(x))-T^n(S^n(x))\|_2^2\ f(x)dx \\
&\le \int \sup_{y\in\mathcal{Y}}\|JT^n(y)\|_2^2\|S^*(x)-S^n(x)\|_2^2\ f(x)dx \\
&\le C(d,M^*)^2\int \|S^*(x)-S^n(x)\|_2^2\ f(x)dx \\
&= C(d,M^*)^2 \|S^*-S^n\|_{L^2_f(\mathcal{X})}^2.
\end{align*}

To bound the deviations of the first derivatives, note that
\begin{align*}
|D_kT^n_k(y)-D_kT^*_k(y)| &= \left|\frac{1}{D_kS_k^n(T^n(y))} - \frac{1}{D_kS^*_k(T^*(y))}\right| \\
&= \left|\frac{D_kS_k^n(T^n(y))-D_kS^*_k(T^*(y))}{[D_kS_k^n(T^n(y))][D_kS^*_k(T^*(y))]}\right| \\
&\le (M^*)^2 \left|D_kS_k^n(T^n(y))-D_kS^*_k(T^*(y))\right| \tag{$S,S^n\in\mathcal{T}(s,d,M^*)$}\\
&\le (M^*)^2\left|D_kS_k^n(T^n(y))-D_kS_k^n(T^*(y))\right|\\
&\quad +(M^*)^2\left|D_kS_k^n(T^*(y))-D_kS^*_k(T^*(y))\right| \\
&\le (M^*)^2\sup_{x\in\mathcal{X}}\|\nabla( D_kS_k^n)(x)\|_2\|T^n(y)-T^*(y)\|_2  \\
&\quad +(M^*)^2\left|D_kS_k^n(T^*(y))-D_kS^*_k(T^*(y))\right|. \tag{mean value inequality}
\end{align*}
Now note that $S_k^n$ depends only on $(x_k,\ldots,x_d)$, and therefore, since $S^n\in\mathcal{T}(s,d,M^*)$ implies that $\max_{j\in[d]}\|D_jD_kS_k^n\|_\infty\le M^*$, we have
\begin{align*}
\sup_{x\in\mathcal{X}}\|\nabla( D_kS_k^n)(x)\|_2 &= \sup_{x\in\mathcal{X}}\sqrt{\sum_{j=1}^d  |D_jD_kS_k^n(x)|^2} \\
&= \sup_{x\in\mathcal{X}}\sqrt{\sum_{j=k}^d  |D_jD_kS_k^n(x)|^2} \\
&\le\sqrt{\sum_{j=k}^d  (M^*)^2} \\
&= M^*\sqrt{d-k+1}.
\end{align*}
Hence, we conclude that
\begin{align*}
|D_kT^n_k(y)-D_kT^*_k(y)|^2 
&\le \{(M^*)^3\sqrt{d-k+1}\|T^n(y)-T^*(y)\|_2 \\
&\qquad +(M^*)^2\left|D_kS_k^n(T^*(y))-D_kS^*_k(T^*(y))\right|\}^2 \\
&\le 2(M^*)^6(d-k+1)\|T^n(y)-T^*(y)\|_2^2 \\
&\qquad +2(M^*)^4\left|D_kS_k^n(T^*(y))-D_kS^*_k(T^*(y))\right|^2. \tag{$(a+b)^2\le 2(a^2+b^2)$}
\end{align*}
Summing over $k$ and integrating against the density $g$, we obtain
\begin{align*}
\sum_{k=1}^d\|D_kT^n_k-& D_kT^*_k\|_{L^2_g(\mathcal{Y})}^2 = \sum_{k=1}^d \int |D_kT^n_k(y)-D_kT^*_k(y)|^2 g(y)dy \\
&\le \sum_{k=1}^d \int 2(M^*)^6(d-k+1)\|T^n(y)-T^*(y)\|_2^2\ g(y)dy \\
&\quad +\sum_{k=1}^d\int 2(M^*)^4\left|D_kS_k^n(T^*(y))-D_kS^*_k(T^*(y))\right|^2g(y)dy \\
&= (M^*)^6d(d+1)\|T^n-T^*\|^2_{L^2_g(\mathcal{Y})}\\
&\quad+2(M^*)^4 \sum_{k=1}^d\int \left|D_kS_k^n(x)-D_kS^*_k(x)\right|^2g(S^*(x))|\det JS^*(x)|dx \\
&= (M^*)^6d(d+1)\|T^n-T^*\|^2_{L^2_g(\mathcal{Y})}\\
&\quad+2(M^*)^4 \sum_{k=1}^d\int \left|D_kS_k^n(x)-D_kS^*_k(x)\right|^2f(x)dx \\
&= (M^*)^6d(d+1)\|T^n-T^*\|^2_{L^2_g(\mathcal{Y})}+2(M^*)^4 \sum_{k=1}^d\|D_k(S_k^n-S^*_k)\|^2_{L^2_f(\mathcal{X})} \\
&\le (M^*)^6d(d+1)C(d,M^*)^2\|S^n-S^*\|^2_{L^2_f(\mathcal{X})}+2(M^*)^4 \sum_{k=1}^d\|D_k(S_k^n-S^*_k)\|^2_{L^2_f(\mathcal{X})} \\
\end{align*}
Putting all of these calculations together, we have shown that
\begin{align*}
\|T^n-T^*\|^2_{H^{1,2}_g(\mathcal{Y})} &= \|T^n-T^*\|^2_{L^2_g(\mathcal{Y})}+\sum_{k=1}^d \|D_k(T^n_k-T^*_k)\|^2_{L^2_g(\mathcal{Y})} \\
&\le C(d,M^*)^2\|S^n-S^*\|^2_{L^2_f(\mathcal{X})} + (M^*)^6d(d+1)C(d,M^*)^2\|S^n-S^*\|^2_{L^2_f(\mathcal{X})}\\
&\quad+2(M^*)^4 \sum_{k=1}^d\|D_k(S_k^n-S^*_k)\|^2_{L^2_f(\mathcal{X})} \\
&\le \tilde{C}(d,M^*)^2\|S^n-S^*\|^2_{H^{1,2}_f(\mathcal{X})},
\end{align*}
where we define
\[
\tilde{C}(d,M)^2 = \max\{C(d,M)^2[1+M^6d(d+1)],2M^4\}.
\]
Finally, we appeal to the bound on $\|S^n-S^*\|^2_{H^{1,2}_f(\mathcal{X})}$ derived in Theorem \ref{thm:KRrate} to conclude the proof:
\begin{align*}
\frac{\min\{m,(M^*)^{-2}\}}{2}\E\|T^n-T^*\|^2_{H^{1,2}_g(\mathcal{Y})}&\le \frac{\min\{m,(M^*)^{-2}\}}{2}\tilde{C}(d,M^*)^2\E\|S^n-S^*\|_{H^{1,2}_f(\mathcal{X})}^2 \\
&\lesssim \begin{cases}
n^{-1/2}, & d < 2s, \\
n^{-1/2}\log n, &d=2s, \\
n^{-s/d}, &d>2s.
\end{cases}
\end{align*}
\end{proof}

\subsection{Dimension ordering}
\label{sec:ordering-proof}
\subsubsection{Proof of Lemma \ref{lemma:anisorate}}

Now we establish a rate of convergence in the anisotropic smoothness setting. Define
\[
\psi_{S}^k(x) = \log f_k(x_k|x_{(k+1):d})-\log g_k(S_k(x)|S_{(k+1):d}(x))+\log D_kS_k(x),
\]
which is $(s_k,\ldots,s_d)$-smooth in $(x_k,\ldots,x_d)$ whenever each $S_k$ is $(s_k+1,s_{k+1},\ldots,s_d)$-smooth in $(x_k,x_{k+1},\ldots,x_d)$ by Assumption \ref{assumption:multismooth} and Lemma \ref{lemma:multismooth}.
Since 
\[
f(x)=\prod_{k=1}^d f_k(x_k|x_{(k+1):d})
\]
by the chain rule of densities, and similarly for $g$, we have
\[
(P_n-P)\psi_S = \sum_{k=1}^d(P_n-P)\psi_{S}^k.
\]
Note that $\psi_{S}^k$ is a function of $S_{k:d}$ and the $d_k$ variables $(x_k,\ldots,x_d)$ only. Defining
\[
\Psi_k(\mathfrak{s},d,M) = \{\psi_S^k : S\in\mathcal{T}(\mathfrak{s},d,M)\},
\]
we have the following analog of Proposition \ref{prop:entbound}.

\begin{prop}
Under Assumptions \ref{assumption:compact}, \ref{assumption:convex}, and \ref{assumption:multismooth}, 
the metric entropy of $\Psi_k(\mathfrak{s},d,M)$ in the $L^\infty$ norm is bounded as
\[
H(\epsilon, \Psi_k(\mathfrak{s},d,M),\|\cdot\|_\infty) 
\lesssim \epsilon^{-d_k/{\sigma}_k}.
\]
Consequently, $\Psi_k(\mathfrak{s},d,M)$ is totally bounded and therefore precompact in $L^\infty(\mathcal{X}_{k:d})$.
\label{prop:multientbound}
\end{prop}

Hence, we obtain in Lemma \ref{lemma:anisorate} a bound on the supremum process analogous to Lemma \ref{lemma:EPrate} now applied to the family $\Psi_k(\mathfrak{s},d,M)$ of $(s_k,\ldots,s_d)$-smooth maps $\psi_{S}^k$.

\begin{proof}[Proof of Lemma \ref{lemma:anisorate}]
The metric entropy integral bounds utilized in the proof of Lemma \ref{lemma:EPrate}, combined with the entropy bound on $\Psi_k(\mathfrak{s},d,M)$ derived in Proposition \ref{prop:multientbound}, yield the following rate
\[
\E\|P_n-P\|_{\Psi_k(\mathfrak{s},d,M)}=\E\left\{\sup_{S\in\mathcal{T}(\mathfrak{s},d,M)} |(P_n-P)\psi_{S}^k|\right\} \lesssim 
c_{n,k}.
\]
Applying the triangle inequality then yields 
\begin{align*}
\E\|P_n-P\|_{\Psi(\mathfrak{s},d,M)} &= \E\left\{\sup_{S\in\mathcal{T}(\mathfrak{s},d,M)}|(P_n-P)\psi_S|\right\}\\
&= \E\left\{\sup_{S\in\mathcal{T}(\mathfrak{s},d,M)}\left|\sum_{k=1}^d(P_n-P)\psi^k_S\right|\right\}\\
&\le \sum_{k=1}^d \E\left\{\sup_{S\in\mathcal{T}(\mathfrak{s},d,M)}\left|(P_n-P)\psi^k_S\right|\right\}\\
&= \sum_{k=1}^d \E\|P_n-P\|_{\Psi_k(\mathfrak{s},d,M)} \\
&\lesssim
\sum_{k=1}^d  c_{n,k}.
\end{align*}
\end{proof}

\subsection{Jacobian flows}
\label{sec:flows-proof}
\begin{proof}[Proof of Theorem \ref{thm:flow}]
The proof is practically identical to that of Theorem \ref{thm:KLconv}, since the functions in $\Psi_m(\Sigma,s,M)$ are $s$-smooth with uniformly bounded derivatives (analogous to $\Psi(s,d,M)$) by definition of $\mathcal{J}_m(\Sigma,s,M)$, the chain rule of differentiation, and the relation
\begin{align*}
\psi_S(x) 
= & \log [f(x)/g(S(x))] - \sum_{j=1}^m \sum_{k=1}^d \log D_kU_k^j(x^j),
\end{align*}
where we define
\[
x^j = \Sigma^j \circ U^{j-1} \circ \Sigma^{j-1}\circ\cdots\circ U^1\circ \Sigma^1(x), \quad j\in[m].
\]
Hence, the entropy estimates for $\Psi(s,d,M)$ in Proposition \ref{prop:entbound} hold also for $\Psi_m(\Sigma,s,M)$. Thus, we obtain similar bounds on $\E\|P_n-P\|_{\Psi_m(\Sigma,s,M)}$ as in Lemma \ref{lemma:EPrate}. Combining this with the risk decomposition (\ref{eq:lossbound}) and the argument in Theorem \ref{thm:KLconv} completes the proof.
\end{proof}

\subsection{On separability}
\label{sec:separability}
Suppose the source $g$ is a product density, i.e., $g(y)=\prod_{k=1}^d g_k(y_k)$ for some smooth densities $g_k:\mathbb{R}\to\mathbb{R}$. As the source density $g$ is a degree of freedom in our problem, we are free to choose $g$ to factor as such. For example, $g$ could be 
the standard normal density in $d$-dimensions or
the uniform density on a box $\mathcal{Y}\subset\mathbb{R}^d$. We will show that the task of estimating the KR map $S^*$ is amenable to distributed computation in this case.

\begin{assumption}
The source density $g$ factors as a product: $g(y)=\prod_{k=1}^d g_k(y_k)$.
\label{assumption:product}
\end{assumption}

Recall the minimization objective defining our estimator:
\[
-\frac{1}{n}\sum_{i=1}^n\left[\log g(S(X^i))+\sum_{k=1}^d\log D_kS_k(X^i)\right],
\]
where $X^i=(X^i_1,\ldots,X^i_d),i=1,\ldots,n$ is an iid random sample from $f$. Here we omit the entropy term involving $f$ that does not depend on $S$ without loss of generality. For a general density $g$, we can simplify the above expression by appealing to the chain rule for densities 
\[
g(y) = \prod_{k=1}^d g_k(y_k|y_{k+1},\ldots,y_d),
\]
where 
\[
g_k(y_k|y_{k+1},\ldots,y_d) = \frac{\int g(y_1,\ldots,y_d)dy_1\cdots dy_{k-1}}{\int g(y_1,\ldots,y_d)dy_1\cdots dy_{k}}.
\]
The objective then becomes
\[
-\frac{1}{n}\sum_{i=1}^n\sum_{k=1}^d\left[\log g_k(S_k(X^i)|S_{k+1}(X^i),\ldots,S_d(X^i))+\log D_kS_k(X^i)\right].
\]
When $g$ is a product density we have $g_k(y_k|y_{k+1},\ldots,y_d)=g_k(y_k)$ and we obtain
\[
\sum_{k=1}^d\left\{-\frac{1}{n}\sum_{i=1}^n\left[\log g_k(S_k(X^i))+\log D_kS_k(X^i)\right]\right\},
\]
which is a separable objective over the component maps $S_k$. In this case, we can find our estimator $S^n=(S^n_1,\ldots,S^n_d)$ by solving for the components $S^n_k$ in parallel.

\end{document}